\documentclass[11pt,letterpaper,onecolumn]{srllab_template} 

\usepackage{amsmath,amsfonts,bm}

\def\eqref#1{equation~\ref{#1}}

\def\1{\bm{1}}

\DeclareMathAlphabet{\mathsfit}{\encodingdefault}{\sfdefault}{m}{sl}
\SetMathAlphabet{\mathsfit}{bold}{\encodingdefault}{\sfdefault}{bx}{n}

\DeclareMathOperator*{\argmax}{arg\,max}

\usepackage[authoryear,round]{natbib}
\usepackage{nicefrac}
\usepackage{wrapfig}
\usepackage{algorithm}
\newtheorem{assumption}{Assumption}
\newtheorem{definition}{Definition}
\newtheorem{proposition}{Proposition}

\newtheorem{property}{Property}

\usepackage[capitalize,noabbrev]{cleveref}
\crefname{assumption}{Assumption}{Assumptions}
\Crefname{assumption}{Assumption}{Assumptions}
\crefname{property}{Property}{Properties}
\Crefname{property}{Property}{Properties}
\usepackage{algorithmic}
\usepackage{multirow}
\usepackage{mathtools}
\graphicspath{{./}}

\crefname{equation}{Eq.}{Eqs.}

\newcommand{\method}{\texttt{DCRL}}
\newcommand{\trl}{\texttt{TRL}}
\newcommand{\td}{\texttt{TD}}

\newcommand{\std}[1]{{\scriptsize $\pm$ #1}}

\newcommand{\NStepTable}[2]{%
\begin{minipage}[t]{0.48\textwidth}
    \captionsetup{width=\linewidth}
    \centering
    \caption{Results with $n=#1$.}
    \label{tab:n#1}
    \resizebox{\linewidth}{!}{%
        \begin{tabular}{@{}llccc@{}}
            \toprule
            \texttt{Environment} & \texttt{Task} & \texttt{TD-#1} & \texttt{TRL-#1} & \texttt{DCRL-#1} \\ \midrule
            #2
            \bottomrule
        \end{tabular}%
    }
\end{minipage}
}

\title{Recursive Value Learning for Long-Horizon Offline Goal-Conditioned RL}

\author{Hyeonseong Jeon\textsuperscript{1} and Youngwoon Lee\textsuperscript{2}\\
\Affilfont\textsuperscript{1}Yonsei University, \textsuperscript{2}Seoul National University\\
\texttt{yeonsumia@snu.ac.kr, youngwoon@snu.ac.kr}}
\reportnumber{}

\newcommand{\eg} {\emph{e.g.}}
\newcommand{\ie} {\emph{i.e.}}

\begin{abstract}
Scaling offline goal-conditioned reinforcement learning (GCRL) to long-horizon tasks is difficult because (1) long-range value learning depends on shorter-range estimates that may still be inaccurate, and (2) max-based value backups can amplify overestimation through repeated propagation. We propose \method~(\texttt{D}ivide-and-\texttt{C}onquer RL), which recursively decomposes each trajectory segment into a balanced binary tree and trains the values from leaves to root. Each parent is therefore updated only after its children, using an exact factorization of the observed route rather than selecting among noisy alternatives. Since this objective learns values along demonstrated routes that are not necessarily optimal, \method\ jointly propagates values across trajectories to discover shorter routes. Thanks to the balanced binary tree, \method\ reduces worst-case bootstrap depth from linear to logarithmic, and this shorter dependency structure empirically corresponds to much slower error accumulation. Across diverse goal-reaching tasks, \method\ substantially outperforms prior flat offline GCRL methods, and on the five most challenging long-horizon OGBench tasks, it improves the best prior average score from $55$ to $64$, surpassing all flat and hierarchical baselines.
\end{abstract}
\begin{document}
\maketitle
\begin{center}
\vspace{-0.75\baselineskip}
{\projectpagefont
\textbf{Project page:} \href{https://yeonsumia.github.io/dcrl}{\texttt{yeonsumia.github.io/dcrl}}
}
\end{center}
\section{Introduction}

\begin{figure}[h]
    \centering
    \includegraphics[width=0.83\linewidth]{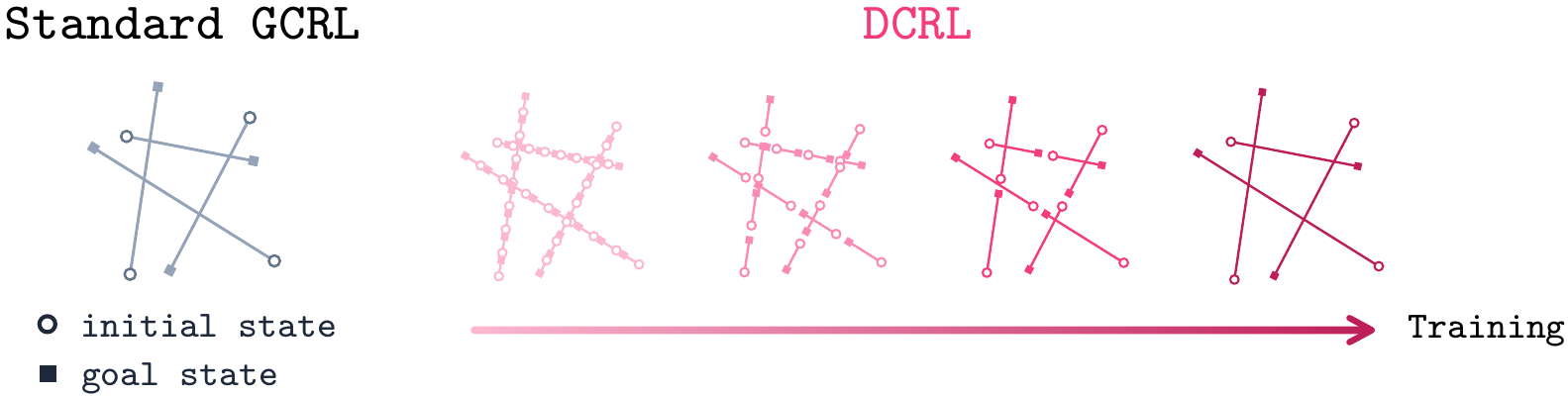}
    \caption{\method\ (right) learns short segments first, then uses them to learn longer ones, whereas standard GCRL (left) randomly samples pairs of all lengths in no particular order.}
    \label{fig:main_figure}
\end{figure}

Long-horizon goal-conditioned value learning has a natural dependency structure: the value of a long trajectory segment is built from the values of its shorter constituent segments. Yet most offline goal-conditioned reinforcement learning (GCRL) methods form bootstrapped targets from independently sampled transitions or subgoals, without ensuring that shorter-horizon values are learned before the longer-horizon values that depend on them. A long-range estimate may therefore bootstrap from shorter-range values that are themselves still inaccurate. Moreover, standard Bellman-style and transitive value backups choose optimistically among actions or intermediate states~\citep{bellman1966dynamic,floyd1962algorithm}. With finite offline data, this selection can favor overestimated candidates and propagate their errors through repeated backups, turning small local errors into large global inconsistencies.

We propose \method~(\texttt{D}ivide-and-\texttt{C}onquer RL), which makes this dependency explicit and enforces it through bottom-up training. The key idea is simple: learn shorter-range values (children) before the longer-range values (parents) that depend on them (\cref{fig:main_figure}). Given a sampled trajectory segment, \textbf{Divide} recursively splits it at its \textit{midpoint} until reaching single-step segments. \textbf{Conquer} then trains a value function by processing the resulting segments from leaves to root. Thus, each parent is updated only after its children. A horizon-$H$ segment therefore produces a balanced binary tree with $O(H)$ nodes but only $O(\log_2 H)$ sequential composition levels. This gives \method\ \emph{logarithmic} bootstrap depth in the horizon, whereas standard backups have \emph{linear} worst-case depth.

At each internal node of the tree, \method\ applies an exact trajectory factorization along the observed route.
For any intermediate state $w$ on a trajectory segment from $s$ to $g$, the corresponding behavior value $V_\tau$ satisfies 
\[
    V_\tau(s,g) = V_\tau(s,w) \cdot V_\tau(w,g).
\]
Unlike standard max-based backups, this factorization evaluates a single route supported by the dataset. \method\ thus addresses the two sources of long-horizon error directly: bottom-up scheduling imposes the dependency order among value estimates, while exact route factorization avoids optimistic selection among noisy alternatives. We empirically show that \method\ accumulates less long-range error than standard max-based backups.

\begin{wrapfigure}{r}{0.42\textwidth}
    \vspace{-12pt}
    \centering
    \includegraphics[width=\linewidth]{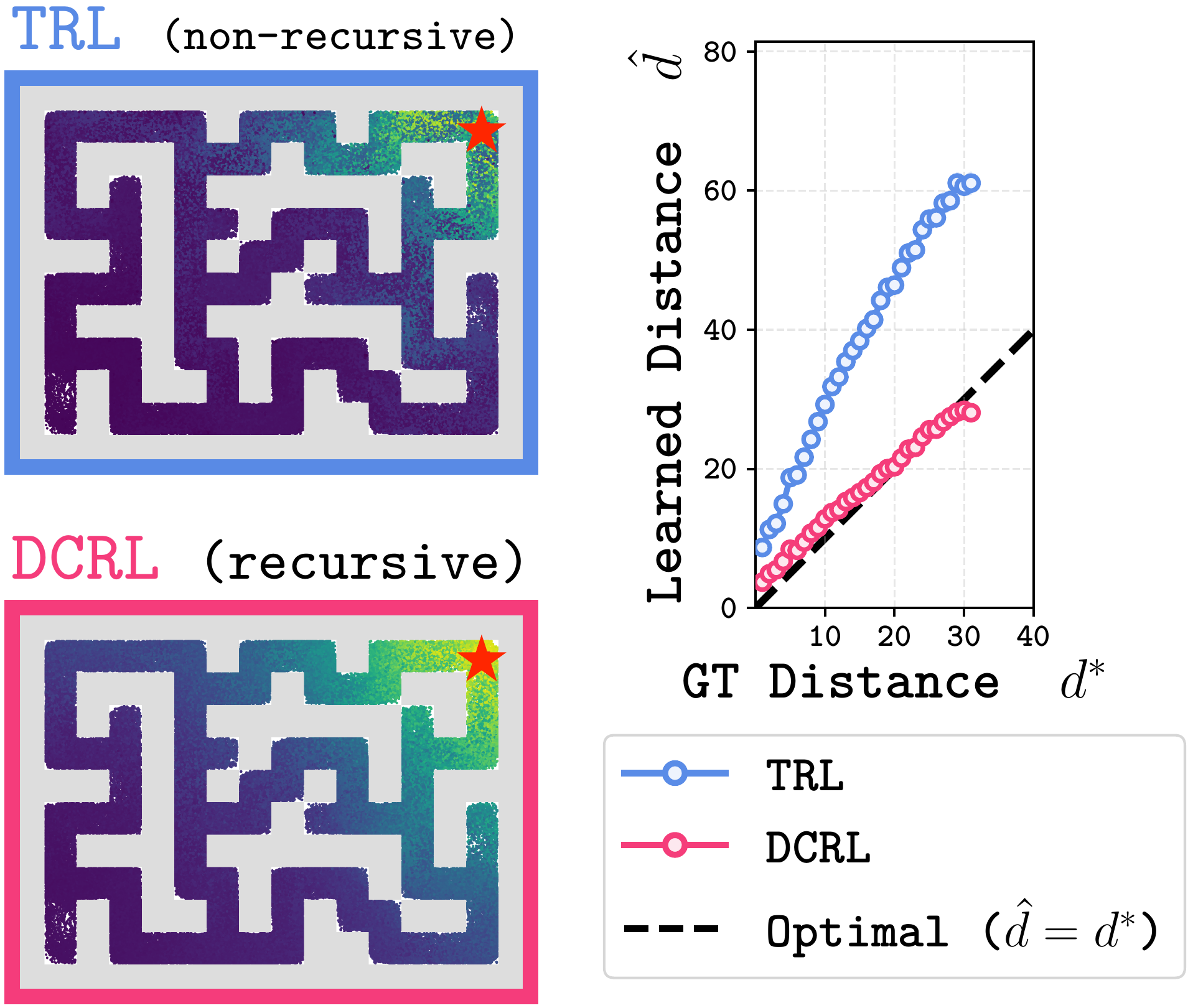}
    \caption{In \texttt{humanoidmaze-giant}, \method\ learns near-optimal distance estimates. We measure the distances predicted by \method\ and \trl\ from each state to a fixed goal (red star).} 
    \label{fig:value_map}
    \vspace{-15pt}
\end{wrapfigure}

Explicit tree construction distinguishes \method\ from Transitive RL (\trl)~\citep{park2026transitive}. \trl\ independently samples a trajectory segment and a random in-trajectory split at each update, thereby learning from isolated transitive decompositions. \method\ instead instantiates each sampled trajectory segment as a balanced binary tree, preserves its parent--child dependencies, and learns the value function on its segments from leaves to root. Thus, \trl\ samples individual transitive relations, whereas \method\ explicitly constructs and executes a divide-and-conquer computation.
As illustrated in \cref{fig:value_map}, this structured learning order produces substantially more precise long-range value estimates than \trl.

To preserve the ordering while maintaining diverse minibatches, \method\ interleaves multiple recursive trees through \textbf{slot scheduling}, which processes each tree bottom-up while mixing segment lengths within each minibatch.

Although recursive divide-and-conquer value learning mitigates error accumulation over the horizon, it recovers only behavior values. To recover optimal values, \method\ jointly trains the value function using multistep value propagation with goals globally relabeled from the dataset~\citep{andrychowicz2017hindsight,sutton2018reinforcement}. This propagation objective identifies shorter routes by combining segments from different trajectories. Thus, \method's two objectives serve complementary roles: recursive divide-and-conquer value learning provides reliable behavior routes, while multistep propagation recovers optimal routes by composing cross-trajectory routes.

In summary, our main contributions are threefold:
\begin{itemize}[leftmargin=*]
    \item \textbf{Algorithm.} We introduce \method, which explicitly instantiates divide-and-conquer value learning as a balanced binary tree over in-trajectory segments, trains each tree bottom-up using value factorization, and processes multiple trees in parallel through slot scheduling.

    \item \textbf{Analysis.} We prove that \method\ reduces worst-case bootstrap depth from \emph{linear} to \emph{logarithmic} in the horizon. In a controlled experiment, we show that this shorter dependency structure corresponds to substantially slower error accumulation.

    \item \textbf{Performance.} \method\ substantially outperforms prior flat offline GCRL algorithms across diverse goal-reaching tasks spanning different domains, horizons, and observation modalities, and surpasses hierarchical approaches on several challenging long-horizon tasks.
\end{itemize}

\section{Related Work}
\paragraph{Goal-conditioned RL.}
Goal-conditioned RL aims to learn a policy that takes the optimal action from any state to reach a goal. We focus on the offline setting, which learns purely from a fixed dataset without online interaction. Prior work can be broadly categorized into Temporal Difference (TD) methods~\citep{sutton2018reinforcement,haarnoja2018soft,kostrikov2022offline,park2025horizon}, Monte Carlo (MC) methods~\citep{sutton2018reinforcement,eysenbach2021c,eysenbach2022contrastive}, hierarchical methods~\citep{barto2003recent,lynch2019learning,park2023hiql,park2025horizon}, probabilistic methods~\citep{eysenbach2021c,zheng2024contrastive}, and quasimetric methods~\citep{kaelbling1993learning,dhiman2018floyd,jurgenson2020sub,wang2023optimal,pikekos2023efficient,myers2024learning,myers2025offline,park2026transitive}. 
\method\ builds on the quasimetric family, which we discuss next.

\paragraph{Quasimetrics in GCRL.}
In deterministic environments, GCRL reduces to learning a value function that estimates the shortest-path distance from any state to a goal. The induced optimal distance is a quasimetric (\ie, an asymmetric distance obeying the triangle inequality), a property that GCRL algorithms exploit in three ways. ``Explicit'' methods impose the triangle inequality through quasimetric-constrained value networks~\citep{wang2023optimal,myers2024learning,myers2025offline}. ``Implicit'' methods enforce it through triangle-inequality value backups~\citep{kaelbling1993learning,dhiman2018floyd,jurgenson2020sub,pikekos2023efficient,park2026transitive}. ``Planning''-based methods use it to compose shortest paths at inference time~\citep{eysenbach2019search,parascandolo2020divide,jurgenson2020sub}.

\method\ is most closely related to two ``implicit'' methods: Sub-goal tree dynamic programming (\texttt{TDP})~\citep{jurgenson2020sub} and Transitive RL (\trl)~\citep{park2026transitive}. \texttt{TDP} recursively predicts intermediate subgoals to minimize decomposition cost and learns top-down. Greedy selection can overestimate the chosen decomposition. \method\ instead fixes the subgoal to the trajectory midpoint and learns bottom-up.
\trl\ uses a similar in-trajectory backup, favoring optimistic splits via an upper-expectile loss. Because \trl\ learns values in no particular order, these optimistic updates rely on noisy estimates, amplifying overestimation through repeated propagation. This non-recursive update has \emph{linear} worst-case bootstrap depth over long horizons. \method\ instead schedules a recursive midpoint decomposition with exact factorization, yielding \emph{logarithmic} bootstrap depth and much slower empirical error growth.

\section{Preliminaries}
\paragraph{Problem Setting.}
We study GCRL in a deterministic controlled Markov process $\mathcal{M}=(\mathcal{S},\mathcal{A}, \mu, p, \gamma)$ comprising the state space, action space, initial-state distribution, deterministic transition function, and discount factor.
Any state can be a goal. We are given an unlabeled dataset $\mathcal{D}=\{\tau^{(i)}\}_{i=1}^N$, where the $i$-th trajectory $\tau^{(i)} = (s^{(i)}_0, a^{(i)}_0, s^{(i)}_1, \dots, s^{(i)}_{H_i})$ has length $H_i$.

Following prior work~\citep{wang2023optimal,park2026transitive}, we adopt the hitting-time formulation of the goal-conditioned value function: the agent receives a reward of $1$ upon first reaching $g$ and enters an absorbing state, so $V^\pi(s,g) = \mathbb{E}_\pi\!\left[\gamma^{\,T_g}\;\middle|\;s_0=s\right]$, where $T_g$ is the first hitting time of $g$ under a goal-conditioned policy $\pi$ and $\gamma\in(0,1)$ is the discount factor. The optimal value function is $V^\ast:=\max_\pi V^\pi$.
In deterministic environments, learning $V^\ast$ is equivalent to learning the optimal distance $d^\ast:=\log_\gamma V^\ast$, since $V^\ast=\gamma^{d^\ast}$. Here $d^\ast(s,g)$ is the minimum number of steps to reach $g$ from $s$ (\ie, the shortest-path distance).

\paragraph{Triangle Inequality in GCRL.}
The optimal distance $d^\ast$ satisfies the triangle inequality for all states $s, w, g \in \mathcal{S}$:
\begin{equation}
d^\ast(s,g) \;\le\; d^\ast(s,w) + d^\ast(w,g),
\label{eq:triangle_inequality}
\end{equation}
with equality when $w$ lies on a shortest path from $s$ to $g$. Through the relation $V^\ast = \gamma^{d^\ast}$, this is equivalent to a multiplicative triangle inequality on values:
\begin{equation}
V^\ast(s,g) \;\ge\; V^\ast(s,w)\cdot V^\ast(w,g).
\label{eq:value_triangle_inequality}
\end{equation}
Prior methods turn this property into a transitive backup $V(s,g)\gets \max_w V(s,w)\cdot V(w,g)$ to learn the value function by maximizing over intermediate states $w$~\citep{kaelbling1993learning,dhiman2018floyd,jurgenson2020sub,pikekos2023efficient,park2026transitive}.

\section{Recursive Value Learning for Long-Horizon Offline GCRL}
Goal-conditioned value learning with standard max-based backups (\eg, TD and transitive backups) suffers from error accumulation.
While every long-range value inherently depends on shorter-range values, standard GCRL methods do not leverage this dependency structure.
We address this by framing goal-conditioned value learning as a divide-and-conquer problem.
We propose \method\ (\texttt{D}ivide-and-\texttt{C}onquer RL), which recursively splits a trajectory segment and learns values from shorter segments to longer ones.
Thus, \method\ exhibits far slower empirical error accumulation than standard backups and achieves strong goal-reaching performance.

In \cref{sec:recursive_dnc}, we introduce recursive goal-conditioned value learning, which learns shorter-horizon (child) values before longer-horizon (parent) values. 
In \cref{sec:update_rule}, we describe the selection-free update rule of recursive value learning, which avoids maximizing over noisy intermediate states. Since this update rule is restricted to observed routes, it recovers only behavior values. To learn optimal values, \method\ also trains the same value function to compose segments across trajectories (\cref{sec:value_propagation}). In \Cref{sec:parallel,sec:practical_implementation}, we detail the practical implementation of \method.

\subsection{Recursive Divide-and-Conquer Value Learning}
\label{sec:recursive_dnc}

We introduce a recursive divide-and-conquer strategy to enforce child-to-parent ordering in goal-conditioned value learning. \method\ recursively splits each trajectory segment (\texttt{Divide}) and learns values from leaves to root (\texttt{Conquer}). For a trajectory segment between randomly sampled states $s_i$ and $s_j$, where $i<j$, we recursively bisect the segment at the midpoint $k=\lfloor{(i+j)/2}\rfloor$ to build a balanced binary tree (\cref{fig:divide_and_conquer}).
By splitting at the midpoint, a horizon-$H$ segment yields $O(H)$ tree nodes but only $O(\log_2 H)$ sequential composition levels.

\paragraph{Divide.}
For a root segment $(s_i,s_j)$, we recursively split each segment into two children, $(s_i,s_k)$ and $(s_k,s_j)$, where $k=\lfloor(i+j)/2\rfloor$, until reaching single-step leaves (\eg, $(s_i,s_{i+1})$).

\paragraph{Conquer.}
We train child segments before their parents. Single-step leaves are grounded (\ie, $V(s_i,s_{i+1})=\gamma$), and each parent is then updated from its previously trained children.

\begin{figure}[h]
    \centering
    \includegraphics[width=\linewidth]{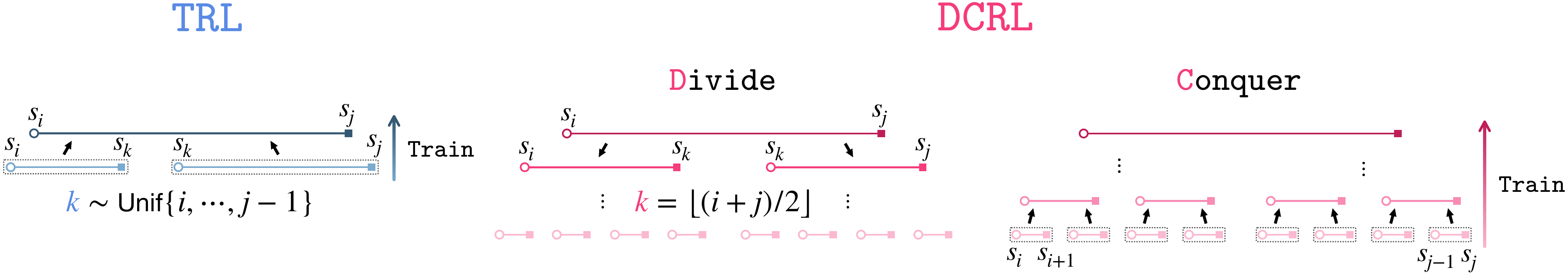}
        \caption{\method\ explicitly implements \textit{divide-and-conquer}. \texttt{\textcolor{highlight}{D}ivide}: Split a root segment $(s_i,s_j)$ at its midpoint, then recursively split each half until reaching single-step leaves. \texttt{\textcolor{highlight}{C}onquer}: Train the value function on leaves first and then on their parents, proceeding up to the root. In contrast, \trl\ backs up from a single random subgoal $s_k$, where $k \in \{ i,\ldots,j-1\}$.}
    \label{fig:divide_and_conquer}
\end{figure}

\subsection{Update Rule for Recursive Value Learning}
\label{sec:update_rule}
For recursive value learning, we seek a selection-free update that avoids amplifying errors through maximization over noisy intermediate states.
Unlike max-based backups, our update rule factorizes the value of a trajectory segment between $s$ and $g$ at its midpoint $w$:
\begin{equation}
V_\tau(s,g) = V_\tau(s,w)\cdot V_\tau(w,g).
\label{eq:pessimistic_value_backup}
\end{equation}
Because \Cref{eq:pessimistic_value_backup} uses a fixed intermediate state without maximization, $V_\tau$ learns behavior values~\citep{gulcehre2021regularized}. The corresponding distance $d_\tau(s,g) = d_\tau(s,w) + d_\tau(w,g)$ is the length of the trajectory's own route through $w$. It upper-bounds $d^\ast$: the trajectory's route need not be the shortest path, but it is supported by the data.
This behavior value function $V_\tau$ admits a useful route-suboptimality decomposition. We prove the following:
\begin{property}
\label{prop:superadditivity}
Let $d_\tau:=\log_\gamma V_\tau$, with $V_\tau$ satisfying the exact factorization in \emph{\Cref{eq:pessimistic_value_backup}}, and define the route suboptimality $e(s,g):=d_\tau(s,g)-d^\ast(s,g)$. Then,
\begin{align*}
  e(s,g) \ge e(s,w) + e(w,g),
\end{align*}
for \textbf{any} intermediate state $w$ on a trajectory segment from $s$ to $g$ (proof in Appendix~\ref{app:proof_prop_superadditivity}).
\end{property}

\Cref{prop:superadditivity} shows that route suboptimality can grow under composition: a parent route is at least as suboptimal as its two children combined. This motivates learning long-horizon behavior values recursively from shorter, less-suboptimal constituents. We therefore learn child values ($V_\tau(s,w),V_\tau(w,g)$) before using them to update their parent ($V_\tau(s,g)$), avoiding bootstrapping from unconverged child estimates. For recursive value learning (\cref{sec:recursive_dnc}), we choose $w$ as the segment midpoint, yielding a balanced binary tree.

\subsection{Value Propagation for Optimality}
\label{sec:value_propagation}
Recursive divide-and-conquer value learning yields only in-trajectory behavior values ($V_\tau$).
To recover optimal values ($V^\ast$), we additionally train the value function with multistep value-propagation.
\cref{fig:behavior_optimal_comparison} illustrates a didactic example of the discrepancy between $V_\tau$ and $V^\ast$, confirming that the optimal value often requires composing segments across trajectories.
Furthermore, the distribution of recursive tree samples is heavily skewed toward short segments.
Decomposing a horizon-$H$ root segment yields $2H-1$ tree nodes, of which $H$ are leaves. Models trained only on this distribution generalize poorly to distant relabeled goals.

\begin{wrapfigure}{r}{0.37\textwidth}
    \vspace{-10pt}
    \centering
    \includegraphics[width=\linewidth]{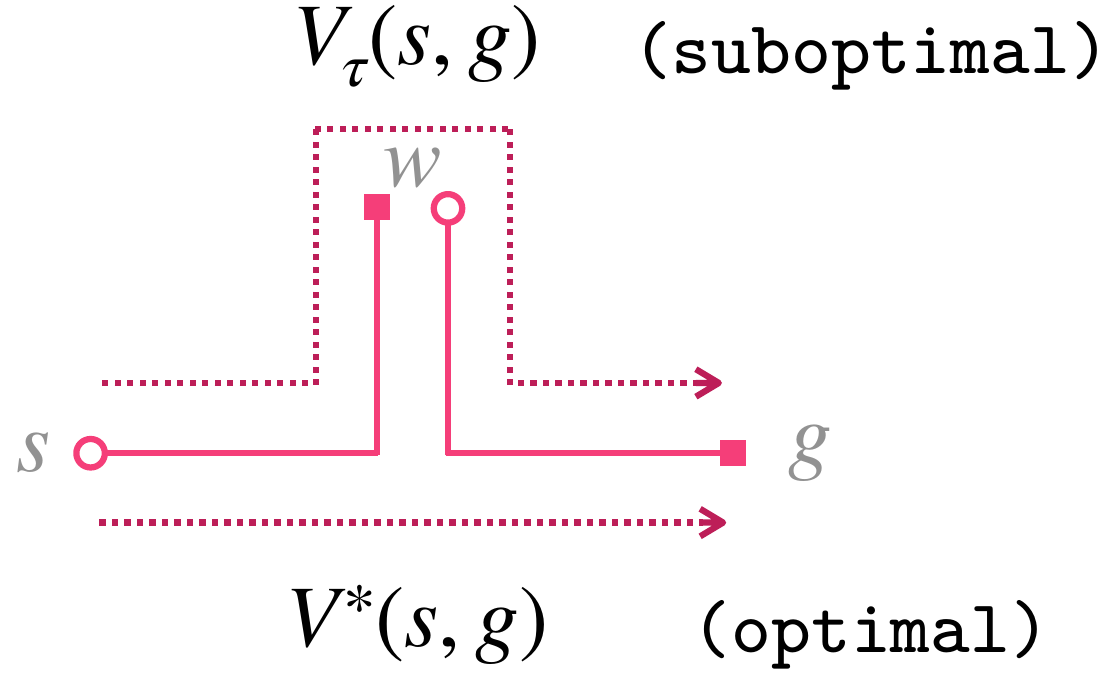}
    \caption{Optimal values require composing segments across trajectories. Given $s\rightarrow w$ and $w \rightarrow g$, the optimal route from $s$ to $g$ cannot be recovered from $V_\tau$ alone.} 
    \label{fig:behavior_optimal_comparison}
\end{wrapfigure}
Thus, we optimize our value function with an additional objective using i.i.d. samples with globally relabeled goals. Starting from the original triangle inequality (\Cref{eq:value_triangle_inequality}), we fix the intermediate subgoal $w$ to a future in-trajectory state $s_{i+n}$ with propagation horizon $n$:
\begin{equation}
    V^\ast(s_i,g) \ge \gamma^n \cdot V^\ast(s_{i+n},g).
    \label{eq:propagation_inequality}
\end{equation}
Because $s_{i+n}$ lies $n$ steps ahead on a demonstrated trajectory, $V^\ast(s_i,s_{i+n}) \ge \gamma^n$. \Cref{eq:propagation_inequality} motivates pushing $V(s_i, g)$ up toward $\gamma^n \cdot V(s_{i+n}, g)$, which is equivalent to multistep propagation~\citep{sutton2018reinforcement}. We optimize this objective using an upper-expectile loss because we seek the shortest route connecting $s_i$ to $g$. 

Importantly, divide-and-conquer supports $n$-step propagation by quickly learning long-range behavior values through a shallow dependency path. For a horizon-$H$ segment, divide-and-conquer has $O(\log_2 H)$ bootstrap depth, whereas $n$-step propagation has $O(H/n)$ depth. Long-range behavior values therefore emerge through fewer stages, giving a reliable basis for propagation.

\paragraph{Division of Labor.} 
The divide-and-conquer (\cref{sec:recursive_dnc}) and propagation (\cref{sec:value_propagation}) objectives train a shared value function but serve complementary roles: the former learns behavior values for within-trajectory pairs through a shallow dependency structure, whereas the latter composes segments across trajectories to recover optimal values. In \cref{tab:dcrl_ablation}, removing either objective degrades performance, confirming both roles are essential to \method.

\subsection{Implementation: Parallel Slot Scheduling}
\label{sec:parallel}
We introduce \textbf{slot scheduling} to efficiently execute recursive divide-and-conquer value learning.
Training on one tree at a time would be slow and produce highly correlated batches with limited trajectory and horizon diversity. To process multiple root segments concurrently, we maintain multiple independent slots. As illustrated in \cref{fig:sample_slot} and \cref{alg:sampleslot}, each slot manages the recursive tree of a single root segment. A slot $\mathcal{T}$ is initialized as a sequence of recursive samples ordered from leaves to root. Samples drawn via $\mathcal{T}.\texttt{pop()}$ are therefore processed from leaves to root, enforcing our principle: every child is trained before its parent.

\begin{figure}[ht]
    \begin{minipage}[t]{0.50\textwidth}
        \vspace{0pt}
        \captionsetup{width=\linewidth}
        \centering
        \includegraphics[width=\linewidth]{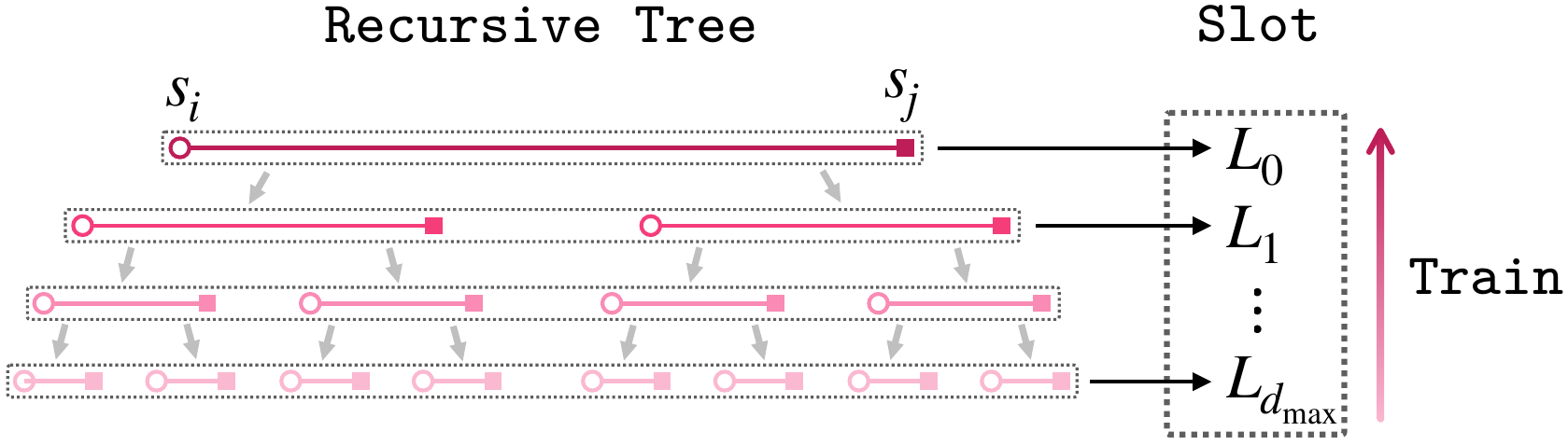}
        \caption{A \textbf{recursive tree} is flattened into a slot, ordered from leaves ($L_{d_{\max}}$) to the root ($L_0$), so every child is learned before its parent.}
        \label{fig:sample_slot}
        \vspace{-5pt}
    \end{minipage}\hfill
    \begin{minipage}[t]{0.48\textwidth}
        \vspace{-5pt}
        \captionsetup{width=\linewidth}
        \begin{algorithm}[H]
            \caption{$\texttt{SampleSlot}(\mathcal{D})$}
            \label{alg:sampleslot}
            \begin{algorithmic}[1]
                \STATE Initialize per-depth lists $\{L_d\}_{d=0}^{d_{\max}}$
                \STATE Sample $\tau= (s_0, a_0, s_1, \ldots, s_H) \sim \mathcal{D}$
                \STATE Sample $i<j \sim \text{Unif}\{0,\ldots,H\}$
                \STATE Build a \textbf{recursive tree} with $(s_i, \ldots, s_j)$ 
                \STATE \textbf{return} $\texttt{concat}(L_{d_{\max}}, \ldots, L_0)$
            \end{algorithmic}
        \end{algorithm}
        \vspace{-5pt}
    \end{minipage}
    
\end{figure}

During training, we maintain $S=128$ slots $\{\mathcal{T}_\ell\}_{\ell=1}^S$ (for batch size $1024$), drawing samples uniformly across them at each step. 
Because slots drain and refill independently, they become desynchronized, so later batches span a range of lengths while preserving child-to-parent ordering within each slot. We study how the number of slots ($S$) affects \method's performance in Appendix~\ref{app:number_of_slots} and compare \method\ with other divide-and-conquer strategies in \cref{sec:a3}.

\begin{figure}[ht]
    \centering
    \includegraphics[width=\linewidth]{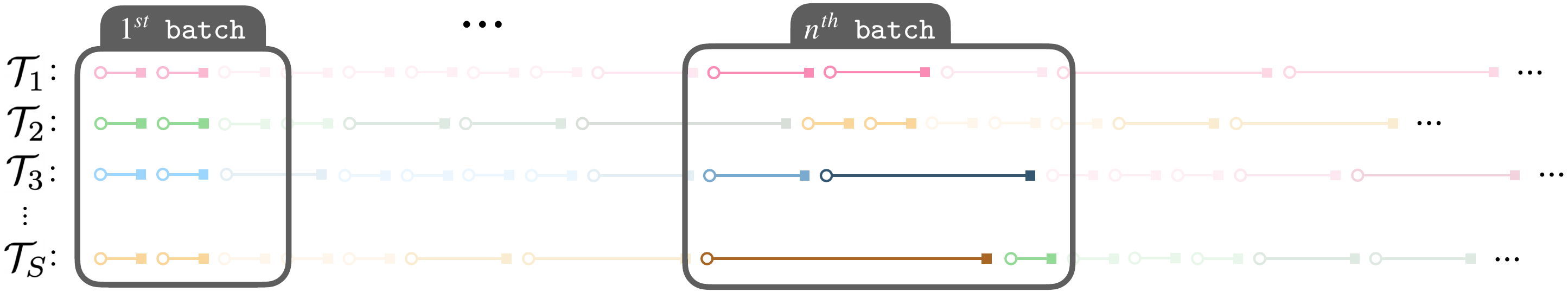}
    \caption{Maintaining $S$ independent \textbf{slots} (${\mathcal{T}_1,\ldots,\mathcal{T}_S}$) enables (i) learning child segments before their parents and (ii) mixing diverse sample lengths within batches. Each slot begins by emitting single-step leaves, as illustrated by the $1^{\text{st}}$ batch; as slots drain and refill at different rates, they become desynchronized, mixing short-horizon and long-horizon segments, as in the $n^{\text{th}}$ batch. Each slot schedules a new tree once its current one is consumed.}
    \label{fig:data_stream}
\end{figure}

\subsection{Implementation: Critic Learning}
\label{sec:practical_implementation}
In practice, we train the action-value function $Q(s, a, g): \mathcal{S} \times \mathcal{A} \times \mathcal{S} \to \left[0,1\right]$ by minimizing a divide-and-conquer loss and a propagation loss simultaneously:
\begin{equation}
    \mathcal{L}^{\method}(Q) = \mathcal{L}^{\texttt{dc}}(Q) + \mathcal{L}^{\texttt{prop}}(Q).
    \label{eq:dcrl_loss}
\end{equation}
\paragraph{Divide-and-Conquer Loss.} $\mathcal{L}^{\texttt{dc}}(Q)$ learns behavior values from recursive tree samples $\tau_{\text{DC}}$ using the exact value factorization:
\begin{equation}
    \mathcal{L}^{\texttt{dc}}(Q)
    = \mathbb{E}_{\tau_\text{DC}\sim \bm{\mathcal{T}}}\!\Bigl[
    \mathcal{L}_\text{BCE}\bigl(
    Q(s_i,a_i,s_j),\;
    \operatorname{\texttt{sg}}[Q(s_i,a_i,s_k)]\,
    \bar{Q}(s_k,a_k,s_j)
    \bigr)\Bigr],
    \label{eq:dc_loss}
\end{equation}
where $k=\lfloor{(i+j)/2}\rfloor$, $\operatorname{\texttt{sg}}$ denotes stop-gradient, $\bar{Q}$ is the EMA target critic, and $\mathcal{L}_\text{BCE}$ is the binary cross-entropy loss.
To obtain each divide-and-conquer sample $\tau_\text{DC}$, we uniformly select a slot $\mathcal{T}\in\{\mathcal{T}_1,\ldots,\mathcal{T}_S\}$ and pop its next sample. We evaluate the first child value $Q(s_i,a_i,s_k)$ using the online critic and the second child value $\bar{Q}(s_k,a_k,s_j)$ using the EMA critic. Either child value is replaced by $\gamma$ when its segment is a single step. In Appendix~\ref{app:target_value_source}, we ablate the critic used for each child value (online or EMA) and find that using the online critic for the first and the EMA critic for the second (\ie, $\operatorname{\texttt{sg}}[Q]\cdot \bar{Q}$) performs best.

\paragraph{Propagation Loss.} $\mathcal{L}^{\texttt{prop}}(Q)$ learns optimal values from globally relabeled samples $\tau_{\text{prop}}$ using multistep propagation:
\begin{equation}
    \mathcal{L}^{\texttt{prop}}(Q)
    = \mathbb{E}_{\tau_\text{prop}\sim \bm{\mathcal{D}}}\!\Bigl[\mathcal{L}^{\kappa_{\texttt{prop}}}_\text{BCE}\bigl(Q(s_i, a_i, g),\; \gamma^n\cdot\bar{Q}(s_{i+n}, a_{i+n}, g)\bigr)\Bigr],
    \label{eq:prop_loss}
\end{equation}
where $\mathcal{L}^{\kappa_{\texttt{prop}}}_\text{BCE}$ is the BCE loss with expectile $\kappa_{\texttt{prop}} \in (0.5, 1)$, and $g$ is a future, random, or current state.
For an in-trajectory goal $g=s_j$ within $n$ steps of $s_i$ ($j-i\le n$), we use $\gamma^{j-i}$ instead of $\gamma^n\cdot \bar{Q}$.
The propagation horizon $n$ is task-specific and studied in \cref{sec:a2}.

The two losses differ by design: $\mathcal{L}^{\texttt{dc}}$ learns behavior routes with plain BCE ($\kappa_{\texttt{dc}}=0.5$), whereas $\mathcal{L}^{\texttt{prop}}$ favors shorter routes through upper-expectile BCE ($\kappa_{\texttt{prop}}>0.5$), since the optimal value requires the shortest of many routes.
We ablate $\kappa_{\texttt{dc}}$ and $\kappa_{\texttt{prop}}$ in Appendix~\ref{app:expectile}; performance peaks at $\kappa_{\texttt{dc}}=0.5$, consistent with optimistic weighting introducing bias when learning behavior routes.
Appendix~\ref{app:algorithm} summarizes the full training procedure for \method.

\section{Experiments}
\label{sec:experiments}
Our experiments and analysis address three questions:
\begin{itemize}[leftmargin=*]
    \item \textbf{Q1:} How does \method\ compare with prior offline GCRL methods across diverse task domains, horizons, dataset scales, and observation modalities? ($\rightarrow$ \cref{sec:a1})
    \item \textbf{Q2:} How does \method\ scale with the propagation horizon $n$ and the task horizon $H$ compared with standard Bellman-style and transitive backups? ($\rightarrow$ \cref{sec:a2})
    \item \textbf{Q3:} How does \method\ compare with alternative divide-and-conquer strategies in terms of goal-reaching performance and training efficiency? ($\rightarrow$ \cref{sec:a3})
\end{itemize}

We compare \method\ against prior offline GCRL algorithms spanning four value-learning families:
\begin{itemize}[leftmargin=*]
    \item \textbf{Temporal Difference}: \texttt{IQL}~\citep{kostrikov2022offline}, \texttt{SAC+BC}~\citep{haarnoja2018soft, park2025horizon}, \texttt{TD}, \texttt{TD-n}~\citep{sutton2018reinforcement}
    \item \textbf{Monte Carlo}: \texttt{CRL}~\citep{eysenbach2022contrastive}, \texttt{MC}~\citep{tian2020model, shah2021rapid}
    \item \textbf{Hierarchical}: \texttt{POR}~\citep{xu2022policy}, \texttt{HIQL}~\citep{park2023hiql}, \texttt{SHARSA}~\citep{park2025horizon}
    \item \textbf{Triangle Inequality}: \texttt{QRL}~\citep{wang2023optimal}, \texttt{TMD}~\citep{myers2025offline},  \texttt{TDP}~\citep{kaelbling1993learning, jurgenson2020sub}, \texttt{COE}~\citep{pikekos2023efficient}, \texttt{TRL}~\citep{park2026transitive}
\end{itemize}

We present 95\% confidence intervals in plots and standard deviations in tables.
In tables, any values within 95\% of the best performance are marked in \textcolor{highlight}{red}.
Descriptions of each baseline, experimental details, and full results are provided in Appendices~\ref{app:baseline}, \ref{app:experimental_setup}, and \ref{app:experimental_results}, respectively.

\subsection{How Well Does \method\ Compare with Prior Offline GCRL Methods?}
\label{sec:a1}

We evaluate \method\ and prior offline GCRL methods on goal-reaching tasks spanning diverse domains, horizons, dataset scales, and observation modalities (state and pixel).
First, \cref{tab:large_scale_performance} compares these methods on the five most challenging long-horizon tasks from OGBench~\citep{park2025ogbench}.
\method\ achieves the best average score, outperforming all flat and hierarchical baselines. Notably, \method\ is the only flat method to achieve nonzero success on \texttt{cube-octuple}, while attaining the highest scores on \texttt{humanoidmaze-giant} and both puzzle tasks.

\begin{table}[ht]
    \centering
    \caption{\method\ achieves the best average performance on \textbf{long-horizon} tasks.}
    \label{tab:large_scale_performance}
    \resizebox{\textwidth}{!}{%
        \begin{tabular}{@{}lcccccccccccccc@{}}
            \toprule
                                        &                    & \multicolumn{4}{c}{\textbf{Temporal Difference}} & \multicolumn{2}{c}{\textbf{Monte Carlo}} & \multicolumn{2}{c}{\textbf{Hierarchical}} & \multicolumn{5}{c}{\textbf{Triangle Inequality}}                                                                                                                                                                                                                                                    \\
            \cmidrule(lr){3-6} \cmidrule(lr){7-8} \cmidrule(lr){9-10} \cmidrule(lr){11-15}
            \texttt{Environment}        & \texttt{FBC}       & \texttt{IQL}                                     & \texttt{SAC+BC}                                  & \texttt{TD}                               & \texttt{TD-n}\                                   & \texttt{CRL}         & \texttt{MC}         & \texttt{HIQL}        & \texttt{SHARSA}     & \texttt{QRL}       & \texttt{TDP}       & \texttt{COE}       & \texttt{TRL}                               & \textcolor{highlight}{\method}              \\ \midrule
            \texttt{humanoidmaze-giant-1B} & \texttt{0 \std{0}} & \texttt{3 \std{4}}                               & \texttt{5 \std{0}}                            & \texttt{3 \std{2}}                        & \texttt{78 \std{4}}                              & \texttt{62 \std{42}} & \texttt{79 \std{4}} & \texttt{25 \std{6}} & \texttt{42 \std{3}} & \texttt{3 \std{2}} & \texttt{2 \std{0}} & \texttt{2 \std{1}} & \texttt{79 \std{2}}                        & \texttt{\textcolor{highlight}{93} \std{3}}  \\
            \texttt{puzzle-4x5-1B}         & \texttt{0 \std{0}} & \texttt{20 \std{0}}                              & \texttt{19 \std{1}}                           & \texttt{19 \std{1}}                       & \texttt{92 \std{8}}                              & \texttt{1 \std{1}}   & \texttt{47 \std{7}} & \texttt{12 \std{5}}   & \texttt{90 \std{4}} & \texttt{0 \std{0}} & \texttt{0 \std{0}} & \texttt{0 \std{0}} & \texttt{\textcolor{highlight}{97} \std{1}} & \texttt{\textcolor{highlight}{100} \std{0}} \\
            \texttt{puzzle-4x6-1B}         & \texttt{0 \std{0}} & \texttt{17 \std{2}}                              & \texttt{11 \std{8}}                           & \texttt{13 \std{3}}                       & \texttt{53 \std{11}}                             & \texttt{0 \std{0}}   & \texttt{37 \std{4}} & \texttt{4 \std{3}}   & \texttt{55 \std{3}} & \texttt{0 \std{0}} & \texttt{0 \std{0}} & \texttt{0 \std{0}} & \texttt{51 \std{5}}                        & \texttt{\textcolor{highlight}{87} \std{5}}  \\
            \texttt{cube-quadruple-100M}     & \texttt{1 \std{0}} & \texttt{41 \std{3}}                              & \texttt{39 \std{4}}                           & \texttt{20 \std{2}}                       & \texttt{21 \std{6}}                              & \texttt{22 \std{2}}   & \texttt{1 \std{1}} & \texttt{62 \std{6}}   & \texttt{\textcolor{highlight}{74} \std{2}} & \texttt{0 \std{0}} & \texttt{0 \std{0}} & \texttt{0 \std{0}} & \texttt{0 \std{0}}                        & \texttt{37 \std{4}} \\
            \texttt{cube-octuple-1B}         & \texttt{0 \std{0}} & \texttt{0 \std{0}}                              & \texttt{0 \std{0}}                           & \texttt{0 \std{0}}                       & \texttt{0 \std{0}}                             & \texttt{0 \std{0}}   & \texttt{0 \std{0}} & \texttt{0 \std{0}}   & \texttt{\textcolor{highlight}{15} \std{1}} & \texttt{0 \std{0}} & \texttt{0 \std{0}} & \texttt{0 \std{0}} & \texttt{0 \std{0}}                        & \texttt{5 \std{1}} \\
            \midrule
            \texttt{Average}            & \texttt{\textbf{0}}         & \texttt{\textbf{16}}                                      & \texttt{\textbf{15}}                                      & \texttt{\textbf{11}}                               & \texttt{\textbf{49}}                                      & \texttt{\textbf{17}}          & \texttt{\textbf{33}}         & \texttt{\textbf{21}}          & \texttt{\textbf{55}}         & \texttt{\textbf{1}}         & \texttt{\textbf{0}}         & \texttt{\textbf{0}}         & \texttt{\textbf{45}}                                & \texttt{\textcolor{highlight}{\textbf{64}}}          \\
            \bottomrule
        \end{tabular}%
    }
\end{table}

Next, \cref{fig:ogbench_overall} compares aggregate performance on standard-horizon OGBench tasks (10 state-based and 10 pixel-based). \method\ performs best in both modalities.

\begin{figure}[ht]
    \centering
    \includegraphics[width=1.0\linewidth]{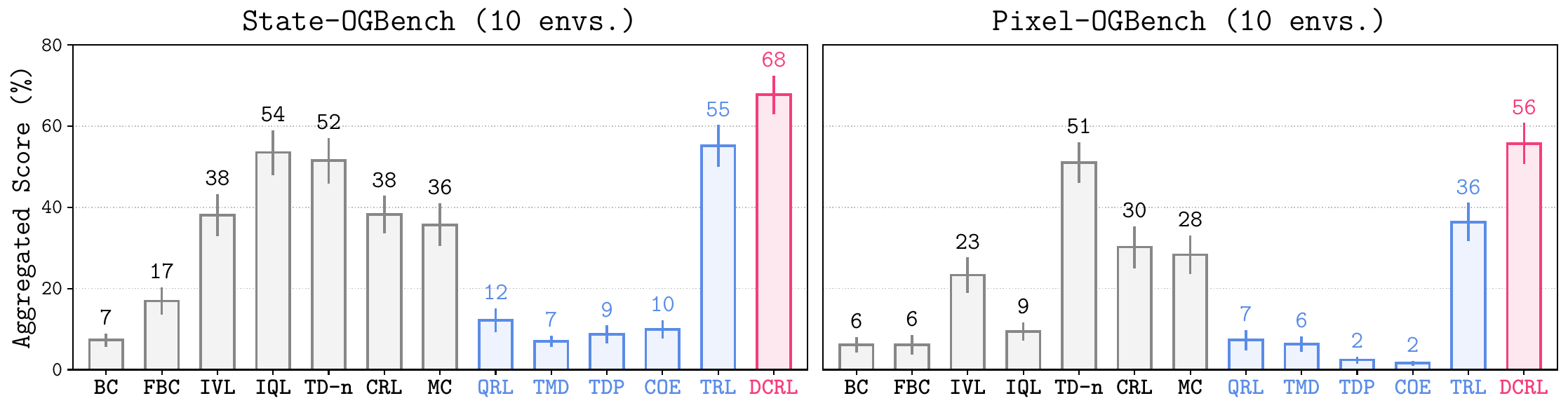}
    \caption{\method\ achieves state-of-the-art performance on standard OGBench tasks across diverse task \textbf{domains} and observation \textbf{modalities}. Full results are in Tables~\ref{tab:standard_performance} and \ref{tab:visual_performance}.}
    \label{fig:ogbench_overall}
\end{figure}

Finally, we test \method\ and prior methods on CALVIN~\citep{mees2022calvin}, a compositional, long-horizon manipulation environment.
We use the in-domain setup, where four subtasks must be solved in sequence using 1,239 task-agnostic trajectories spanning 34 subtasks.
\cref{fig:calvin_results} reports \texttt{Return} (subtasks completed), \texttt{Completion} (full-task success), and task progress.
Notably, \method\ is the only flat method that completes all four subtasks consecutively.

\begin{figure}[ht]
    \centering
    \includegraphics[width=1.0\linewidth]{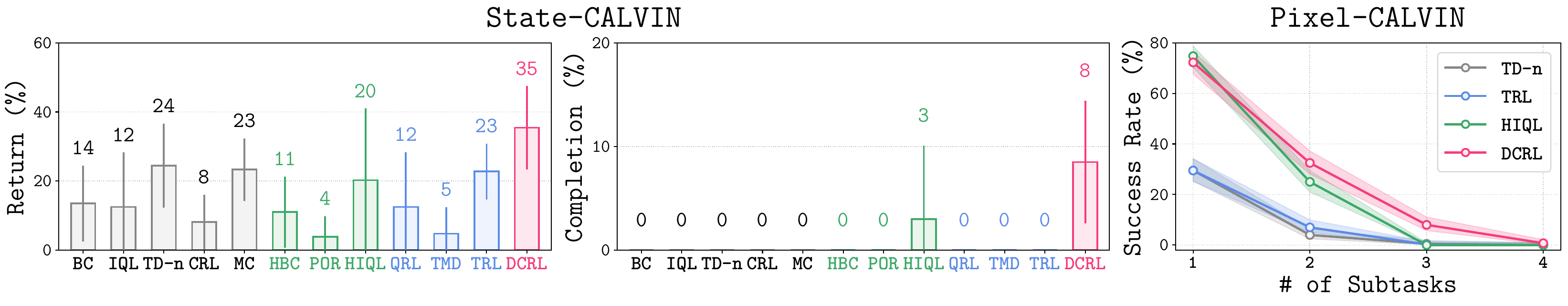}
    \caption{\method\ achieves the highest four-subtask completion rate on state-based CALVIN (8\%), versus 3\% for the best hierarchical baseline and 0\% for all flat baselines. On pixel-based CALVIN, which excludes privileged object poses, \method\ reaches at least three subtasks in 8\% of rollouts and all four in 1\%, whereas all three baselines remain at 0\% beyond two subtasks.}
    \label{fig:calvin_results}
\end{figure}

\subsection{Is \method\ Better Than \td\ and \trl\ Across Propagation and Task Horizons?}
\label{sec:a2}

We first study how the propagation horizon $n$ affects \method's performance.
While \method\ relies on value propagation to recover optimal values, this mechanism can also be applied to \trl. 
We therefore compare \texttt{DCRL-n} against \texttt{TD-n} and \texttt{TRL-n} (\trl\ with $n$-step propagation) across a range of horizons $n$, isolating the effect of \method's recursive value learning.

\cref{fig:result_prop_h} (left) shows the results on two long-horizon tasks across $n\in\{1,5,10,25,50,100\}$.
\texttt{DCRL-n} achieves the best performance at larger $n$ ($\ge 25$) on both tasks, while \texttt{TD-n} performs especially poorly on the puzzle task and \texttt{TRL-n} struggles on the cube-manipulation task.

At $n\in\{1,5\}$, \texttt{DCRL-n} performs similarly to \texttt{TD-n} and worse than \texttt{TRL-n}. With small $n$, \texttt{DCRL-n}'s propagation chain is long ($O(H/n)$), so it cannot fully leverage its recursively learned long-horizon behavior values to recover optimal values. For $n\ge25$, \texttt{DCRL-n} maintains a clear margin over both baselines. See \cref{fig:nstep_comparison_exhaustive} for the full comparison.

To understand how error accumulation scales with task horizon $H$, we analyze each method's bootstrap depth. In Appendix~\ref{app:error_accumulation_theory}, we prove that \texttt{DCRL-n} has \textbf{logarithmic} bootstrap depth in $H$, whereas \texttt{TD-n} and \texttt{TRL-n} have \textbf{linear} worst-case depth. We then test whether this shallower dependency structure yields slower error growth under a fixed training budget.

For this controlled comparison, we use a discrete environment (\texttt{combination-lock})~\citep{park2025horizon} to measure error accumulation.
In \cref{fig:result_prop_h} (right), we report mean long-range Q-error ($|\hat{Q}-Q^\ast|$) across task horizons $H\in \{256,\ldots,2048\}$ at a fixed $n=64$. \texttt{DCRL-n} remains relatively flat while \texttt{TD-n} and \texttt{TRL-n} accumulate error rapidly. Appendix~\ref{app:error_accumulation_empirical} further separates the effects of $n$-step propagation, expectile selection, and recursive scheduling.

\begin{figure}[ht]
  \centering
  \includegraphics[width=\linewidth]{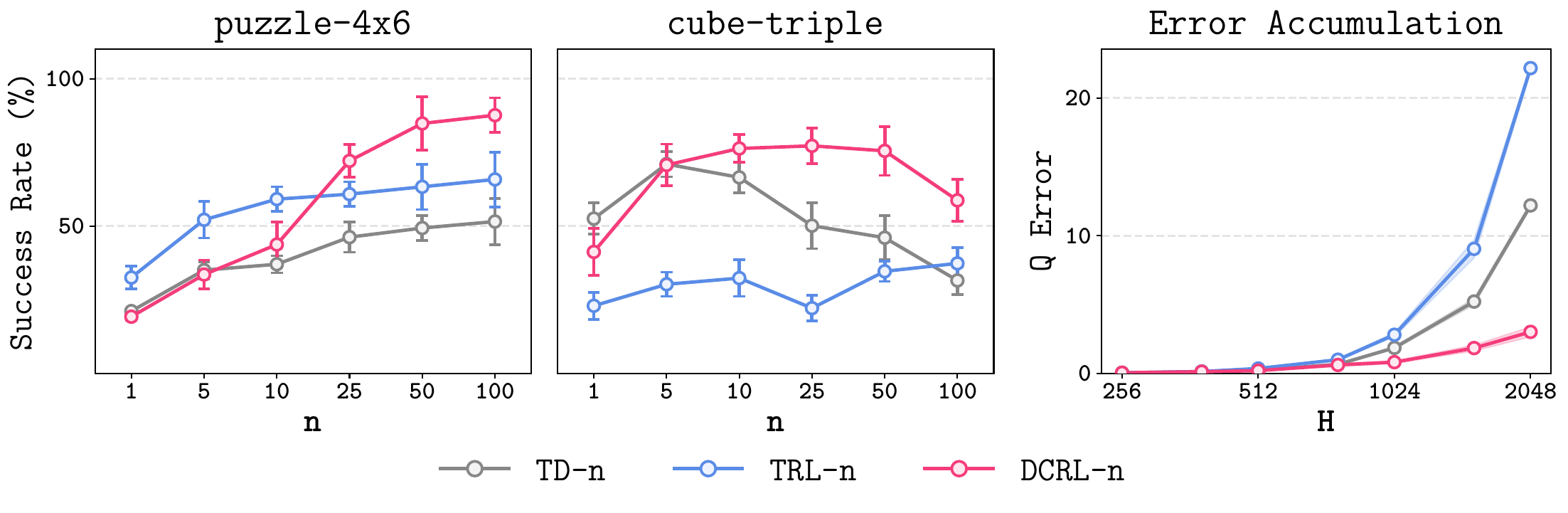}
  \caption{Left: \texttt{DCRL-n} surpasses \texttt{TD-n} and \texttt{TRL-n} at long \textbf{propagation horizons} ($n\ge 25$). Right: Q-error in \texttt{DCRL-n} grows far more slowly with \textbf{task horizon} $H$ than in the others.}
  \label{fig:result_prop_h}
\end{figure}

\subsection{Comparison with Alternative Divide-and-Conquer Strategies}
\label{sec:a3}
We compare \method\ with alternative strategies for organizing or modifying divide-and-conquer value learning. Descriptions and more experiments are in Appendices~\ref{app:dnc_strategies} and \ref{app:a3_c3_comparison}, respectively.

\begin{itemize}[leftmargin=*]
    \item \texttt{Curriculum} {\small ($\Delta$)}: learns values on short segments first using $H/\Delta$ stages.
    \item \trl\ {\small ($\lambda$)}: up-weights short state--goal pairs and down-weights long ones in the value loss.
    \item \method\ {\small (w/ random split)}: uses a random in-trajectory subgoal rather than the midpoint.
    \item \method\ {\small (w/ reverse order)}: learns recursive values in reverse (parent-to-child) order.
    \item \method\ {\small (w/o recursion)}: replaces recursion with i.i.d.\ sampling and adds one-step grounding.
    \item \method\ {\small (w/o propagation)}: uses recursive value learning alone to recover optimal values.
    \item \method\ {\small (w/ max-based backup)}: uses $V(s,g) \gets \max_w V(s,w)\cdot V(w,g)$ to select subgoals.
\end{itemize}

\noindent
\begin{minipage}[t]{0.48\textwidth}
    \centering
    \captionsetup{width=\linewidth}
    \captionof{table}{Comparison of divide-and-conquer strategies in \texttt{humanoidmaze-giant}.}
    \label{tab:a3_comparison}
    \resizebox{\linewidth}{!}{%
        \begin{tabular}{lcc}
            \toprule
            \texttt{Method} & \texttt{Success} $\uparrow$ & \texttt{Throughput} $\uparrow$ \\
            \midrule
            \texttt{Curriculum} \small{($\Delta=25$)} & \texttt{28 \std{7}} & \texttt{96 \std{2}} \\
            \texttt{Curriculum} \small{($\Delta=50$)} & \texttt{39 \std{8}} & \texttt{99 \std{4}} \\
            \texttt{Curriculum} \small{($\Delta=100$)} & \texttt{29 \std{5}} & \texttt{99 \std{2}} \\
            \midrule
            \trl\ \small{($\lambda=0$)} & \texttt{54 \std{5}} & \texttt{\textcolor{highlight}{110} \std{1}} \\
            \trl\ \small{($\lambda=0.01$)} & \texttt{53 \std{2}} & \texttt{\textcolor{highlight}{107} \std{1}} \\
            \trl\ \small{($\lambda=0.1$)} & \texttt{52 \std{3}} & \texttt{\textcolor{highlight}{106} \std{2}} \\
            \midrule
            \method\ \small{(w/ random split)} & \texttt{\textcolor{highlight}{92} \std{1}} & \texttt{66 \std{1}} \\
            \method\ \small{(w/ reverse order)} & \texttt{18 \std{6}} & \texttt{82 \std{4}} \\
            \method\ & \texttt{\textcolor{highlight}{93} \std{3}} & \texttt{84 \std{3}} \\
            \bottomrule
        \end{tabular}%
    }
\end{minipage}\hfill
\begin{minipage}[t]{0.48\textwidth}
    \cref{tab:a3_comparison} reports success and throughput (gradient steps per second) at a fixed batch size of 1024.
    \method\ achieves over $1.7\times$ \trl's success rate at 24\% lower throughput. Random splitting creates deeper, less balanced trees and has lower throughput, but performance remains unchanged, likely because exact factorization and child-to-parent scheduling are retained.
    Reversing the learning order reduces success from 93\% to 18\%, consistent with parents bootstrapping from not-yet-updated child estimates.
    Curriculum learning is highly sensitive to $\Delta$, and its best performance trails \method\ and \trl. Finally, \trl's distance re-weighting does not improve the unweighted baseline.
\end{minipage}

\clearpage

\begin{wraptable}{l}{0.48\textwidth}
    \centering
    \captionsetup{width=\linewidth}
    \captionof{table}{Ablation studies of \method\ on state-based OGBench (10 environments).}
    \label{tab:dcrl_ablation}
    \resizebox{\linewidth}{!}{%
        \begin{tabular}{lcc}
            \toprule
            \texttt{Method} & \texttt{Success} $\uparrow$ & \texttt{Throughput} $\uparrow$ \\
            \midrule
            \method & \texttt{\textcolor{highlight}{68} \std{2}} & \texttt{\textcolor{highlight}{193} \std{3}} \\
            \method\ \small{(w/o recursion)} & \texttt{39 \std{3}} & \texttt{173 \std{1}} \\
            \method\ \small{(w/o propagation)} & \texttt{45 \std{2}} & \texttt{\textcolor{highlight}{201} \std{5}} \\
            \method\ \small{(w/ max-backup)} & \texttt{62 \std{2}} & \texttt{157 \std{2}} \\
            \bottomrule
        \end{tabular}%
    }
\end{wraptable}

\cref{tab:dcrl_ablation} ablates \method's recursion, propagation, and value-update rule. Removing recursion causes the largest performance drop, highlighting the recursive tree's central role. Removing propagation sharply degrades performance, showing that one optimistic divide-and-conquer objective cannot replace separate behavior-value learning and optimal-value propagation. Max-based backups underperform exact factorization, consistent with selection-induced overestimation, and lower throughput due to subgoal-search overhead. Overall, these ablations validate \method's division of labor: recursive factorization learns reliable behavior values, while propagation recovers optimality.

\section{Conclusion}
We introduce \method, which exploits a natural dependency structure in goal-conditioned value learning: longer-segment values are built from shorter ones.
\method\ realizes recursive value learning through a balanced binary tree over in-trajectory segments, learns behavior values from leaves to root, and propagates them across trajectories toward optimality.

Across offline GCRL benchmarks, \method\ substantially outperforms prior flat methods and surpasses hierarchical methods on several challenging long-horizon tasks. By mitigating horizon-induced error accumulation, \method\ provides a foundation for scaling GCRL to longer horizons.

\paragraph{Limitations.}
Like other triangle-inequality methods, \method\ assumes deterministic dynamics. Under stochastic dynamics, exact factorization generally fails because the expectation of a product need not factorize (\ie, $\mathbb{E}[XY] \neq \mathbb{E}[X]\cdot \mathbb{E}[Y]$).
Additionally, hierarchical approaches outperform \method\ on cube-manipulation tasks, despite these tasks having shorter horizons than \texttt{humanoidmaze-giant}, where \method\ excels. Thus, their difficulty does not stem solely from horizon length. Extending recursive sampling to stochastic dynamics and understanding when hierarchy remains beneficial are important directions for future work.

\bibliography{references}
\bibliographystyle{plainnat}

\appendix
\clearpage

\section{Theoretical Analysis}
\label{app:theory}

\subsection{Proof of \cref{prop:superadditivity}}
\label{app:proof_prop_superadditivity}
\begin{proof}
\begin{align*}
    d^\ast(s,g) + e(s,g) &= d_\tau(s,g) \\
                         &= d_\tau(s,w) + d_\tau(w,g) \\
                         &= \left[ d^\ast(s,w) + e(s,w) \right] + \left[ d^\ast(w,g) + e(w,g) \right] \\
                         &= \left[ d^\ast(s,w) + d^\ast(w,g) \right] + \left[ e(s,w) + e(w,g) \right] \\
                         &\ge d^\ast(s,g) + e(s,w) + e(w,g)
\end{align*}
Therefore, we have $e(s,g) \ge e(s,w) + e(w,g)$.
\end{proof}

\subsection{Error Accumulation Analysis}
\label{app:error_accumulation}

We analyze error accumulation theoretically and empirically. The theoretical subsection (Appendix~\ref{app:error_accumulation_theory}) formalizes backup dependencies through bootstrap depth and derives upper error bounds for \texttt{TD-n}, \texttt{TRL-n}, and \texttt{DCRL-n}, highlighting the effects of selection-free factorization and child-to-parent scheduling. The empirical subsection (Appendix~\ref{app:error_accumulation_empirical}) tests these predictions in a controlled \texttt{combination-lock} environment by measuring long-range error across horizons and ablating the recursive schedule.

\subsubsection{Theoretical Validation}
\label{app:error_accumulation_theory}

We compare \texttt{DCRL-n} with Bellman-style and transitive-backup methods (\texttt{TD-n}, \texttt{TRL-n}). Both \texttt{DCRL-n} and \texttt{TRL-n} include the same additional $n$-step propagation objective, whereas vanilla \trl\ uses only its transitive objective.

For the dependency analysis, each sample $p=(s_i,s_j)$ denotes the critic value associated with $(s_i,a_i,s_j)$, where $a_i$ is the dataset action at $s_i$. We suppress this action argument and write
$Q_\bullet(p):=Q_\bullet(s_i,a_i,s_j)$. For $\gamma\in(0,1)$, define the corresponding distance estimate
\[
\hat d_\bullet(p):=\log_\gamma Q_\bullet(p),
\qquad
e_\bullet(p):=\hat d_\bullet(p)-d^\ast(p),
\]
where $d^\ast(p):=\log_\gamma Q^\ast(s_i,a_i,s_j)$ is the
optimal distance conditioned on first taking the dataset action $a_i$. Let
$\mathcal{P}=\{(s_i,s_j):0\le i\le j\le T\}$ be the set of in-trajectory pairs and let $\mathcal{P}_\circ\subset\mathcal{P}$ denote the grounded pairs.
For the structural analysis, pairs within \(n\) steps are treated as grounded because their targets require no further bootstrapping: \texttt{TD-n} grounds them through \(n\)-step returns, while the shared propagation objective provides the same grounding to \texttt{DCRL-n} and \texttt{TRL-n}. Vanilla \trl\ grounds only one-step pairs. The error bounds below additionally idealize these grounded targets as exact relative to \(d^\ast\); the bootstrap-depth results require only that recursion terminates at them.

A possibly stochastic backup operator
$B:\mathcal{P}\setminus\mathcal{P}_\circ\to2^{\mathcal{P}}$
maps each non-grounded sample $p=(s,g)$ to the learned constituents in its bootstrap target:
\begin{equation*}
  B_{\texttt{TD-n}}(s_i,g)=\{(s_{i+n},g)\},\qquad
  B_{\texttt{DCRL-n}}(s_i,s_j)=B_{\texttt{TRL-n}}(s_i,s_j)=\{(s_i,s_k),(s_k,s_j)\},
\end{equation*}
where $k=\lfloor(i+j)/2\rfloor$ for \texttt{DCRL-n} and
$k\sim\mathrm{Unif}\{i+1,\ldots,j-1\}$ for \texttt{TRL-n}.
\trl's original sampling also permits $k=i$; because this yields an identity decomposition that does not advance the dependency graph, we condition on the effective nontrivial backup $i<k<j$.

\begin{definition}
\label{def:dag_main}
For a root sample $p$ and a realization $\omega$ of all stochastic backup choices, the \emph{\textbf{realized dependency graph}} $G_B(p;\omega)$ is obtained by recursively adding an edge $(p',q)$ for every $q\in B(p';\omega)$, stopping when $q\in\mathcal{P}_\circ$. For \emph{\texttt{TRL-n}}, this graph represents the recursive unrolling of successive random backups, not an explicitly stored tree.
\end{definition}

\begin{definition}
\label{def:depth_main}
The realized bootstrap depth $\mathrm{D}_B(p;\omega)$ is the longest path from $p$ to a grounded leaf in $G_B(p;\omega)$. For a stochastic operator, define
\[
\mathrm{D}^{\mathrm{avg}}_B(p)
:=\mathbb{E}_\omega[\mathrm{D}_B(p;\omega)],
\qquad
\mathrm{D}^{\max}_B(p)
:=\sup_\omega\mathrm{D}_B(p;\omega).
\]
For deterministic operators, we simply write $\mathrm{D}_B(p)$.
\end{definition}

\begin{proposition}
\label{prop:depth_main}
Given a non-grounded sample $p=(s_i,s_j)$ with $H:=j-i$ and fixed $n$,
\begin{align*}
  \mathrm{D}_{\emph{\texttt{TD-n}}}(p)
    &=\lceil H/n\rceil-1,\\
  \mathrm{D}_{\emph{\texttt{DCRL-n}}}(p)
    &=\lceil\log_2(H/n)\rceil,\\
  \mathrm{D}^{\mathrm{avg}}_{\emph{\texttt{TRL-n}}}(p)
    &=\Theta(\log(H/n)),&
  \mathrm{D}^{\max}_{\emph{\texttt{TRL-n}}}(p)
    &=\Theta(H-n),\\
  \mathrm{D}^{\mathrm{avg}}_{\emph{\trl}}(p)
    &=\Theta(\log H),&
  \mathrm{D}^{\max}_{\emph{\trl}}(p)
    &=\Theta(H).
\end{align*}
\end{proposition}

\begin{proof}[Proof sketch]
Let $D_\bullet(L)$ denote the bootstrap depth of a segment of length $L$, with $D_\bullet(L)=0$ for $L\le n$.

For \texttt{TD-n}, each backup reduces the remaining length by $n$, so
\[
D_{\texttt{TD-n}}(L)
=
1+D_{\texttt{TD-n}}(L-n),
\]
which gives
$D_{\texttt{TD-n}}(H)=\lceil H/n\rceil-1$.

For \texttt{DCRL-n}, midpoint splitting yields
\[
D_{\texttt{DCRL-n}}(L)
=
1+\max\!\left\{
D_{\texttt{DCRL-n}}(\lfloor L/2\rfloor),
D_{\texttt{DCRL-n}}(\lceil L/2\rceil)
\right\}.
\]
After $d$ levels, the longest remaining segment has length
$\lceil H/2^d\rceil$. The smallest $d$ for which this length is at most $n$ is
$d=\lceil\log_2(H/n)\rceil$.

For \texttt{TRL-n}, recursively choosing uniformly random splits yields the same partition structure as a random binary search tree, stopped when every segment has length at most $n$. For fixed $n$, this truncation does not change the asymptotic height, so standard random-tree results give expected depth $\Theta(\log(H/n))$~\citep{devroye1986note}. In the worst case, every split is adjacent to an endpoint, reducing the longest remaining segment by one per level. Hence, the maximum depth is $H-n=\Theta(H-n)$. Vanilla \trl\ stops only at one-step segments; setting $n=1$ gives expected depth $\Theta(\log H)$ and maximum depth $H-1=\Theta(H)$.
\end{proof}

The midpoint split halves the longest remaining segment at each level. Recursive uniform splitting has the structure of a random binary search tree, giving logarithmic average depth, whereas repeatedly splitting next to an endpoint gives linear worst-case depth. Accordingly, \texttt{TD-n} has $\Theta(H/n)$ sequential propagation levels, \texttt{DCRL-n} has $\Theta(\log(H/n))$, and \texttt{TRL-n} lies between these cases depending on its realized splits. The transitive branch of vanilla \trl\ has the same random-split structure but stops only at one-step leaves.

\paragraph{Error transmission.}
We characterize how each backup transmits the error of its children.

\begin{assumption}
\label{ass:kappa_main}
For each method $\bullet$, there exist $\epsilon\ge0$ and
$\zeta_\bullet\ge1$ such that, for every non-grounded $p$ and backup realization $\omega$,
\begin{equation*}
  |e_\bullet(p)|
  \le
  \epsilon+\zeta_\bullet
  \sum_{q\in B_\bullet(p;\omega)}|e_\bullet(q)|,
  \qquad
  e_\bullet(p)=0
  \quad\text{for }p\in\mathcal{P}_\circ.
\end{equation*}
\end{assumption}

We use $\zeta=1$ for nonselective symmetric backups and $\zeta>1$ to model amplification from optimistic selection. In \texttt{TRL-n}, an upper expectile ($\kappa>0.5$) can preferentially weight overestimated transitive targets; we model this effect with $\zeta_{\texttt{TRL-n}}>1$. At $\kappa=0.5$, the regression is symmetric, and we set $\zeta_{\texttt{TRL-n}}=1$. Vanilla \trl\ uses the same transitive backup but has no additional $n$-step propagation objective and therefore only one-step grounding.
\texttt{DCRL-n} instead uses the deterministic midpoint $k=\lfloor(i+j)/2\rfloor$ and the factorized distance target
$\hat d(s_i,s_k)+\hat d(s_k,s_j)$ under a symmetric loss. Because there is no candidate set to select from, we model this nonselective backup with $\zeta_{\texttt{DCRL-n}}=1$. This objective cannot itself discover shortcuts, which is why \texttt{DCRL-n} pairs it with a separate multistep propagation objective.
\texttt{TD-n}'s expectile ($\kappa>0.5$) is applied across the sample distribution and therefore acts as an implicit maximum over the $n$-step continuations available from $(s_i,g)$ in the data, in the manner of
IQL~\citep{kostrikov2022offline}. Where several continuations of differing value are supported, this selection can amplify child error, which we model with $\zeta_{\texttt{TD-n}}>1$; when the continuation is effectively unique, we use $\zeta_{\texttt{TD-n}}=1$.

\paragraph{Upper bounds.}
First, we discuss upper error bounds for \texttt{TD-n}, \texttt{TRL-n}, and \texttt{DCRL-n}.
Applying \cref{ass:kappa_main} repeatedly, from a sample to the grounded leaves on which it ultimately rests, bounds that sample's error. Doing so for each method gives the following.

At $\zeta=1$, \texttt{TD-n} forms a chain of $d_{\texttt{TD-n}}=\lceil H/n\rceil-1$ backups, giving $|e|\le\epsilon d_{\texttt{TD-n}}=O(\epsilon H/n)$. \texttt{DCRL-n} forms a binary tree of depth $d=\lceil\log_2(H/n)\rceil$. Unrolling \cref{ass:kappa_main} over this tree gives $|e|\le\epsilon(2^d-1)=O(\epsilon H/n)$. Thus, both methods admit the same asymptotic upper bound, despite having different sequential depths.
For \texttt{TRL-n} with \(\kappa>0.5\), consider a realized decomposition of depth
\(d_\omega=\mathrm{D}_{\texttt{TRL-n}}(p;\omega)\). Let \(E_r\) bound the error at depth \(r\), with \(E_0=0\). Its selection gain \(\zeta_{\texttt{TRL-n}}>1\) gives
\[
E_r\le 2\zeta_{\texttt{TRL-n}}E_{r-1}+\epsilon,
\]
and therefore
\[
E_{d_\omega}
\le
\epsilon\frac{(2\zeta_{\texttt{TRL-n}})^{d_\omega}-1}
{2\zeta_{\texttt{TRL-n}}-1}
=
O\!\left(\epsilon(2\zeta_{\texttt{TRL-n}})^{d_\omega}\right).
\]
Thus, logarithmic-depth realizations yield a polynomial error bound in \(H/n\), whereas highly unbalanced realizations can produce a much looser bound. Selection amplification therefore compounds with decomposition depth.

The upper bound cannot separate \texttt{TD-n} from \method, since a horizon-$H$ value ultimately depends on the same $H/n$ grounded segments. However, their dependency structures have different sequential depths.

\paragraph{Propagation depth.}
The static upper bounds above do not account for the order in which grounded information propagates through learned targets. Consider an idealized local backup process in which information traverses at most one dependency edge per update stage. By definition, information from a grounded leaf requires at least $\mathrm{D}_B(p)$ sequential stages to reach a sample $p$. Therefore,
\begin{align*}
  \texttt{TD-n}:&\quad \Theta(H/n)\ \text{stages},\\
  \method:&\quad \Theta(\log(H/n))\ \text{stages}.
\end{align*}
Thus, \method\ has an asymptotically shorter propagation path from grounded segments to long-horizon values. Bootstrap depth alone does not determine estimation error; however, a shorter path requires fewer sequential stages for grounded information to reach long-horizon values. We empirically test whether this structural advantage corresponds to lower finite-budget error in Appendix~\ref{app:error_accumulation_empirical}.

\paragraph{Two mechanisms.}
The analysis highlights two complementary mechanisms:
\begin{enumerate}[left=0pt]
  \item \textbf{Fixed-split factorization.} The exact factorization in \cref{eq:pessimistic_value_backup} avoids the selection bias introduced by upper-expectile transitive backups such as \trl\ and \texttt{TRL-n} {\small($\kappa>0.5$)}.
  \item \textbf{Recursive scheduling.} Balanced recursion reduces the dependency depth from linear to logarithmic, while the leaf-first schedule trains child values before using them in parent updates.
\end{enumerate}
The first mechanism avoids maximization-induced selection error, whereas the second shortens and explicitly orders the propagation path from grounded to long-horizon values.

\subsubsection{Empirical Validation}
\label{app:error_accumulation_empirical}

In this section, we empirically test how bootstrap depth and recursive scheduling affect finite-budget error accumulation. We examine whether (i) shallower dependency structures correspond to slower long-range error growth and (ii) the depth advantage is realized when children are trained before their parents.

\paragraph{Didactic environment.}
Measuring error accumulation in benchmarks (\eg, OGBench) is confounded because the
optimal value $d^\ast(s,g)$ is unavailable in closed form and any proxy
introduces a discretization error that itself grows with the horizon.
We therefore use the \texttt{combination-lock} environment of
\citet{park2025horizon}, introduced for exactly this purpose.
It consists of $H$ states in a line and two discrete actions; each state has one
\emph{answer} action, fixed by a random seed, that advances one step, while the
other returns the agent to state $0$.
States are encoded as $\lceil\log_2 H\rceil$-dimensional binary vectors under a
random permutation of indices, so the ordering cannot be read off the
representation.
Following \citet{park2025horizon}, we use the undiscounted step-cost setting ($\gamma=1$) and learn distances directly rather than through $\log_\gamma Q$, which is defined only for $\gamma\in(0,1)$. Across environments of increasing horizon, distances and errors are therefore not capped by the discounted value range. Under this convention, $Q^\ast(s,g)=-d^\ast(s,g)$ and $\hat Q(s,g)=-\hat d(s,g)$, giving $|\hat Q-Q^\ast|=|\hat d-d^\ast|$. Thus, Q-error and distance error have identical magnitudes and units.

We adapt the task to the goal-conditioned setting required by \method\ and \trl\ because their backups factor through an intermediate state and therefore require
values for state--goal \emph{pairs} rather than a single fixed goal.
Along the forward path, the optimal distance is available exactly,
$d^\ast(s,g) = g-s$, giving ground truth at every horizon.

All variants learn $\hat d(s,g)$ by MSE regression in distance space and use the same network ($[512,512,512]$ MLP with LayerNorm and GELU), optimizer (Adam, $3\!\times\!10^{-4}$, batch $512$, target-update rate $0.005$), gradient budget ($5$M steps), and uniform sampling of $(s,g)$. Let $h:=g-s$ and set $n=64$.

We compare five variants. \texttt{TD-n}, both \texttt{TRL-n} variants, and \texttt{DCRL-n} use $n=64$. Each \texttt{TRL-n} variant retains \trl's transitive objective and adds the same $n$-step propagation objective used by \texttt{DCRL-n}; vanilla \trl\ uses only the transitive objective. \texttt{TRL-n} {\small($\kappa=0.7$)} uses an upper-expectile transitive loss, whereas \texttt{TRL-n} {\small($\kappa=0.5$)} uses symmetric transitive regression. Here, $\kappa$ applies only to the transitive objective; the added $n$-step propagation objective is unchanged.

\paragraph{Measuring error accumulation.}
Let $e(h) := \mathbb{E}_{g-s=h}\,\lvert \hat d(s,g) - d^\ast(s,g)\rvert$ be the mean absolute distance error at horizon $h$. As $H$ increases, the number of state--goal pairs grows as $H^2$, making the value function increasingly difficult to fit. We capture the resulting total long-range error directly, defining \emph{error accumulation} as the mean error on pairs spanning at least half of the environment horizon:
\begin{equation}
  \mathcal{A}(H) \;:=\; \mathbb{E}_{(s,g):g-s \ge H/2}\, \left[ \lvert \hat d(s,g) - d^\ast(s,g)\rvert \right].
  \label{eq:accumulation}
\end{equation}
This metric reports long-range error in the original distance units without normalization by short-range fitting error. All variants use the same architecture and optimization budget, while the four $n$-step variants also receive identical exact supervision for $h\le n$. Differences among these four therefore isolate their backup mechanisms and schedules; vanilla \trl\ additionally measures the effect of omitting $n$-step propagation.

\paragraph{Separating propagation, selection, and recursion.}
\cref{fig:accumulation} reports $\mathcal{A}(H)$ over seven horizons spanning 256--2,048. At long horizons, error follows
\trl > \texttt{TRL-n}\ {\small($\kappa=0.7$)}
> \texttt{TRL-n}\ {\small($\kappa=0.5$)}
$\approx$ \texttt{TD-n} > \texttt{DCRL-n}.
Adding $n$-step propagation substantially reduces \trl's error. With this propagation fixed, lowering the transitive expectile from $\kappa=0.7$ to $0.5$ further reduces error, showing that optimistic subgoal selection amplifies estimation noise. \texttt{TRL-n} {\small($\kappa=0.5$)} performs similarly to \texttt{TD-n}, whereas \texttt{DCRL-n} remains relatively flat and achieves the lowest error. Since \texttt{TRL-n} {\small($\kappa=0.5$)} and \texttt{DCRL-n} both use symmetric regression and the same $n$-step propagation, the remaining gap reflects \method's fixed-midpoint recursion and child-to-parent scheduling.

\paragraph{Child-to-parent scheduling realizes the depth advantage.}
\cref{fig:schedule_ablation} isolates the second mechanism by holding the backup
operator fixed and varying only the schedule.
The benefit of scheduling grows with the horizon. Up to $H\approx1024$, the two variants perform comparably (\method\ {\small \texttt{(w/o sched.)}}: $\mathcal{A}=1.61$ versus \method: $1.78$ at $H=1024$), suggesting that both can adequately learn relatively shallow dependencies. Beyond this point, \method\ {\small \texttt{(w/o sched.)}} degrades sharply---reaching $3.72$ at $H=1536$ and $9.14$ at $H=2048$, nearly matching \texttt{TD-n} ($9.10$)---while \method\ remains flat. This result suggests that unordered sampling becomes increasingly prone to bootstrapping from unconverged child estimates as the horizon grows. The ablation therefore supports the practical importance of child-to-parent scheduling for long-horizon value learning.

\begin{figure}[t]
  \centering
  \begin{minipage}{0.48\textwidth}
    \centering
    \includegraphics[width=\linewidth]{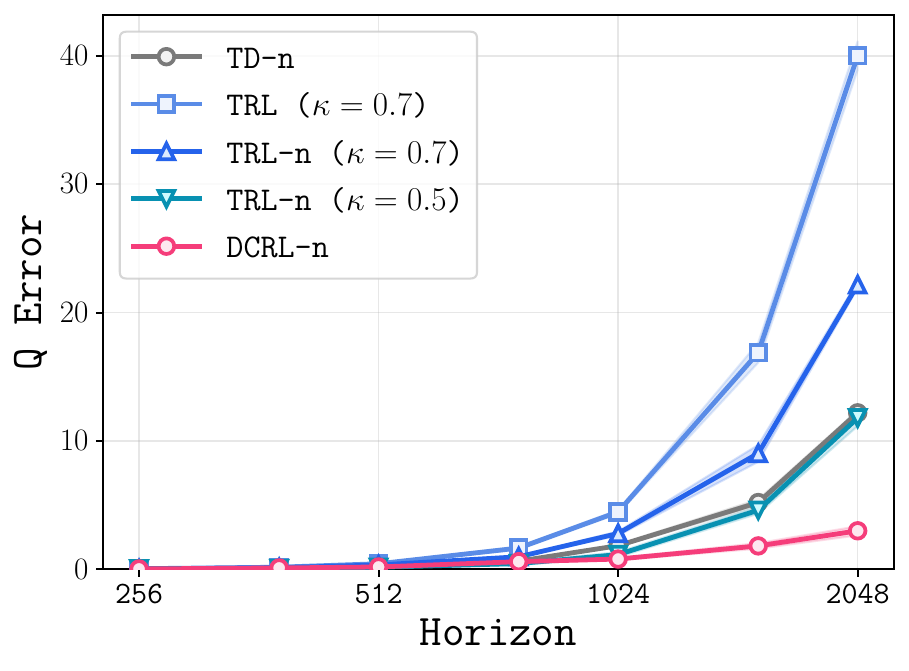}
    \caption{On \texttt{combination-lock}, accumulated error $\mathcal{A}(H)$ from \cref{eq:accumulation} ($3$ seeds; shading: $\pm1$ s.d.). Adding $n$-step propagation reduces \trl's error, while symmetric transitive regression ($\kappa=0.5$) further lowers \texttt{TRL-n}'s error to the \texttt{TD-n} level. \texttt{DCRL-n} remains lowest and grows slowest, consistent with its logarithmic bootstrap depth and child-to-parent scheduling.}
    \label{fig:accumulation}
  \end{minipage}\hfill
  \begin{minipage}{0.48\textwidth}
    \centering
    \includegraphics[width=\linewidth]{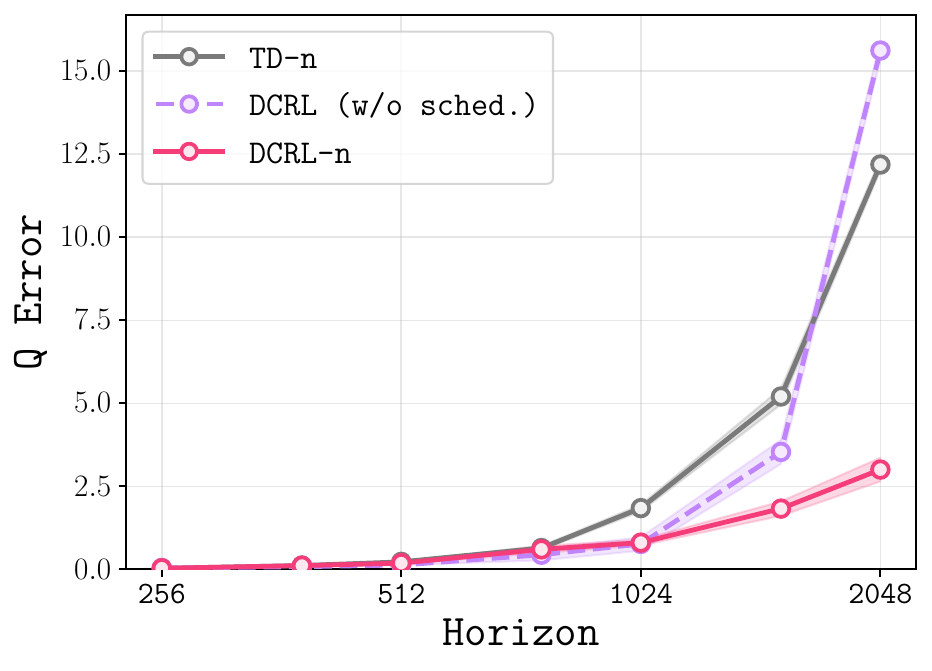}
    \caption{On \texttt{combination-lock}, we isolate child-to-parent scheduling while holding the factorized backup and $n$-step propagation fixed ($3$ seeds; shading: $\pm1$ s.d.). I.i.d.\ sampling performs comparably for $H\lesssim1024$, but its error rises sharply at longer horizons and approaches \texttt{TD-n}'s. \texttt{DCRL-n} remains flat, showing that ordering becomes important as dependency chains deepen.}
    \label{fig:schedule_ablation}
  \end{minipage}
\end{figure}

\paragraph{Scope.}
This experiment isolates errors induced by backup mechanisms and scheduling and differs from our full implementation in three respects. First, values are regressed directly in distance space with squared loss, whereas the deployed agents parameterize $Q\in(0,1)$ and use multiplicative factorization under binary cross-entropy loss. Distance regression leaves error unbounded and therefore observable. Second, the didactic variants optimize only the backup objectives: \trl\ uses its transitive objective, \texttt{TRL-n} adds $n$-step propagation, and \texttt{DCRL-n} combines recursive factorization with the same propagation objective. They omit policy learning and implementation details of the full agents. Third, action and intermediate-state selection behave differently in \texttt{combination-lock}. Over actions, the correct successor and reset state differ in value by $\Theta(s)$, far exceeding approximation noise; hence, \texttt{TD-n}'s expectile has no ambiguous choice and $\zeta_{\texttt{TD-n}}=1$. The experiment does not measure action-selection amplification under the multimodal benchmark data. Over intermediate states, every split is exactly tied, $d^\ast(s,k)+d^\ast(k,g)=g-s$, so upper-expectile transitive regression favors estimation noise. Accordingly, the gap between \trl\ and \texttt{TRL-n} {\small($\kappa=0.7$)} reflects the addition of $n$-step propagation, while the gap between the two \texttt{TRL-n} variants isolates transitive selection amplification. The gap between \texttt{TRL-n} {\small($\kappa=0.5$)} and \texttt{DCRL-n} reflects fixed-midpoint recursion and child-to-parent scheduling. These conclusions concern backup mechanisms and scheduling rather than end-to-end performance, which \cref{sec:experiments} evaluates separately.

\begin{figure}[ht]
    \centering
    \includegraphics[width=\linewidth]{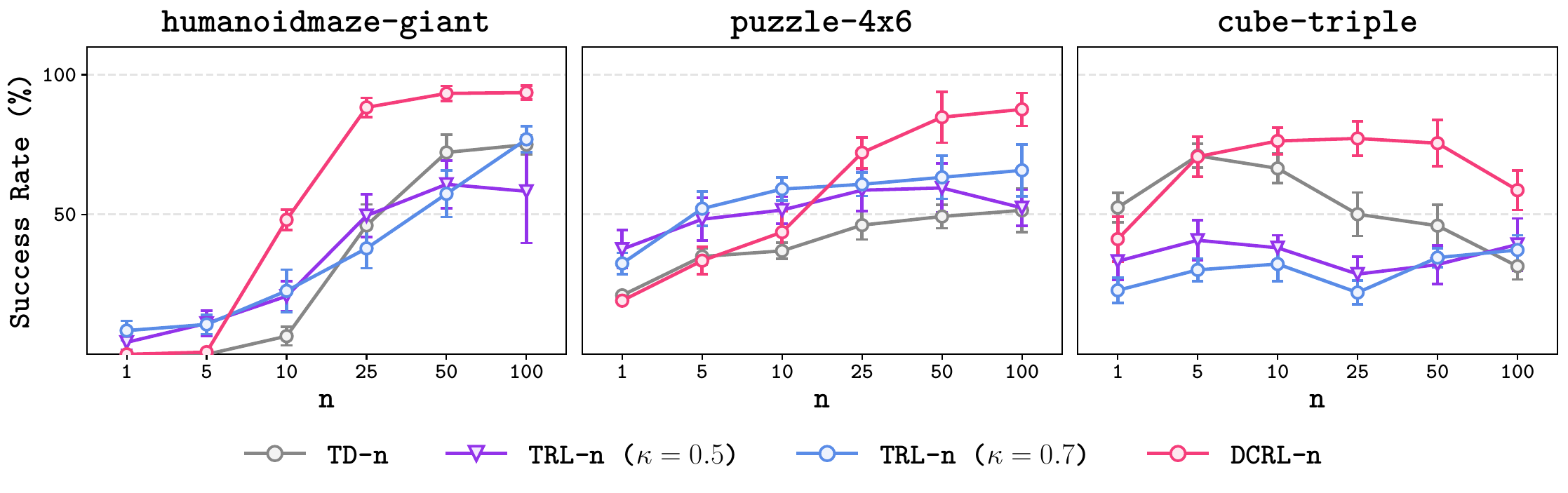}
    \caption{Exhaustive comparison across propagation horizons $n$ on three long-horizon environments. We report success rate across
    $n\in\{1,5,10,25,50,100\}$. \texttt{DCRL-n} performs best for
    $n\ge25$ on all tasks, while the two \texttt{TRL-n} variants perform similarly overall.}
    \label{fig:nstep_comparison_exhaustive}
\end{figure}

\paragraph{End-to-end propagation-horizon comparison.}
\Cref{fig:nstep_comparison_exhaustive} complements the controlled error
analysis with benchmark success rates across propagation horizons.
\texttt{DCRL-n} performs best for $n\ge25$ on all three tasks.
The two \texttt{TRL-n} variants perform similarly overall, with neither consistently outperforming the other. Thus, although symmetric
transitive regression reduces error in the controlled environment, this
benefit does not consistently translate into higher benchmark performance.

\section{Full Algorithm}
\label{app:algorithm}

\begin{algorithm}[ht]
    \caption{Divide-and-Conquer Reinforcement Learning (\method)}
    \label{alg:dcrl}
    \begin{algorithmic}[1]
        \STATE \textbf{Input:} Dataset $\mathcal{D}$, number of slots $S$, propagation horizon $n$
        \STATE Initialize critic $Q(s,a,g)$, slots $\{\mathcal{T}_\ell \gets \texttt{SampleSlot}(\mathcal{D})\}_{\ell=1}^{S}$
        \WHILE{not converged}
        \STATE \textcolor{highlight}{$\triangleright$ Divide-and-Conquer}
        \STATE Sample $\ell \sim \mathrm{Unif}(\{1,\ldots,S\})$
        \STATE Sample $\tau_{\text{DC}}=(s_i, a_i, s_k, a_k, s_j) \gets \mathcal{T}_\ell.\texttt{pop}()$
        \IF{$\mathcal{T}_\ell = \emptyset$}
        \STATE $\mathcal{T}_\ell \gets \texttt{SampleSlot}(\mathcal{D})$
        \ENDIF
        \STATE \textcolor{highlight}{$\triangleright$ Propagation}
        \STATE Sample $\tau_{\text{prop}}=(s_i, a_i, s_{i+n}, a_{i+n}, g) \sim \mathcal{D}$
        \vspace{1mm}
        \STATE Train $Q$ by minimizing $\mathcal{L}^{\texttt{dc}}(Q) + \mathcal{L}^{\texttt{prop}}(Q)$ \hfill (Eqs.~\ref{eq:dc_loss},\,\ref{eq:prop_loss})
        \vspace{1mm}
        \ENDWHILE
    \end{algorithmic}
\end{algorithm}

\section{Baselines}
\label{app:baseline}

Offline GCRL methods generally comprise two components: goal-conditioned value estimation and policy extraction. We describe the main value-learning mechanism of each baseline below and note its policy-extraction mechanism when it is essential to the method.
Our implementations of \texttt{TD}, \texttt{TD-n}, \texttt{TDP}, \texttt{COE}, and \texttt{TRL} follow \citet{park2026transitive}.

\paragraph{Temporal-difference methods.}
\texttt{IQL}~\citep{kostrikov2022offline} fits a goal-conditioned state-value function by asymmetric expectile regression over dataset actions and trains its critic with Bellman backups. The expectile approximates action maximization without evaluating out-of-distribution actions; at $\kappa=0.5$, the objective reduces to behavioral SARSA value estimation.
\texttt{SAC+BC}~\citep{haarnoja2018soft,park2025horizon} extracts a stochastic policy by maximizing the learned critic through reparameterized gradients while adding an entropy or behavior-cloning term that constrains the policy toward dataset actions.
\td\ is a one-step, IQL-style goal-conditioned TD baseline, whereas \texttt{TD-n} replaces the one-step bootstrap with an $n$-step return and clips $n$ when the sampled goal occurs sooner~\citep{sutton2018reinforcement}.
\method\ uses the same $n$ as its propagation horizon on each task; we report these values in Appendix~\ref{app:experimental_setup}.

\paragraph{Monte Carlo methods.}
\texttt{CRL}~\citep{eysenbach2022contrastive} learns a goal-conditioned critic by contrastively distinguishing future states from independently sampled states and derives the policy from this learned reachability signal.
\texttt{MC} learns a behavioral goal-conditioned value by directly regressing $Q(s_i,a_i,s_j)$ onto the empirical trajectory return $\gamma^{j-i}$, avoiding bootstrapping~\citep{tian2020model,shah2021rapid}. This estimator is simple and stable but recovers behavior rather than optimal values under suboptimal data and can have high variance under stochastic dynamics.

\paragraph{Hierarchical methods.}
\texttt{POR}~\citep{xu2022policy} is a policy-guided imitation method that uses a learned policy to select useful intermediate targets from the offline data and imitates the resulting improved decisions.
\texttt{HIQL}~\citep{park2023hiql} treats latent future states as high-level actions: a high-level policy proposes a waypoint and a low-level goal-conditioned policy reaches it, with both policies extracted by advantage-weighted regression.
\texttt{SHARSA}~\citep{park2025horizon} reduces the effective decision horizon by training a high-level subgoal policy together with a low-level policy that executes each temporally extended subgoal.

\paragraph{Triangle-inequality methods.}
\texttt{QRL}~\citep{wang2023optimal} parameterizes the goal-reaching value as a quasimetric, architecturally imposing non-negativity, asymmetry, and the triangle inequality.
\texttt{TMD}~\citep{myers2025offline} learns an asymmetric temporal-distance representation from offline data while enforcing quasimetric structure, then uses this distance to extract a goal-conditioned policy.
\texttt{TDP}~\citep{kaelbling1993learning,jurgenson2020sub} applies a Floyd--Warshall-style backup that chooses the best intermediate subgoal; in continuous domains, we approximate the hard maximum over the state space using subgoal candidates sampled from the dataset.
\texttt{COE}~\citep{pikekos2023efficient} enforces the compositional optimality equation with a separate generator that proposes the maximizing intermediate state; for offline training, the generator is regularized toward in-distribution states.
\texttt{TRL}~\citep{park2026transitive} propagates reachability by composing two shorter goal-conditioned values through an intermediate state using a transitive backup, yielding an implicit divide-and-conquer structure.

\paragraph{Policy-only methods.}
We additionally include behavior cloning (\texttt{BC}), hierarchical behavior cloning (\texttt{HBC})~\citep{gupta2019relay}, and flow-matching behavior cloning (\texttt{FBC})~\citep{chi2023diffusion}. \texttt{BC} directly imitates dataset actions, \texttt{HBC} separately imitates high-level waypoints and low-level actions, and \texttt{FBC} uses an expressive flow-matching policy to model potentially multimodal behavior distributions.

\section{Alternative Divide-and-Conquer Strategies}
\label{app:dnc_strategies}
Here we describe alternative divide-and-conquer strategies that we compare with \method\ in \cref{sec:a3}.

\paragraph{Curriculum learning.}
The simplest divide-and-conquer approach learns values on short trajectories first, gradually increasing their length via a manual curriculum. For length-$T$ trajectories with curriculum step $\Delta$, the $i$-th stage ($i=1, \dots, T/\Delta$) learns values for sub-trajectories of length $(i-1)\cdot\Delta+1$ to $i\cdot\Delta$, sampled uniformly from full trajectories.

\paragraph{Distance re-weighting.}
\citet{park2026transitive} implement divide-and-conquer implicitly by downweighting distant pairs in the loss. Each sample is weighted by $w(s_i,s_j) := \big(1 + \log_\gamma Q(s_i,a_i,s_j)\big)^{-\lambda}$, where $\log_\gamma Q$ is the learned distance. The weight decays with distance, so short pairs dominate; $\lambda$ controls the decay rate, and $\lambda = 0$ recovers uniform weighting.

\paragraph{Random subgoal sampling.}
While \method\ splits each problem at the trajectory midpoint, an alternative samples a random in-trajectory subgoal. We implement \method\ {\small (w/ random split)}, which is identical to \cref{alg:dcrl} but uses a random split $k \sim \text{Unif}(i+1, j-1)$ instead of the midpoint during recursive tree construction (\cref{alg:sampleslot}).

\paragraph{Max-based backup.}
\method\ backs up each node through the trajectory midpoint with an exact factorization, $Q(s_i,a_i,s_j) \gets Q(s_i,a_i,s_k)\,Q(s_k,a_k,s_j)$ with $k=\lfloor(i+j)/2\rfloor$. A natural alternative replaces this fixed split with a \emph{learned} one: at each node we sample $M=10$ in-trajectory subgoal candidates $m \sim \text{Unif}(i+1, j-1)$ and both split and back up at the candidate maximizing the factorized value under the current critic,
\begin{equation}
m^\star = \argmax_{m}\; Q(s_i,a_i,s_m)\,Q(s_m,a_m,s_j), \qquad Q(s_i,a_i,s_j) \;\gets\; Q(s_i,a_i,s_{m^\star})\,Q(s_{m^\star},a_{m^\star},s_j).
\end{equation}
When a segment has fewer than $M$ interior states, we simply enumerate all of them. We implement \method\ {\small (w/ max-based backup)}, identical to \cref{alg:dcrl} except that the midpoint split $k$ in the recursive tree construction (\cref{alg:sampleslot}) is replaced by $m^\star$. Because the split now depends on $Q$ and determines the node's children, the tree is constructed top-down using the current critic, and the schedule is rebuilt every $5000$ steps so the splits track the critic as it improves; the recursion, the base-case grounding $Q(s_i,a_i,s_{i+1})\!\gets\!\gamma$, and the value-propagation loss are all unchanged from \method. For the demonstrated route, the factorized target $\gamma^{m-i}\gamma^{j-m}=\gamma^{j-i}$ is independent of $m$. Maximizing learned values may nevertheless favor either shortcuts introduced by propagation or estimation error. This ablation tests whether adaptive split selection improves upon the fixed midpoint.

\paragraph{Without recursive scheduling.}
To test whether the depth-stratified, child-to-parent schedule is necessary, we ablate the recursion entirely while retaining the factorization backup. Instead of recursively decomposing each trajectory into a tree (\cref{alg:sampleslot}), we sample pairs $(s_i, s_j)$ i.i.d.\ per batch with $j - i \sim \text{Geom}(1-\gamma)$, split once at the midpoint $k=\lfloor(i+j)/2\rfloor$, and apply the same backup $Q(s_i,a_i,s_j) \gets Q(s_i,a_i,s_k)\,Q(s_k,a_k,s_j)$. 

Without recursion, the sub-problems $Q(s_i,a_i,s_k)$ and $Q(s_k,a_k,s_j)$ are no longer guaranteed to have been trained before their parent. Moreover, i.i.d.\ sampling does not explicitly provide the one-step leaves that ground each recursive tree. Without reliable base-case supervision, the factorization target becomes self-referential and the critic collapses ($\bar{Q}\!\to\!0$). We therefore add an explicit one-step loss on adjacent transitions,
\begin{equation}
\mathcal{L}_1(Q) = \mathbb{E}_{(s,a,s')\sim\mathcal{D}} \big[\, \mathcal{L}_\text{BCE}\big(Q(s,a,s'),\, \gamma\big) \,\big],
\end{equation}
which restores the exact $\gamma^1$ base case that the recursion supplies implicitly. We implement \method\ {\small (w/o recursion)} with this grounding term and the unchanged value-propagation loss. This construction isolates the absence of recursive scheduling while retaining the same factorization and propagation objectives with equivalent one-step grounding.

\paragraph{Without propagation.}
To test whether divide-and-conquer alone can recover optimal values, we remove $\mathcal{L}^{\texttt{prop}}$ and increase the divide-and-conquer expectile from $\kappa_{\texttt{dc}}=0.5$ to $0.7$. The upper expectile favors higher-valued routes across samples, tasking the recursive objective with learning optimal values rather than behavior values. All other components remain unchanged.

\paragraph{Reverse scheduling.}
We implement \method\ {\small (w/ reverse order)} by retaining the same midpoint tree, factorized backup, one-step groundings, and propagation loss as \method, but reversing each slot's traversal order so that parents are emitted before their children. This ablation isolates traversal direction while preserving the recursive tree structure.

\section{Experimental Setup}
\label{app:experimental_setup}

\paragraph{Policy Extraction.}
After training the action-value function $Q(s,a,g)$, we extract a goal-conditioned policy $\pi(s,g): \mathcal{S} \times \mathcal{S} \to \mathcal{A}$ that takes the highest-$Q$ action. Policy extraction can be performed either during or after value training.
We describe policy extraction procedures in Appendix~\ref{sec:policy_extraction}. 

\paragraph{Evaluation.}
To evaluate offline goal-conditioned RL algorithms, we use a suite of environments from OGBench~\citep{park2025ogbench} and CALVIN~\citep{mees2022calvin}.

For OGBench, each environment is evaluated on five test-time tasks corresponding to distinct start--goal pairs, with 15 rollouts per state-based task and 50 rollouts per pixel-based task.

For CALVIN, we use the in-domain Task D$\rightarrow$Task D setup and the dataset provided by \citet{shi2022skill}. We evaluate the target sequence of four consecutive subtasks---opening the drawer, turning on the lightbulb, moving the slider to the left, and turning on the LED---using 50 rollouts for the state-based variant and 100 for the pixel-based variant.

\paragraph{Pixel-based CALVIN.}
The dataset of \citet{shi2022skill} provides proprioceptive and scene states but no images, so we re-render it: every transition is replayed in the CALVIN simulator and an image is captured by the static third-person camera, yielding $128 \times 128 \times 3$ RGB observations. Rendering is done once offline and cached, so all methods consume identical pixel trajectories. Each observation comprises the image and a $15$-dimensional proprioceptive vector and is normalized using dataset statistics; the goal is the rendered image of the target scene, in which all four subtasks are complete. The evaluation protocol is otherwise unchanged from the state-based variant: we use a single fixed far goal, and a rollout succeeds at level $k$ if at least $k$ of the four subtasks are completed in the prescribed order.

\paragraph{Visual encoder.}
For pixel-based CALVIN, we use a smaller version of the Impala CNN~\citep{espeholt2018impala}, matching the encoder used for the pixel-based OGBench environments, so that architecture is not a confounding factor across benchmarks. Observations and goals are encoded by separate encoder instances, both of which are trained end-to-end with the rest of the network at the same learning rate. We apply random-crop augmentation with probability $0.5$. The padding is increased from $3$ to $6$ pixels: OGBench's pixel observations are $64 \times 64$, whereas ours are $128 \times 128$, so a $3$-pixel pad would shift the image by half the relative amount and cover a correspondingly smaller fraction of the receptive field. Scaling the padding with the resolution keeps the effective augmentation strength, in units of the CNN's coverage, comparable to the OGBench setting.

\paragraph{Long-horizon OGBench tasks.}
For large-scale, long-horizon tasks (\cref{tab:large_scale_performance}), we report the mean performance over four random seeds, averaged across checkpoints at 800k, 900k, and 1M training steps.
For \method's value propagation, we use $n=100$ on \texttt{humanoidmaze-giant-1B} and \texttt{puzzle-\{4x5,4x6\}-1B}, and $n=25$ on \texttt{cube-quadruple-100M} and \texttt{cube-octuple-1B}.
Cube tasks decompose into single-block pick-and-place subtasks, so a smaller $n$ suffices even when the full task is long.
Full long-horizon results are reported in \cref{tab:large_scale_performance_appendix}.

\begin{table}[h]
\centering
\caption{Long-horizon environment specifications.} \label{tab:env_specs}
\resizebox{\textwidth}{!}{
    \begin{tabular}{lcccc} \toprule \texttt{Environment} & \texttt{State Dim.} ($\mathcal{S}$) & \texttt{Goal Domain} ($\mathcal{G}$) & \texttt{Action Dim.} ($\mathcal{A}$) & \texttt{Episode length} \\
    \midrule
    \texttt{humanoidmaze-giant} & \texttt{69} & $\mathbb{R}^2$   & \texttt{21}  & \texttt{4000}  \\
    \texttt{puzzle-4x5}         & \texttt{99}  & $\{0,1\}^{20}$  & \texttt{5}   & \texttt{1000}  \\
    \texttt{puzzle-4x6}         & \texttt{115}  & $\{0,1\}^{24}$  & \texttt{5}  & \texttt{1000}  \\
    \texttt{cube-quadruple}     & \texttt{55}  & $\mathbb{R}^{12}$  & \texttt{5}   & \texttt{1000}  \\
    \texttt{cube-octuple}       & \texttt{91}  & $\mathbb{R}^{24}$  & \texttt{5}   & \texttt{1500}  \\
    \bottomrule
    \end{tabular}
}
\end{table}

\paragraph{Standard OGBench tasks.}
For state-based tasks (\cref{fig:ogbench_overall} (\textit{left})), we report mean performance over four random seeds, averaged across the 800k, 900k, and 1M checkpoints. 
We use $n=25$ for \method's value propagation on all standard OGBench tasks, both state-based and pixel-based.
The state-based environments include:
\begin{itemize}[leftmargin=*]
    \item \texttt{antmaze-large}, \texttt{antsoccer-arena}, \texttt{pointmaze-large}, \texttt{scene}
    \item \texttt{cube-\{single,double\}}
    \item \texttt{humanoidmaze-\{medium,large\}}
    \item \texttt{puzzle-\{3x3,4x4\}}
\end{itemize}

For pixel-based tasks (\cref{fig:ogbench_overall} (\textit{right})), we report mean performance over four random seeds, averaged across the 300k, 400k, and 500k checkpoints. The pixel-based environments include:
\begin{itemize}[leftmargin=*]
    \item \texttt{visual-scene}
    \item \texttt{visual-antmaze-\{medium,large,giant\}}
    \item \texttt{visual-cube-\{single,double,triple\}}
    \item \texttt{visual-puzzle-\{3x3,4x4,4x5\}}
\end{itemize}
Full state-based results and pixel-based results are reported in \cref{tab:standard_performance} and \cref{tab:visual_performance}.
\paragraph{CALVIN benchmark.}
In \cref{fig:calvin_results}, we report mean performance over eight random seeds at the 500k checkpoint for state-based CALVIN, and over four random seeds at the 200k checkpoint for pixel-based CALVIN. On state-based CALVIN, we train \texttt{HIQL} without a latent representation, so its high-level policy predicts subgoals directly in the $21$-dimensional observation space and is therefore comparable to the other value-learning methods, none of which learn a separate goal representation. On pixel-based CALVIN, this choice is not viable: a representation-free high-level policy would have to regress the full goal image under an isotropic Gaussian, so the BC term dominates the advantage weighting and the high-level objective degenerates into pixel reconstruction. We therefore retain \texttt{HIQL}'s latent subgoal representation in the pixel setting, with the high-level policy predicting $\phi(g)$ in the space produced by the goal encoder. To keep the comparison as close as possible to the state-based setting, $\phi$ is trained solely by the value objective rather than by an auxiliary representation loss, so \texttt{HIQL} learns no goal representation beyond what its value function induces. We use $n=25$ for \method's value propagation. Full CALVIN results are reported in \cref{tab:calvin_results_appendix} and \cref{tab:calvin_pixel_results_appendix}.

\paragraph{$n$-step comparison.}
For the $n$-step comparison (\cref{fig:result_prop_h}), we use 10 random seeds and report mean performance at the 1M checkpoint for each $n \in \{1, 5, 10, 25, 50, 100\}$.
We use \texttt{humanoidmaze-giant-1B}, \texttt{puzzle-4x6-1B}, and \texttt{cube-triple-100M}.
Full task-level results for \texttt{TD-n}, \texttt{TRL-n} {\small($\kappa=0.7$)}, and \texttt{DCRL-n} are reported in Appendix~\ref{app:propagation_horizon_comparison}.

\section{Policy Extraction}
\label{sec:policy_extraction}

We consider the following two standard policy-extraction techniques.

\textbf{Reparameterized gradients}~\citep{fujimoto2021minimalist} extract a Gaussian policy by maximizing the following combined DDPG~\citep{lillicrap2016continuous} and behavioral cloning (BC) objective:
\begin{equation}
    J_{\text{DDPG+BC}}(\pi) = \mathbb{E}_{(s, a, g) \sim \mathcal{D}, \, a^\pi \sim \pi(\cdot | s, g)} \left[ Q(s, a^\pi, g) + \alpha \log \pi(a | s, g) \right],
\end{equation}
where $a^\pi$ is a reparameterized action from the policy, $a$ is a dataset action, and $\alpha$ determines the strength of the BC regularization. This objective maximizes the predicted value while penalizing deviation from the dataset distribution to prevent out-of-distribution exploitation.

\textbf{Rejection sampling}~\citep{chen2023offline, hansen2023idql} defines a nonparametric policy by drawing $N$ action candidates and selecting the candidate with the highest value:
\begin{equation}
    \pi(s, g) \stackrel{d}{=} \arg\max_{
\{a_i\}_{i=1}^N \;\text{s.t.}\; a_i \sim \pi_\beta(\cdot|s,g)} Q(s, a_i, g),
\end{equation}
where $N$ denotes the number of candidates, $\stackrel{d}{=}$ denotes equality in distribution, and $\pi_\beta$ is a goal-conditioned BC policy. In practice, the BC policy is parameterized as an expressive diffusion or flow-matching model to capture multimodal behavior distributions effectively.

In our experiments, we use both reparameterized gradients and rejection sampling, depending on the environment. We report the exact extraction method and hyperparameters (\eg, $\alpha$ and $N$) in Appendix~\ref{sec:hyperparameters}.

\section{Additional Related Work}
\label{app:add_related_work}
\paragraph{Curriculum Learning.} 
Our recursive schedule, which learns shorter samples before longer ones, resembles a curriculum over horizon length. Curriculum learning orders tasks by difficulty~\citep{bengio2009curriculum,graves2017automated,florensa2018automatic}, requiring a task space and difficulty estimates. In contrast, \method\ decomposes purely by horizon: it recursively bisects each trajectory at its temporal midpoint, independent of the task or its difficulty. The short-to-long ordering thus emerges from the decomposition, and all lengths are trained concurrently across parallel slots rather than in sequential stages.

\clearpage
\section{Ablation Studies}
\label{app:ablation_study}

In this section, we present ablation studies on the key components of \method. 

\subsection{Target Value Sources for the Divide-and-Conquer Loss}
\label{app:target_value_source}
In our divide-and-conquer loss (\Cref{eq:dc_loss}), \method\ computes the target value by multiplying two critic estimates corresponding to the first child and the second child.
We investigate the effect of using either the \textit{online} network ($Q$) or the \textit{EMA} target network ($\bar{Q}$) for each child estimate.
As shown in \cref{fig:target_ablation}, we evaluate all four combinations: $Q \cdot \bar{Q}$, $\bar{Q} \cdot \bar{Q}$, $Q \cdot Q$, and $\bar{Q} \cdot Q$.
Across the \texttt{humanoidmaze-giant} and \texttt{puzzle-4x6} environments, the combination of the online first child and the EMA second child ($Q \cdot \bar{Q}$) achieves the best success rate.
Relying entirely on the online network ($Q \cdot Q$) causes value estimates to collapse, while the other configurations ($\bar{Q} \cdot \bar{Q}$ and $\bar{Q} \cdot Q$) degrade performance on at least one task.

\begin{figure}[ht]
    \centering
    \includegraphics[width=0.55\textwidth]{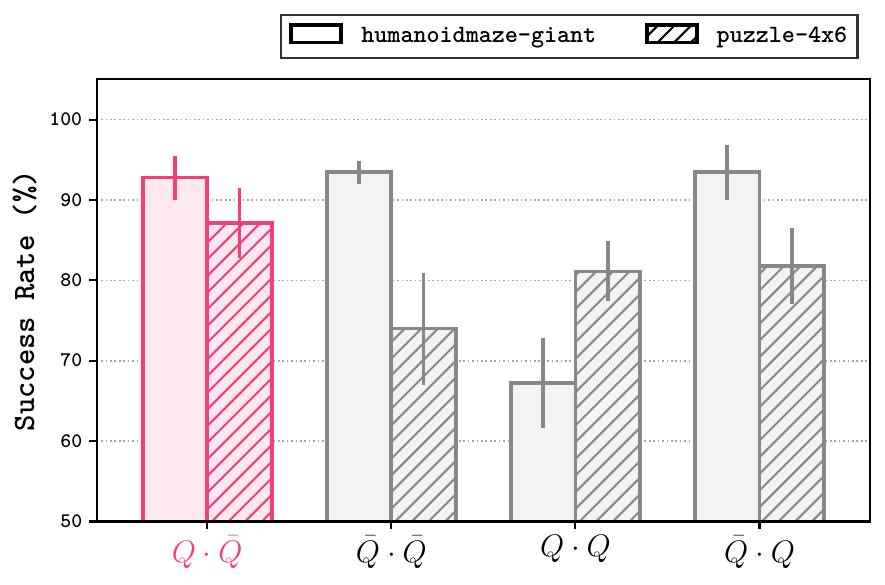}
    \caption{Using the online first child and the EMA second child ($Q \cdot \bar{Q}$) yields the most robust performance on \texttt{humanoidmaze-giant} and \texttt{puzzle-4x6}.}
    \label{fig:target_ablation}
\end{figure}
\clearpage
\subsection{Expectile Parameters}
\label{app:expectile}

We evaluate the sensitivity of \method\ to the expectile parameters $\kappa_{\texttt{dc}}$ (divide-and-conquer loss) and $\kappa_{\texttt{prop}}$ (propagation loss). 
As shown in \cref{fig:kappa_dc_ablation} and \cref{fig:kappa_prop_ablation}, \method\ achieves the best performance with $\kappa_{\texttt{dc}}=0.5$ and $\kappa_{\texttt{prop}}=0.7$ on both \texttt{humanoidmaze-giant} and \texttt{cube-triple}. Deviating from these values generally degrades performance.

\begin{figure}[ht]
    \centering
    \begin{minipage}[t]{0.48\textwidth}
        \centering
        \includegraphics[width=\linewidth]{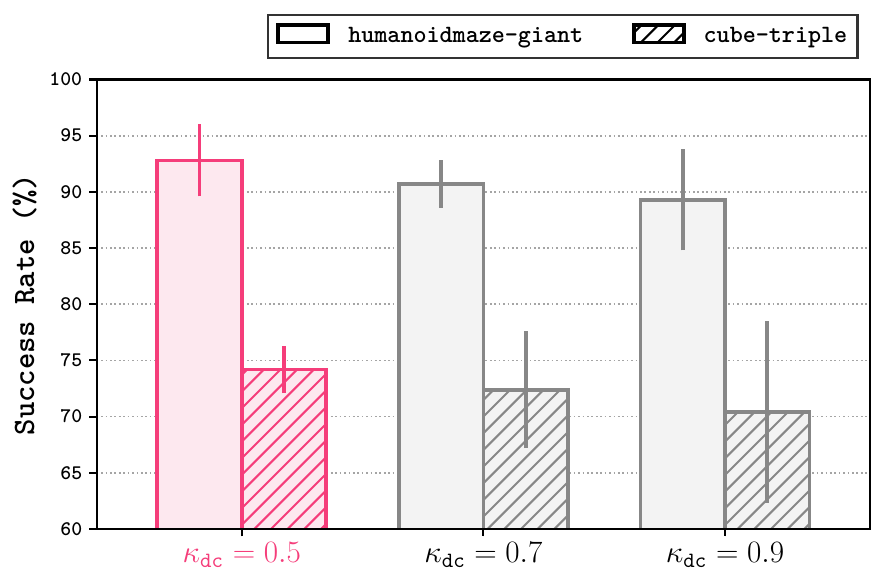}
        \caption{Ablation of the expectile parameter $\kappa_{\texttt{dc}}$ for the divide-and-conquer loss $\mathcal{L}^{\texttt{dc}}$. Performance peaks at $\kappa_{\texttt{dc}}=0.5$ (our default).}
        \label{fig:kappa_dc_ablation}
    \end{minipage}\hfill
    \begin{minipage}[t]{0.48\textwidth}
        \centering
        \includegraphics[width=\linewidth]{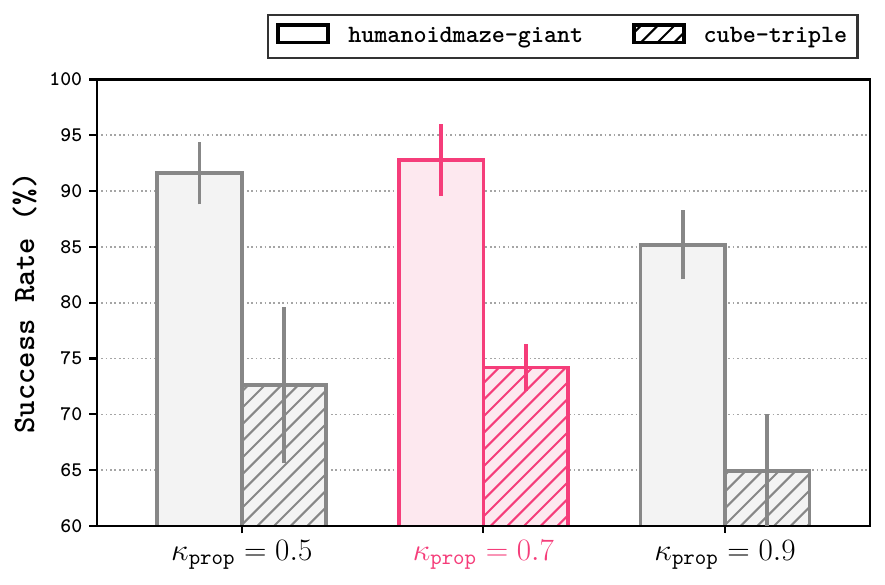}
        \caption{Ablation of the expectile parameter $\kappa_{\texttt{prop}}$ for the propagation loss $\mathcal{L}^{\texttt{prop}}$. Performance peaks at $\kappa_{\texttt{prop}}=0.7$ (our default).}
        \label{fig:kappa_prop_ablation}
    \end{minipage}
\end{figure}

\clearpage
\subsection{Number of Slots}
\label{app:number_of_slots}

We study how the number of slots affects \method's performance and training time. Larger numbers of slots incur greater CPU overhead because they must be maintained concurrently, increasing training time.
Figures~\ref{fig:a3_pareto} and \ref{fig:a3_pareto_c3} provide a Pareto analysis of \method\ and \trl\ {\small($\lambda=0$)} in \texttt{humanoidmaze-giant} and \texttt{cube-triple}.
In \texttt{humanoidmaze-giant}, \method\ reaches 90\% success in under 3 hours, whereas \trl\ remains below 80\% even after 8 hours.
In \texttt{cube-triple}, \method\ reaches 70\% success in under 2 hours, whereas \trl\ stalls below 50\% and does not improve with a larger training budget.
We measure training time on the same GPU (NVIDIA L40S).

\begin{figure}[ht]
    \begin{minipage}[c]{0.48\textwidth}
        \centering
        \captionsetup{width=\linewidth}
        \includegraphics[width=\linewidth]{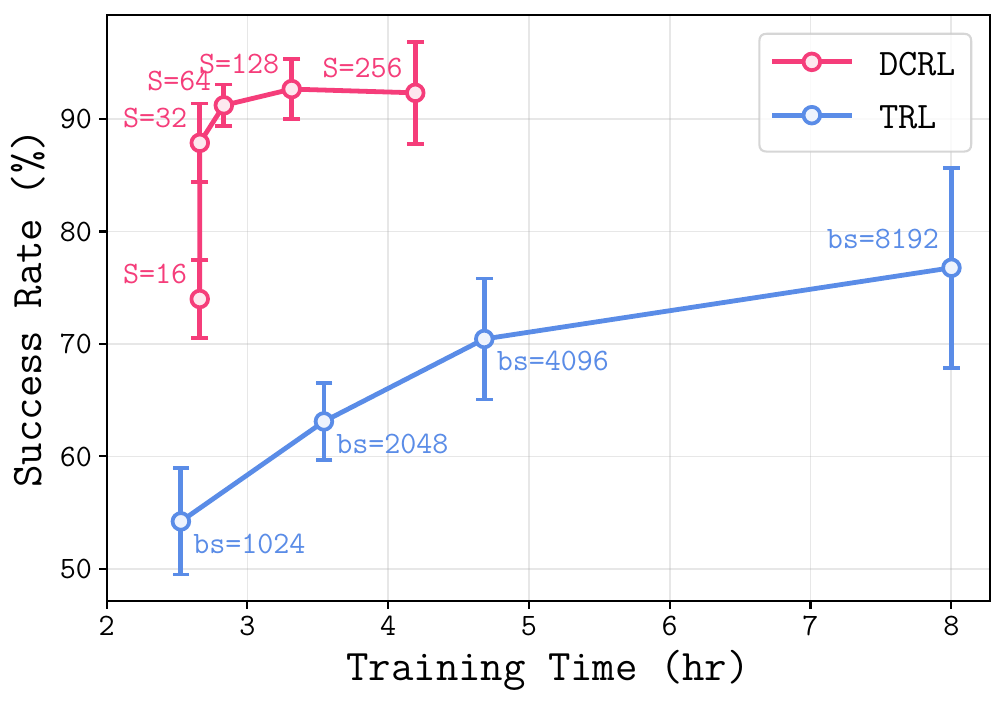}
        \caption{Across training budgets, \method\ outperforms \trl. In \texttt{humanoidmaze-giant}, we plot success rate (\%) against training time (hours), while varying the number of slots (\texttt{S=16,32,64,128,256}) for \method\ at a fixed batch size (1024), and varying the batch size (\texttt{bs=1024,2048,4096,8192}) for \trl.}
        \label{fig:a3_pareto}
    \end{minipage}\hfill
    \begin{minipage}[c]{0.48\textwidth}
        \centering
        \captionsetup{width=\linewidth}
        \includegraphics[width=\linewidth]{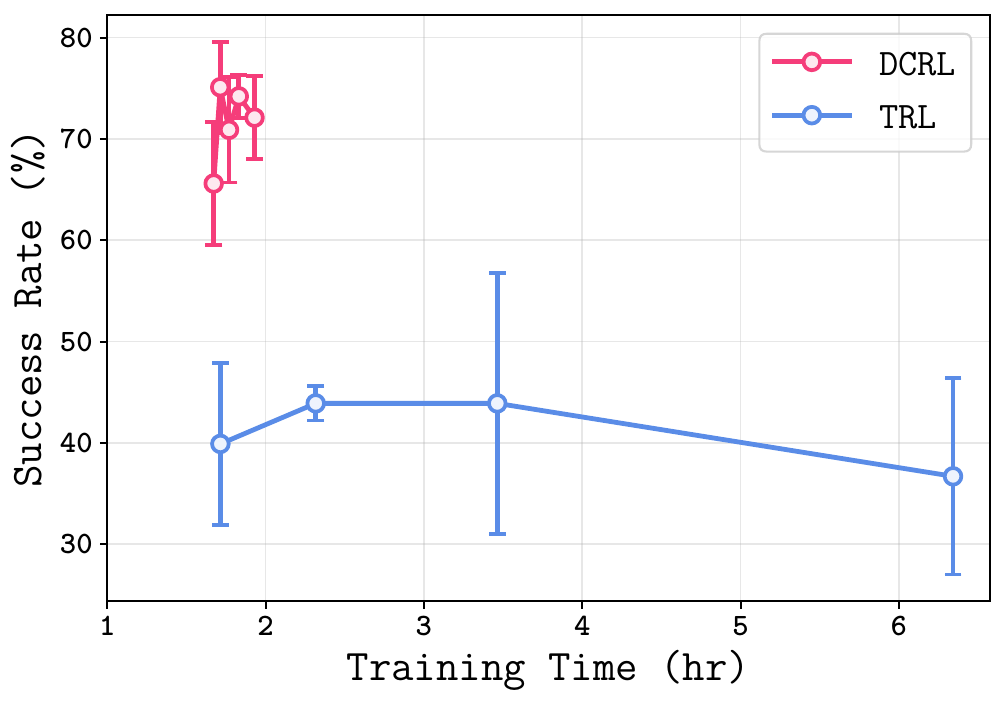}
        \caption{Same Pareto analysis as Fig.~\ref{fig:a3_pareto} in \texttt{cube-triple}. \method\ consistently achieves a better trade-off between performance and training efficiency than \trl. \method\ is trained with varying numbers of slots (\texttt{S}) and a fixed batch size (1024), while \trl\ is trained with varying batch sizes (\texttt{bs}).}
        \label{fig:a3_pareto_c3}
    \end{minipage}
\end{figure}

\clearpage
\section{Divide-and-Conquer Strategies on Cube Manipulation}
\label{app:a3_c3_comparison}

We provide additional comparisons of divide-and-conquer strategies on \texttt{cube-triple}. As shown in \cref{tab:a3_c3_100m_comparison}, the results are mostly consistent with our findings on \texttt{humanoidmaze-giant} (\cref{tab:a3_comparison}). \method\ achieves performance comparable to random splitting and about $1.7\times$ the success rate of the best \trl\ variant, with approximately 7\% lower throughput. The same pattern holds for the ablations: the random split performs no better while cutting throughput by 43\%, curriculum learning remains highly sensitive to $\Delta$, and distance re-weighting fails to improve upon \trl. We measure training throughput for Tables~\ref{tab:a3_comparison}, \ref{tab:dcrl_ablation}, and \ref{tab:a3_c3_100m_comparison} on the same GPU (NVIDIA L40S).

\begin{figure}[ht]
        \centering
        \captionsetup{width=\linewidth}
        \captionof{table}{Comparison of divide-and-conquer strategies on \texttt{cube-triple} using the same batch size (1024). We report success rate (\%) and training throughput (gradient steps per second).}
        \label{tab:a3_c3_100m_comparison}
        \resizebox{0.6\linewidth}{!}{%
        \begin{tabular}{lcc}
            \toprule
            \texttt{Method} & \texttt{Success} $\uparrow$ & \texttt{Throughput} $\uparrow$ \\
            \midrule
            \texttt{Curriculum} \small{($\Delta=25$)} & \texttt{7 \std{4}} & \texttt{\textcolor{highlight}{155} \std{2}} \\
            \texttt{Curriculum} \small{($\Delta=50$)} & \texttt{35 \std{5}} & \texttt{\textcolor{highlight}{154} \std{2}} \\
            \texttt{Curriculum} \small{($\Delta=100$)} & \texttt{15 \std{6}} & \texttt{\textcolor{highlight}{155} \std{2}} \\
            \midrule
            \trl\ \small{($\lambda=0$)} & \texttt{37 \std{8}} & \texttt{\textcolor{highlight}{161} \std{2}} \\
            \trl\ \small{($\lambda=0.01$)} & \texttt{42 \std{3}} & \texttt{\textcolor{highlight}{163} \std{2}} \\
            \trl\ \small{($\lambda=0.1$)} & \texttt{34 \std{4}} & \texttt{\textcolor{highlight}{160} \std{1}} \\
            \midrule
            \method\ \small{(w/ random split)} & \texttt{\textcolor{highlight}{72} \std{4}} & \texttt{86 \std{1}} \\
            \method\ \small{(w/ midpoint split)} & \texttt{\textcolor{highlight}{71} \std{3}} & \texttt{\textcolor{highlight}{152} \std{1}} \\
            \bottomrule
        \end{tabular}%
        }
\end{figure}

\section{Full Experimental Results}
\label{app:experimental_results}
\subsection{CALVIN Benchmark}
\label{app:calvin}
\begin{table}[ht]
    \centering
    \caption{Full results on state-based CALVIN.}
    \label{tab:calvin_results_appendix}
    \resizebox{\textwidth}{!}{%
        \begin{tabular}{@{}lcccccccccccc@{}}
            \toprule
            & & \multicolumn{2}{c}{\textbf{Temporal Difference}} & \multicolumn{2}{c}{\textbf{Monte Carlo}} & \multicolumn{3}{c}{\textbf{Hierarchical}} & \multicolumn{4}{c}{\textbf{Triangle Inequality}} \\
            \cmidrule(lr){3-4} \cmidrule(lr){5-6} \cmidrule(lr){7-9} \cmidrule(lr){10-13}
            \texttt{Metric} & \texttt{BC} & \texttt{IQL} & \texttt{TD-n} & \texttt{MC} & \texttt{CRL} & \texttt{HBC} & \texttt{POR} & \texttt{HIQL} & \texttt{QRL} & \texttt{TMD} & \texttt{TRL} & \textcolor{highlight}{\method} \\
            \midrule
            \texttt{Return} & \texttt{13.5 \std{13.2}} & \texttt{12.5 \std{18.9}} & \texttt{24.4 \std{14.6}} & \texttt{23.3 \std{10.9}} & \texttt{8.2 \std{9.5}} & \texttt{11.0 \std{12.3}} & \texttt{3.8 \std{7.2}} & \texttt{20.2 \std{24.8}} & \texttt{12.5 \std{18.9}} & \texttt{4.8 \std{9.3}} & \texttt{22.8 \std{9.7}} & \texttt{\textcolor{highlight}{35.4} \std{14.4}} \\
            \texttt{Completion} & \texttt{0.0 \std{0.0}} & \texttt{0.0 \std{0.0}} & \texttt{0.0 \std{0.0}} & \texttt{0.0 \std{0.0}} & \texttt{0.0 \std{0.0}} & \texttt{0.0 \std{0.0}} & \texttt{0.0 \std{0.0}} & \texttt{3.0 \std{8.5}} & \texttt{0.0 \std{0.0}} & \texttt{0.0 \std{0.0}} & \texttt{0.0 \std{0.0}} & \texttt{\textcolor{highlight}{8.5} \std{7.1}} \\
            \bottomrule
        \end{tabular}%
    }
\end{table}

\begin{table}[ht]
    \centering
    \caption{Full results on pixel-based CALVIN. We report the success rate
    (\%) of completing at least $k$ consecutive subtasks, together with the
    average number of subtasks completed per rollout.}
    \label{tab:calvin_pixel_results_appendix}
    \resizebox{0.6\textwidth}{!}{%
        \begin{tabular}{@{}lccccc@{}}
            \toprule
                             & \multicolumn{4}{c}{\texttt{\# of Subtasks}}                                                                          & \\
            \cmidrule(lr){2-5}
            \texttt{Method}  & \texttt{1}           & \texttt{2}           & \texttt{3}          & \texttt{4}         & \texttt{Avg.\ Subtasks} \\ \midrule
            \texttt{TD-n}    & \texttt{30 \std{22}} & \texttt{4 \std{3}}   & \texttt{0 \std{1}}  & \texttt{0 \std{0}} & \texttt{0.34 \std{0.24}} \\
            \texttt{TRL}     & \texttt{30 \std{30}} & \texttt{7 \std{9}}   & \texttt{0 \std{0}}  & \texttt{0 \std{0}} & \texttt{0.37 \std{0.39}} \\
            \texttt{HIQL}    & \texttt{\textcolor{highlight}{75} \std{50}} & \texttt{25 \std{50}} & \texttt{0 \std{0}} & \texttt{0 \std{0}} & \texttt{1.00 \std{0.82}} \\
            \textcolor{highlight}{\method} & \texttt{72 \std{26}} & \texttt{\textcolor{highlight}{32} \std{22}} & \texttt{\textcolor{highlight}{8} \std{7}} & \texttt{\textcolor{highlight}{1} \std{2}} & \texttt{\textcolor{highlight}{1.14} \std{0.54}} \\
            \bottomrule
        \end{tabular}%
    }
\end{table}

\clearpage
\subsection{Long-Horizon OGBench Tasks}
\label{app:long_horizon_tasks}
\begin{table}[ht]
    \centering
    \caption{Full results on large-scale, long-horizon OGBench tasks.}
    \label{tab:large_scale_performance_appendix}
    \resizebox{\textwidth}{!}{%
        \begin{tabular}{@{}llcccccccccccccc@{}}
            \toprule
                                 &                  &                    & \multicolumn{4}{c}{\textbf{Temporal Difference}} & \multicolumn{2}{c}{\textbf{Monte Carlo}} & \multicolumn{2}{c}{\textbf{Hierarchical}} & \multicolumn{5}{c}{\textbf{Triangle Inequality}}                                                                                                                                                                                                                                                                             \\
            \cmidrule(lr){4-7} \cmidrule(lr){8-9} \cmidrule(lr){10-11} \cmidrule(lr){12-16}
            \texttt{Environment} & \texttt{Task}    & \texttt{FBC}       & \texttt{IQL}                                     & \texttt{SAC+BC}  & \texttt{TD}                               & \texttt{TD-n}\                                   & \texttt{MC}         & \texttt{CRL}         & \texttt{HIQL}        & \texttt{SHARSA}                             & \texttt{QRL}       & \texttt{TDP}       & \texttt{COE}       & \texttt{TRL}                                & \textcolor{highlight}{\method}              \\ \midrule
            \multirow{6}{*}{\texttt{humanoidmaze-giant-1B}}
                                 & \texttt{task1}   & \texttt{0 \std{0}} & \texttt{0 \std{0}}                               & \texttt{5 \std{3}}                            & \texttt{2 \std{4}}                        & \texttt{68 \std{14}}                             & \texttt{64 \std{9}}                         & \texttt{52 \std{38}} & \texttt{16 \std{8}} & \texttt{23 \std{8}}                         & \texttt{1 \std{1}} & \texttt{0 \std{0}} & \texttt{1 \std{1}} & \texttt{71 \std{15}}                        & \texttt{\textcolor{highlight}{91} \std{6}}  \\
                                 & \texttt{task2}   & \texttt{0 \std{0}} & \texttt{3 \std{7}}                               & \texttt{12 \std{3}}                           & \texttt{8 \std{4}}                        & \texttt{90 \std{13}}                             & \texttt{87 \std{5}}                         & \texttt{63 \std{46}} & \texttt{38 \std{12}} & \texttt{46 \std{14}}                        & \texttt{2 \std{2}} & \texttt{3 \std{3}} & \texttt{2 \std{2}} & \texttt{87 \std{6}}                         & \texttt{\textcolor{highlight}{97} \std{3}}  \\
                                 & \texttt{task3}   & \texttt{0 \std{0}} & \texttt{5 \std{6}}                               & \texttt{8 \std{6}}                            & \texttt{1 \std{1}}                        & \texttt{65 \std{12}}                             & \texttt{83 \std{8}}                         & \texttt{68 \std{46}} & \texttt{26 \std{14}} & \texttt{16 \std{3}}                         & \texttt{2 \std{3}} & \texttt{3 \std{2}} & \texttt{4 \std{2}} & \texttt{44 \std{8}}                         & \texttt{\textcolor{highlight}{90} \std{5}}  \\
                                 & \texttt{task4}   & \texttt{0 \std{0}} & \texttt{2 \std{3}}                               & \texttt{2 \std{3}}                            & \texttt{2 \std{2}}                        & \texttt{78 \std{14}}                             & \texttt{78 \std{8}}                         & \texttt{57 \std{41}} & \texttt{12 \std{6}}  & \texttt{35 \std{7}}                        & \texttt{0 \std{0}} & \texttt{0 \std{0}} & \texttt{1 \std{2}} & \texttt{\textcolor{highlight}{94} \std{4}}  & \texttt{\textcolor{highlight}{96} \std{4}}  \\
                                 & \texttt{task5}   & \texttt{0 \std{0}} & \texttt{3 \std{4}}                               & \texttt{0 \std{0}}                            & \texttt{3 \std{1}}                        & \texttt{89 \std{10}}                             & \texttt{84 \std{11}}                        & \texttt{68 \std{46}} & \texttt{32 \std{7}} & \texttt{89 \std{3}}                        & \texttt{8 \std{8}} & \texttt{4 \std{2}} & \texttt{5 \std{2}} & \texttt{\textcolor{highlight}{99} \std{1}}  & \texttt{\textcolor{highlight}{94} \std{6}}  \\
                                 & \texttt{overall} & \texttt{0 \std{0}} & \texttt{3 \std{4}}                               & \texttt{5 \std{0}}                            & \texttt{3 \std{2}}                        & \texttt{78 \std{4}}                              & \texttt{79 \std{4}}                         & \texttt{62 \std{42}} & \texttt{25 \std{6}} & \texttt{42 \std{3}}                         & \texttt{3 \std{2}} & \texttt{2 \std{0}} & \texttt{2 \std{1}} & \texttt{79 \std{2}}                         & \texttt{\textcolor{highlight}{93} \std{3}}  \\ \midrule
            \multirow{6}{*}{\texttt{puzzle-4x5-1B}}
                                 & \texttt{task1}   & \texttt{0 \std{0}} & \texttt{\textcolor{highlight}{100} \std{0}}      & \texttt{\textcolor{highlight}{95} \std{3}}    & \texttt{84 \std{7}}                       & \texttt{\textcolor{highlight}{100} \std{0}}      & \texttt{\textcolor{highlight}{100} \std{0}} & \texttt{7 \std{5}}   & \texttt{58 \std{25}} & \texttt{\textcolor{highlight}{100} \std{0}} & \texttt{0 \std{0}} & \texttt{0 \std{0}} & \texttt{0 \std{0}} & \texttt{\textcolor{highlight}{100} \std{0}} & \texttt{\textcolor{highlight}{100} \std{0}} \\
                                 & \texttt{task2}   & \texttt{0 \std{0}} & \texttt{0 \std{0}}                               & \texttt{0 \std{0}}                            & \texttt{4 \std{3}}                        & \texttt{\textcolor{highlight}{98} \std{4}}       & \texttt{67 \std{12}}                        & \texttt{0 \std{0}}   & \texttt{0 \std{0}}   & \texttt{\textcolor{highlight}{99} \std{1}}                         & \texttt{0 \std{0}} & \texttt{0 \std{0}} & \texttt{0 \std{0}} & \texttt{\textcolor{highlight}{99} \std{1}}  & \texttt{\textcolor{highlight}{100} \std{0}} \\
                                 & \texttt{task3}   & \texttt{0 \std{0}} & \texttt{0 \std{0}}                               & \texttt{0 \std{0}}                            & \texttt{2 \std{2}}                        & \texttt{84 \std{28}}                             & \texttt{8 \std{10}}                         & \texttt{0 \std{0}}   & \texttt{0 \std{0}}   & \texttt{93 \std{5}}                         & \texttt{0 \std{0}} & \texttt{0 \std{0}} & \texttt{0 \std{0}} & \texttt{\textcolor{highlight}{100} \std{0}} & \texttt{\textcolor{highlight}{100} \std{0}} \\
                                 & \texttt{task4}   & \texttt{0 \std{0}} & \texttt{0 \std{0}}                               & \texttt{0 \std{0}}                            & \texttt{2 \std{2}}                        & \texttt{\textcolor{highlight}{98} \std{3}}       & \texttt{50 \std{9}}                         & \texttt{0 \std{0}}   & \texttt{0 \std{0}}   & \texttt{87 \std{5}}                         & \texttt{0 \std{0}} & \texttt{0 \std{0}} & \texttt{0 \std{0}} & \texttt{\textcolor{highlight}{99} \std{1}}  & \texttt{\textcolor{highlight}{99} \std{1}} \\
                                 & \texttt{task5}   & \texttt{0 \std{0}} & \texttt{0 \std{0}}                               & \texttt{0 \std{0}}                            & \texttt{2 \std{2}}                        & \texttt{81 \std{17}}                             & \texttt{8 \std{11}}                         & \texttt{0 \std{0}}   & \texttt{0 \std{0}}   & \texttt{72 \std{11}}                         & \texttt{0 \std{0}} & \texttt{0 \std{0}} & \texttt{0 \std{0}} & \texttt{88 \std{8}}                         & \texttt{\textcolor{highlight}{99} \std{1}}  \\
                                 & \texttt{overall} & \texttt{0 \std{0}} & \texttt{20 \std{0}}                              & \texttt{19 \std{1}}           & \texttt{19 \std{1}}                       & \texttt{92 \std{8}}                              & \texttt{47 \std{7}}                         & \texttt{1 \std{1}}   & \texttt{12 \std{5}}   & \texttt{90 \std{4}}                         & \texttt{0 \std{0}} & \texttt{0 \std{0}} & \texttt{0 \std{0}} & \texttt{\textcolor{highlight}{97} \std{1}}  & \texttt{\textcolor{highlight}{100} \std{0}} \\ \midrule
            \multirow{6}{*}{\texttt{puzzle-4x6-1B}}
                                 & \texttt{task1}   & \texttt{0 \std{0}} & \texttt{87 \std{9}}                              & \texttt{48 \std{39}}       & \texttt{61 \std{16}}                      & \texttt{\textcolor{highlight}{100} \std{0}}      & \texttt{\textcolor{highlight}{98} \std{4}}  & \texttt{0 \std{0}}   & \texttt{17 \std{12}} & \texttt{\textcolor{highlight}{99} \std{1}} & \texttt{0 \std{0}} & \texttt{0 \std{0}} & \texttt{0 \std{0}} & \texttt{\textcolor{highlight}{100} \std{0}} & \texttt{\textcolor{highlight}{100} \std{0}} \\
                                 & \texttt{task2}   & \texttt{0 \std{0}} & \texttt{0 \std{0}}                               & \texttt{5 \std{10}}                           & \texttt{2 \std{2}}                        & \texttt{39 \std{21}}                             & \texttt{46 \std{30}}                        & \texttt{0 \std{0}}   & \texttt{4 \std{2}}   & \texttt{61 \std{13}}                        & \texttt{0 \std{0}} & \texttt{0 \std{0}} & \texttt{0 \std{0}} & \texttt{66 \std{13}}                        & \texttt{\textcolor{highlight}{99} \std{2}}  \\
                                 & \texttt{task3}   & \texttt{0 \std{0}} & \texttt{0 \std{0}}                               & \texttt{0 \std{0}}                            & \texttt{0 \std{0}}                        & \texttt{81 \std{21}}                             & \texttt{34 \std{10}}                        & \texttt{0 \std{0}}   & \texttt{0 \std{0}}   & \texttt{71 \std{15}}                        & \texttt{0 \std{0}} & \texttt{0 \std{0}} & \texttt{0 \std{0}} & \texttt{67 \std{21}}                        & \texttt{\textcolor{highlight}{100} \std{0}} \\
                                 & \texttt{task4}   & \texttt{0 \std{0}} & \texttt{0 \std{0}}                               & \texttt{0 \std{0}}                            & \texttt{1 \std{1}}                        & \texttt{46 \std{28}}                             & \texttt{5 \std{6}}                          & \texttt{0 \std{0}}   & \texttt{0 \std{0}}   & \texttt{45 \std{7}}                        & \texttt{0 \std{0}} & \texttt{0 \std{0}} & \texttt{0 \std{0}} & \texttt{23 \std{7}}                         & \texttt{\textcolor{highlight}{100} \std{0}}  \\
                                 & \texttt{task5}   & \texttt{0 \std{0}} & \texttt{0 \std{0}}                               & \texttt{0 \std{0}}                            & \texttt{0 \std{0}}                        & \texttt{0 \std{0}}                               & \texttt{0 \std{0}}                          & \texttt{0 \std{0}}   & \texttt{0 \std{0}}   & \texttt{1 \std{2}}                          & \texttt{0 \std{0}} & \texttt{0 \std{0}} & \texttt{0 \std{0}} & \texttt{0 \std{0}}                          & \texttt{\textcolor{highlight}{37} \std{24}} \\
                                 & \texttt{overall} & \texttt{0 \std{0}} & \texttt{17 \std{2}}                              & \texttt{11 \std{8}} & \texttt{13 \std{3}}                       & \texttt{53 \std{11}}                             & \texttt{37 \std{4}}                         & \texttt{0 \std{0}}   & \texttt{4 \std{3}}   & \texttt{55 \std{3}}                         & \texttt{0 \std{0}} & \texttt{0 \std{0}} & \texttt{0 \std{0}} & \texttt{51 \std{5}}                         & \texttt{\textcolor{highlight}{87} \std{5}}  \\ \midrule
            \multirow{6}{*}{\texttt{cube-quadruple-100M}}
                                 & \texttt{task1}   & \texttt{3 \std{2}} & \texttt{94 \std{3}} & \texttt{59 \std{13}}  & \texttt{66 \std{6}} & \texttt{51 \std{10}} & \texttt{5 \std{4}} & \texttt{56 \std{6}} & \texttt{88 \std{7}} & \texttt{\textcolor{highlight}{100} \std{0}} & \texttt{0 \std{0}} & \texttt{0 \std{0}} & \texttt{0 \std{0}} & \texttt{1 \std{1}} & \texttt{59 \std{7}} \\
                                 & \texttt{task2}   & \texttt{0 \std{1}} & \texttt{39 \std{4}} & \texttt{77 \std{8}}  & \texttt{25 \std{9}} & \texttt{44 \std{9}} & \texttt{0 \std{0}} & \texttt{46 \std{8}} & \texttt{79 \std{1}} & \texttt{\textcolor{highlight}{100} \std{0}} & \texttt{0 \std{0}} & \texttt{0 \std{0}} & \texttt{0 \std{0}} & \texttt{0 \std{0}} & \texttt{69 \std{8}} \\
                                 & \texttt{task3}   & \texttt{0 \std{0}} & \texttt{63 \std{5}} & \texttt{34 \std{9}} & \texttt{8 \std{3}} & \texttt{9 \std{4}} & \texttt{0 \std{0}} & \texttt{2 \std{1}} & \texttt{69 \std{13}} & \texttt{\textcolor{highlight}{77} \std{3}} & \texttt{0 \std{0}} & \texttt{0 \std{0}} & \texttt{0 \std{0}} & \texttt{0 \std{0}} & \texttt{25 \std{12}} \\
                                 & \texttt{task4}   & \texttt{0 \std{0}} & \texttt{5 \std{2}} & \texttt{5 \std{2}} & \texttt{1 \std{1}} & \texttt{4 \std{3}} & \texttt{0 \std{0}} & \texttt{1 \std{0}} & \texttt{32 \std{9}} & \texttt{\textcolor{highlight}{34} \std{4}} & \texttt{0 \std{0}} & \texttt{0 \std{0}} & \texttt{0 \std{0}} & \texttt{0 \std{0}} & \texttt{7 \std{4}} \\
                                 & \texttt{task5}   & \texttt{0 \std{0}} & \texttt{1 \std{1}} & \texttt{21 \std{8}} & \texttt{1 \std{1}} & \texttt{12 \std{7}} & \texttt{0 \std{0}} & \texttt{2 \std{2}} & \texttt{38 \std{5}} & \texttt{\textcolor{highlight}{61} \std{6}} & \texttt{0 \std{0}} & \texttt{0 \std{0}} & \texttt{0 \std{0}} & \texttt{0 \std{0}} & \texttt{27 \std{3}} \\
                                 & \texttt{overall} & \texttt{1 \std{0}} & \texttt{41 \std{3}} & \texttt{39 \std{4}} & \texttt{20 \std{2}} & \texttt{21 \std{6}} & \texttt{1 \std{1}} & \texttt{22 \std{2}} & \texttt{62 \std{6}} & \texttt{\textcolor{highlight}{74} \std{2}} & \texttt{0 \std{0}} & \texttt{0 \std{0}} & \texttt{0 \std{0}} & \texttt{0 \std{0}} & \texttt{37 \std{4}} \\ \midrule
            \multirow{6}{*}{\texttt{cube-octuple-1B}}
                                 & \texttt{task1}   & \texttt{0 \std{0}} & \texttt{0 \std{0}}                              & \texttt{0 \std{0}}                         & \texttt{0 \std{0}}                      & \texttt{0 \std{0}}      & \texttt{0 \std{0}}  & \texttt{0 \std{0}}   & \texttt{1 \std{2}} & \texttt{\textcolor{highlight}{74} \std{5}} & \texttt{0 \std{0}} & \texttt{0 \std{0}} & \texttt{0 \std{0}} & \texttt{0 \std{0}} & \texttt{23 \std{7}} \\
                                 & \texttt{task2}   & \texttt{0 \std{0}} & \texttt{0 \std{0}}                                                     & \texttt{0 \std{0}} & \texttt{0 \std{0}}                        & \texttt{0 \std{0}}                             & \texttt{0 \std{0}}                        & \texttt{0 \std{0}}   & \texttt{0 \std{0}}   & \texttt{0 \std{0}}                        & \texttt{0 \std{0}} & \texttt{0 \std{0}} & \texttt{0 \std{0}} & \texttt{0 \std{0}}                        & \texttt{0 \std{0}}  \\
                                 & \texttt{task3}   & \texttt{0 \std{0}} & \texttt{0 \std{0}}                                                       & \texttt{0 \std{0}} & \texttt{0 \std{0}}                        & \texttt{0 \std{0}}                             & \texttt{0 \std{0}}                        & \texttt{0 \std{0}}   & \texttt{0 \std{0}}   & \texttt{1 \std{1}}                        & \texttt{0 \std{0}} & \texttt{0 \std{0}} & \texttt{0 \std{0}} & \texttt{0 \std{0}}                        & \texttt{0 \std{0}} \\
                                 & \texttt{task4}   & \texttt{0 \std{0}} & \texttt{0 \std{0}}                                                        & \texttt{0 \std{0}} & \texttt{0 \std{0}}                        & \texttt{0 \std{0}}                             & \texttt{0 \std{0}}                          & \texttt{0 \std{0}}   & \texttt{0 \std{0}}   & \texttt{0 \std{0}}                        & \texttt{0 \std{0}} & \texttt{0 \std{0}} & \texttt{0 \std{0}} & \texttt{0 \std{0}}                         & \texttt{0 \std{0}}  \\
                                 & \texttt{task5}   & \texttt{0 \std{0}} & \texttt{0 \std{0}}                                     & \texttt{0 \std{0}} & \texttt{0 \std{0}}                        & \texttt{0 \std{0}}                               & \texttt{0 \std{0}}                          & \texttt{0 \std{0}}   & \texttt{0 \std{0}}   & \texttt{0 \std{0}}                          & \texttt{0 \std{0}} & \texttt{0 \std{0}} & \texttt{0 \std{0}} & \texttt{0 \std{0}}                          & \texttt{0 \std{0}} \\
                                 & \texttt{overall} & \texttt{0 \std{0}} & \texttt{0 \std{0}}                              & \texttt{0 \std{0}}                          & \texttt{0 \std{0}}                       & \texttt{0 \std{0}}                             & \texttt{0 \std{0}}                         & \texttt{0 \std{0}}   & \texttt{0 \std{0}}   & \texttt{\textcolor{highlight}{15} \std{1}}                         & \texttt{0 \std{0}} & \texttt{0 \std{0}} & \texttt{0 \std{0}} & \texttt{0 \std{0}}                         & \texttt{5 \std{1}}  \\
            \bottomrule
        \end{tabular}%
    }
\end{table}

\clearpage
\subsection{Propagation Horizon Comparison}
\label{app:propagation_horizon_comparison}
\begin{table}[ht]
\NStepTable{1}{%
    \multirow{6}{*}{\texttt{humanoidmaze-giant-1B}}
    & \texttt{task1}        & \texttt{0 \std{0}}                        & \texttt{0 \std{0}}                        & \texttt{0 \std{0}} \\
    & \texttt{task2}        & \texttt{0 \std{0}}                        & \texttt{0 \std{0}}                        & \texttt{0 \std{0}} \\
    & \texttt{task3}        & \texttt{0 \std{0}}                        & \texttt{0 \std{0}}                        & \texttt{0 \std{0}} \\
    & \texttt{task4}        & \texttt{0 \std{0}}                        & \texttt{0 \std{0}}                        & \texttt{0 \std{0}} \\
    & \texttt{task5}        & \texttt{0 \std{0}}                        & \texttt{\textcolor{highlight}{42} \std{18}}& \texttt{0 \std{0}} \\
    & \texttt{overall}      & \texttt{0 \std{0}}                        & \texttt{\textcolor{highlight}{8} \std{4}} & \texttt{0 \std{0}} \\ \midrule
    \multirow{6}{*}{\texttt{puzzle-4x5-1B}}
    & \texttt{task1}        & \texttt{\textcolor{highlight}{100} \std{0}}& \texttt{\textcolor{highlight}{100} \std{1}}& \texttt{\textcolor{highlight}{100} \std{0}} \\
    & \texttt{task2}        & \texttt{1 \std{2}}                        & \texttt{\textcolor{highlight}{85} \std{16}}& \texttt{0 \std{0}} \\
    & \texttt{task3}        & \texttt{0 \std{0}}                        & \texttt{\textcolor{highlight}{67} \std{29}}& \texttt{0 \std{0}} \\
    & \texttt{task4}        & \texttt{0 \std{0}}                        & \texttt{\textcolor{highlight}{96} \std{6}}& \texttt{0 \std{0}} \\
    & \texttt{task5}        & \texttt{0 \std{0}}                        & \texttt{\textcolor{highlight}{63} \std{24}}& \texttt{0 \std{0}} \\
    & \texttt{overall}      & \texttt{20 \std{0}}                       & \texttt{\textcolor{highlight}{82} \std{15}}& \texttt{20 \std{0}} \\ \midrule
    \multirow{6}{*}{\texttt{puzzle-4x6-1B}}
    & \texttt{task1}        & \texttt{\textcolor{highlight}{99} \std{1}}& \texttt{\textcolor{highlight}{99} \std{2}}& \texttt{\textcolor{highlight}{99} \std{2}} \\
    & \texttt{task2}        & \texttt{6 \std{5}}                        & \texttt{\textcolor{highlight}{28} \std{10}}& \texttt{3 \std{4}} \\
    & \texttt{task3}        & \texttt{0 \std{0}}                        & \texttt{\textcolor{highlight}{23} \std{13}}& \texttt{0 \std{0}} \\
    & \texttt{task4}        & \texttt{0 \std{0}}                        & \texttt{\textcolor{highlight}{13} \std{12}}& \texttt{0 \std{0}} \\
    & \texttt{task5}        & \texttt{0 \std{0}}                        & \texttt{0 \std{0}}                        & \texttt{0 \std{0}} \\
    & \texttt{overall}      & \texttt{21 \std{1}}                       & \texttt{\textcolor{highlight}{33} \std{4}}& \texttt{20 \std{1}} \\ \midrule
    \multirow{6}{*}{\texttt{cube-triple-100M}}
    & \texttt{single\_pnp}  & \texttt{\textcolor{highlight}{95} \std{7}}& \texttt{68 \std{14}}                      & \texttt{89 \std{8}} \\
    & \texttt{triple\_pnp}  & \texttt{\textcolor{highlight}{61} \std{14}}& \texttt{15 \std{8}}                       & \texttt{42 \std{13}} \\
    & \texttt{from\_stack}  & \texttt{\textcolor{highlight}{54} \std{9}}& \texttt{27 \std{11}}                      & \texttt{30 \std{11}} \\
    & \texttt{cycle}        & \texttt{\textcolor{highlight}{26} \std{8}}& \texttt{2 \std{2}}                        & \texttt{5 \std{2}} \\
    & \texttt{stack}        & \texttt{26 \std{9}}                       & \texttt{2 \std{3}}                        & \texttt{\textcolor{highlight}{36} \std{6}} \\
    & \texttt{overall}      & \texttt{\textcolor{highlight}{52} \std{5}}& \texttt{23 \std{5}}                       & \texttt{40 \std{5}} \\
}
\hfill
\NStepTable{5}{%
    \multirow{6}{*}{\texttt{humanoidmaze-giant-1B}}
    & \texttt{task1}        & \texttt{0 \std{0}}                        & \texttt{0 \std{0}}                        & \texttt{0 \std{0}} \\
    & \texttt{task2}        & \texttt{0 \std{0}}                        & \texttt{0 \std{0}}                        & \texttt{\textcolor{highlight}{2} \std{2}} \\
    & \texttt{task3}        & \texttt{0 \std{0}}                        & \texttt{0 \std{0}}                        & \texttt{\textcolor{highlight}{1} \std{2}} \\
    & \texttt{task4}        & \texttt{0 \std{0}}                        & \texttt{0 \std{0}}                        & \texttt{\textcolor{highlight}{1} \std{2}} \\
    & \texttt{task5}        & \texttt{0 \std{0}}                        & \texttt{\textcolor{highlight}{54} \std{18}}& \texttt{0 \std{0}} \\
    & \texttt{overall}      & \texttt{0 \std{0}}                        & \texttt{\textcolor{highlight}{11} \std{4}}& \texttt{1 \std{1}} \\ \midrule
    \multirow{6}{*}{\texttt{puzzle-4x5-1B}}
    & \texttt{task1}        & \texttt{\textcolor{highlight}{100} \std{0}}& \texttt{\textcolor{highlight}{100} \std{0}}& \texttt{\textcolor{highlight}{100} \std{0}} \\
    & \texttt{task2}        & \texttt{80 \std{12}}                      & \texttt{\textcolor{highlight}{95} \std{6}}& \texttt{82 \std{12}} \\
    & \texttt{task3}        & \texttt{34 \std{17}}                      & \texttt{\textcolor{highlight}{87} \std{14}}& \texttt{27 \std{17}} \\
    & \texttt{task4}        & \texttt{67 \std{14}}                      & \texttt{\textcolor{highlight}{97} \std{5}}& \texttt{42 \std{17}} \\
    & \texttt{task5}        & \texttt{15 \std{10}}                      & \texttt{\textcolor{highlight}{75} \std{21}}& \texttt{10 \std{9}} \\
    & \texttt{overall}      & \texttt{59 \std{9}}                       & \texttt{\textcolor{highlight}{91} \std{9}}& \texttt{52 \std{8}} \\ \midrule
    \multirow{6}{*}{\texttt{puzzle-4x6-1B}}
    & \texttt{task1}        & \texttt{\textcolor{highlight}{100} \std{0}}& \texttt{\textcolor{highlight}{96} \std{6}}& \texttt{\textcolor{highlight}{100} \std{0}} \\
    & \texttt{task2}        & \texttt{\textcolor{highlight}{70} \std{13}}& \texttt{31 \std{17}}                      & \texttt{61 \std{14}} \\
    & \texttt{task3}        & \texttt{5 \std{4}}                        & \texttt{\textcolor{highlight}{76} \std{17}}& \texttt{4 \std{5}} \\
    & \texttt{task4}        & \texttt{0 \std{0}}                        & \texttt{\textcolor{highlight}{58} \std{20}}& \texttt{0 \std{0}} \\
    & \texttt{task5}        & \texttt{0 \std{0}}                        & \texttt{0 \std{0}}                        & \texttt{0 \std{0}} \\
    & \texttt{overall}      & \texttt{35 \std{3}}                       & \texttt{\textcolor{highlight}{52} \std{6}}& \texttt{33 \std{2}} \\ \midrule
    \multirow{6}{*}{\texttt{cube-triple-100M}}
    & \texttt{single\_pnp}  & \texttt{\textcolor{highlight}{85} \std{8}}& \texttt{75 \std{8}}                       & \texttt{\textcolor{highlight}{82} \std{8}} \\
    & \texttt{triple\_pnp}  & \texttt{\textcolor{highlight}{77} \std{10}}& \texttt{31 \std{10}}                      & \texttt{\textcolor{highlight}{79} \std{9}} \\
    & \texttt{from\_stack}  & \texttt{69 \std{9}}                       & \texttt{24 \std{10}}                      & \texttt{\textcolor{highlight}{73} \std{10}} \\
    & \texttt{cycle}        & \texttt{\textcolor{highlight}{47} \std{11}}& \texttt{4 \std{4}}                        & \texttt{40 \std{10}} \\
    & \texttt{stack}        & \texttt{\textcolor{highlight}{77} \std{8}}& \texttt{17 \std{8}}                       & \texttt{65 \std{10}} \\
    & \texttt{overall}      & \texttt{\textcolor{highlight}{71} \std{4}}& \texttt{30 \std{4}}                       & \texttt{\textcolor{highlight}{68} \std{6}} \\
}
\end{table}

\begin{table}[ht]
\NStepTable{10}{%
    \multirow{6}{*}{\texttt{humanoidmaze-giant-1B}}
    & \texttt{task1}        & \texttt{12 \std{9}}                       & \texttt{3 \std{4}}                        & \texttt{\textcolor{highlight}{59} \std{10}} \\
    & \texttt{task2}        & \texttt{4 \std{7}}                        & \texttt{20 \std{21}}                      & \texttt{\textcolor{highlight}{84} \std{7}} \\
    & \texttt{task3}        & \texttt{2 \std{3}}                        & \texttt{2 \std{4}}                        & \texttt{\textcolor{highlight}{45} \std{8}} \\
    & \texttt{task4}        & \texttt{13 \std{10}}                      & \texttt{14 \std{18}}                      & \texttt{\textcolor{highlight}{47} \std{10}} \\
    & \texttt{task5}        & \texttt{2 \std{4}}                        & \texttt{\textcolor{highlight}{74} \std{17}}& \texttt{5 \std{4}} \\
    & \texttt{overall}      & \texttt{6 \std{3}}                        & \texttt{23 \std{8}}                       & \texttt{\textcolor{highlight}{48} \std{4}} \\ \midrule
    \multirow{6}{*}{\texttt{puzzle-4x5-1B}}
    & \texttt{task1}        & \texttt{\textcolor{highlight}{100} \std{0}}& \texttt{\textcolor{highlight}{100} \std{0}}& \texttt{\textcolor{highlight}{100} \std{0}} \\
    & \texttt{task2}        & \texttt{\textcolor{highlight}{99} \std{2}}& \texttt{\textcolor{highlight}{99} \std{4}}& \texttt{\textcolor{highlight}{100} \std{0}} \\
    & \texttt{task3}        & \texttt{\textcolor{highlight}{95} \std{7}}& \texttt{\textcolor{highlight}{95} \std{7}}& \texttt{\textcolor{highlight}{100} \std{0}} \\
    & \texttt{task4}        & \texttt{\textcolor{highlight}{96} \std{4}}& \texttt{\textcolor{highlight}{99} \std{2}}& \texttt{\textcolor{highlight}{95} \std{8}} \\
    & \texttt{task5}        & \texttt{48 \std{12}}                      & \texttt{\textcolor{highlight}{91} \std{10}}& \texttt{\textcolor{highlight}{94} \std{6}} \\
    & \texttt{overall}      & \texttt{88 \std{3}}                       & \texttt{\textcolor{highlight}{97} \std{4}}& \texttt{\textcolor{highlight}{98} \std{2}} \\ \midrule
    \multirow{6}{*}{\texttt{puzzle-4x6-1B}}
    & \texttt{task1}        & \texttt{\textcolor{highlight}{100} \std{0}}& \texttt{\textcolor{highlight}{100} \std{1}}& \texttt{\textcolor{highlight}{100} \std{0}} \\
    & \texttt{task2}        & \texttt{\textcolor{highlight}{72} \std{9}}& \texttt{32 \std{20}}                      & \texttt{58 \std{24}} \\
    & \texttt{task3}        & \texttt{13 \std{12}}                      & \texttt{\textcolor{highlight}{81} \std{20}}& \texttt{55 \std{22}} \\
    & \texttt{task4}        & \texttt{0 \std{1}}                        & \texttt{\textcolor{highlight}{83} \std{11}}& \texttt{14 \std{11}} \\
    & \texttt{task5}        & \texttt{0 \std{0}}                        & \texttt{0 \std{0}}                        & \texttt{0 \std{0}} \\
    & \texttt{overall}      & \texttt{37 \std{3}}                       & \texttt{\textcolor{highlight}{59} \std{4}}& \texttt{45 \std{8}} \\ \midrule
    \multirow{6}{*}{\texttt{cube-triple-100M}}
    & \texttt{single\_pnp}  & \texttt{77 \std{9}}                       & \texttt{70 \std{9}}                       & \texttt{\textcolor{highlight}{86} \std{7}} \\
    & \texttt{triple\_pnp}  & \texttt{69 \std{8}}                       & \texttt{34 \std{14}}                      & \texttt{\textcolor{highlight}{88} \std{8}} \\
    & \texttt{from\_stack}  & \texttt{67 \std{9}}                       & \texttt{32 \std{9}}                       & \texttt{\textcolor{highlight}{86} \std{8}} \\
    & \texttt{cycle}        & \texttt{\textcolor{highlight}{42} \std{14}}& \texttt{4 \std{4}}                        & \texttt{\textcolor{highlight}{42} \std{11}} \\
    & \texttt{stack}        & \texttt{78 \std{7}}                       & \texttt{21 \std{6}}                       & \texttt{\textcolor{highlight}{83} \std{9}} \\
    & \texttt{overall}      & \texttt{67 \std{5}}                       & \texttt{32 \std{6}}                       & \texttt{\textcolor{highlight}{77} \std{5}} \\
}
\hfill
\NStepTable{25}{%
    \multirow{6}{*}{\texttt{humanoidmaze-giant-1B}}
    & \texttt{task1}        & \texttt{41 \std{14}}                      & \texttt{15 \std{13}}                      & \texttt{\textcolor{highlight}{89} \std{5}} \\
    & \texttt{task2}        & \texttt{54 \std{16}}                      & \texttt{52 \std{12}}                      & \texttt{\textcolor{highlight}{97} \std{3}} \\
    & \texttt{task3}        & \texttt{42 \std{11}}                      & \texttt{12 \std{12}}                      & \texttt{\textcolor{highlight}{80} \std{11}} \\
    & \texttt{task4}        & \texttt{43 \std{16}}                      & \texttt{27 \std{18}}                      & \texttt{\textcolor{highlight}{88} \std{8}} \\
    & \texttt{task5}        & \texttt{50 \std{21}}                      & \texttt{83 \std{6}}                       & \texttt{\textcolor{highlight}{88} \std{14}} \\
    & \texttt{overall}      & \texttt{46 \std{7}}                       & \texttt{38 \std{7}}                       & \texttt{\textcolor{highlight}{88} \std{3}} \\ \midrule
    \multirow{6}{*}{\texttt{puzzle-4x5-1B}}
    & \texttt{task1}        & \texttt{\textcolor{highlight}{100} \std{0}}& \texttt{\textcolor{highlight}{100} \std{0}}& \texttt{\textcolor{highlight}{100} \std{0}} \\
    & \texttt{task2}        & \texttt{\textcolor{highlight}{99} \std{2}}& \texttt{\textcolor{highlight}{98} \std{3}}& \texttt{\textcolor{highlight}{100} \std{0}} \\
    & \texttt{task3}        & \texttt{\textcolor{highlight}{99} \std{2}}& \texttt{91 \std{16}}                      & \texttt{\textcolor{highlight}{100} \std{1}} \\
    & \texttt{task4}        & \texttt{\textcolor{highlight}{98} \std{2}}& \texttt{\textcolor{highlight}{97} \std{2}}& \texttt{\textcolor{highlight}{99} \std{2}} \\
    & \texttt{task5}        & \texttt{92 \std{6}}                       & \texttt{93 \std{12}}                      & \texttt{\textcolor{highlight}{100} \std{1}} \\
    & \texttt{overall}      & \texttt{\textcolor{highlight}{98} \std{2}}& \texttt{\textcolor{highlight}{96} \std{6}}& \texttt{\textcolor{highlight}{100} \std{1}} \\ \midrule
    \multirow{6}{*}{\texttt{puzzle-4x6-1B}}
    & \texttt{task1}        & \texttt{\textcolor{highlight}{100} \std{0}}& \texttt{\textcolor{highlight}{99} \std{2}}& \texttt{\textcolor{highlight}{100} \std{0}} \\
    & \texttt{task2}        & \texttt{\textcolor{highlight}{63} \std{16}}& \texttt{33 \std{16}}                      & \texttt{\textcolor{highlight}{62} \std{28}} \\
    & \texttt{task3}        & \texttt{56 \std{17}}                      & \texttt{89 \std{4}}                       & \texttt{\textcolor{highlight}{99} \std{2}} \\
    & \texttt{task4}        & \texttt{12 \std{11}}                      & \texttt{83 \std{14}}                      & \texttt{\textcolor{highlight}{88} \std{11}} \\
    & \texttt{task5}        & \texttt{0 \std{0}}                        & \texttt{0 \std{0}}                        & \texttt{\textcolor{highlight}{6} \std{7}} \\
    & \texttt{overall}      & \texttt{46 \std{5}}                       & \texttt{61 \std{4}}                       & \texttt{\textcolor{highlight}{71} \std{7}} \\ \midrule
    \multirow{6}{*}{\texttt{cube-triple-100M}}
    & \texttt{single\_pnp}  & \texttt{65 \std{15}}                      & \texttt{57 \std{14}}                      & \texttt{\textcolor{highlight}{82} \std{15}} \\
    & \texttt{triple\_pnp}  & \texttt{44 \std{15}}                      & \texttt{22 \std{7}}                       & \texttt{\textcolor{highlight}{73} \std{8}} \\
    & \texttt{from\_stack}  & \texttt{49 \std{11}}                      & \texttt{19 \std{11}}                      & \texttt{\textcolor{highlight}{76} \std{6}} \\
    & \texttt{cycle}        & \texttt{20 \std{10}}                      & \texttt{2 \std{2}}                        & \texttt{\textcolor{highlight}{41} \std{10}} \\
    & \texttt{stack}        & \texttt{72 \std{5}}                       & \texttt{12 \std{4}}                       & \texttt{\textcolor{highlight}{83} \std{8}} \\
    & \texttt{overall}      & \texttt{50 \std{8}}                       & \texttt{22 \std{4}}                       & \texttt{\textcolor{highlight}{71} \std{5}} \\
}
\end{table}

\begin{table}[ht]
\NStepTable{50}{%
    \multirow{6}{*}{\texttt{humanoidmaze-giant-1B}}
    & \texttt{task1}        & \texttt{60 \std{9}}                       & \texttt{38 \std{16}}                      & \texttt{\textcolor{highlight}{89} \std{8}} \\
    & \texttt{task2}        & \texttt{80 \std{16}}                      & \texttt{75 \std{11}}                      & \texttt{\textcolor{highlight}{97} \std{3}} \\
    & \texttt{task3}        & \texttt{60 \std{19}}                      & \texttt{35 \std{9}}                       & \texttt{\textcolor{highlight}{89} \std{6}} \\
    & \texttt{task4}        & \texttt{64 \std{22}}                      & \texttt{46 \std{19}}                      & \texttt{\textcolor{highlight}{95} \std{3}} \\
    & \texttt{task5}        & \texttt{\textcolor{highlight}{96} \std{5}}& \texttt{\textcolor{highlight}{93} \std{6}}& \texttt{\textcolor{highlight}{96} \std{3}} \\
    & \texttt{overall}      & \texttt{72 \std{6}}                       & \texttt{57 \std{8}}                       & \texttt{\textcolor{highlight}{93} \std{3}} \\ \midrule
    \multirow{6}{*}{\texttt{puzzle-4x5-1B}}
    & \texttt{task1}        & \texttt{\textcolor{highlight}{100} \std{0}}& \texttt{\textcolor{highlight}{99} \std{1}}& \texttt{\textcolor{highlight}{100} \std{0}} \\
    & \texttt{task2}        & \texttt{\textcolor{highlight}{100} \std{0}}& \texttt{\textcolor{highlight}{97} \std{5}}& \texttt{\textcolor{highlight}{100} \std{0}} \\
    & \texttt{task3}        & \texttt{93 \std{20}}                      & \texttt{\textcolor{highlight}{98} \std{4}}& \texttt{\textcolor{highlight}{100} \std{1}} \\
    & \texttt{task4}        & \texttt{\textcolor{highlight}{97} \std{4}}& \texttt{\textcolor{highlight}{98} \std{4}}& \texttt{\textcolor{highlight}{100} \std{1}} \\
    & \texttt{task5}        & \texttt{83 \std{28}}                      & \texttt{91 \std{8}}                       & \texttt{\textcolor{highlight}{100} \std{1}} \\
    & \texttt{overall}      & \texttt{\textcolor{highlight}{95} \std{10}}& \texttt{\textcolor{highlight}{97} \std{2}}& \texttt{\textcolor{highlight}{100} \std{0}} \\ \midrule
    \multirow{6}{*}{\texttt{puzzle-4x6-1B}}
    & \texttt{task1}        & \texttt{\textcolor{highlight}{100} \std{0}}& \texttt{\textcolor{highlight}{98} \std{3}}& \texttt{\textcolor{highlight}{100} \std{0}} \\
    & \texttt{task2}        & \texttt{27 \std{8}}                       & \texttt{35 \std{28}}                      & \texttt{\textcolor{highlight}{90} \std{11}} \\
    & \texttt{task3}        & \texttt{68 \std{16}}                      & \texttt{91 \std{9}}                       & \texttt{\textcolor{highlight}{100} \std{0}} \\
    & \texttt{task4}        & \texttt{52 \std{14}}                      & \texttt{\textcolor{highlight}{92} \std{6}}& \texttt{\textcolor{highlight}{96} \std{2}} \\
    & \texttt{task5}        & \texttt{0 \std{0}}                        & \texttt{0 \std{1}}                        & \texttt{\textcolor{highlight}{30} \std{11}} \\
    & \texttt{overall}      & \texttt{49 \std{4}}                       & \texttt{63 \std{8}}                       & \texttt{\textcolor{highlight}{83} \std{4}} \\ \midrule
    \multirow{6}{*}{\texttt{cube-triple-100M}}
    & \texttt{single\_pnp}  & \texttt{65 \std{12}}                      & \texttt{76 \std{8}}                       & \texttt{\textcolor{highlight}{87} \std{6}} \\
    & \texttt{triple\_pnp}  & \texttt{45 \std{19}}                      & \texttt{32 \std{12}}                      & \texttt{\textcolor{highlight}{81} \std{7}} \\
    & \texttt{from\_stack}  & \texttt{42 \std{10}}                      & \texttt{32 \std{10}}                      & \texttt{\textcolor{highlight}{73} \std{9}} \\
    & \texttt{cycle}        & \texttt{16 \std{7}}                       & \texttt{6 \std{6}}                        & \texttt{\textcolor{highlight}{39} \std{10}} \\
    & \texttt{stack}        & \texttt{61 \std{14}}                      & \texttt{27 \std{8}}                       & \texttt{\textcolor{highlight}{70} \std{11}} \\
    & \texttt{overall}      & \texttt{46 \std{7}}                       & \texttt{35 \std{3}}                       & \texttt{\textcolor{highlight}{70} \std{5}} \\
}
\hfill
\NStepTable{100}{%
    \multirow{6}{*}{\texttt{humanoidmaze-giant-1B}}
    & \texttt{task1}        & \texttt{59 \std{12}}                      & \texttt{68 \std{10}}                      & \texttt{\textcolor{highlight}{87} \std{8}} \\
    & \texttt{task2}        & \texttt{89 \std{6}}                       & \texttt{\textcolor{highlight}{95} \std{4}}& \texttt{\textcolor{highlight}{97} \std{3}} \\
    & \texttt{task3}        & \texttt{57 \std{11}}                      & \texttt{52 \std{12}}                      & \texttt{\textcolor{highlight}{92} \std{5}} \\
    & \texttt{task4}        & \texttt{79 \std{11}}                      & \texttt{74 \std{15}}                      & \texttt{\textcolor{highlight}{94} \std{6}} \\
    & \texttt{task5}        & \texttt{90 \std{7}}                       & \texttt{\textcolor{highlight}{95} \std{7}}& \texttt{\textcolor{highlight}{97} \std{4}} \\
    & \texttt{overall}      & \texttt{75 \std{3}}                       & \texttt{77 \std{5}}                       & \texttt{\textcolor{highlight}{94} \std{2}} \\ \midrule
    \multirow{6}{*}{\texttt{puzzle-4x5-1B}}
    & \texttt{task1}        & \texttt{\textcolor{highlight}{100} \std{0}}& \texttt{\textcolor{highlight}{98} \std{3}}& \texttt{\textcolor{highlight}{100} \std{0}} \\
    & \texttt{task2}        & \texttt{\textcolor{highlight}{96} \std{6}}& \texttt{\textcolor{highlight}{99} \std{2}}& \texttt{\textcolor{highlight}{100} \std{1}} \\
    & \texttt{task3}        & \texttt{78 \std{28}}                      & \texttt{\textcolor{highlight}{96} \std{6}}& \texttt{\textcolor{highlight}{100} \std{1}} \\
    & \texttt{task4}        & \texttt{\textcolor{highlight}{97} \std{2}}& \texttt{\textcolor{highlight}{98} \std{2}}& \texttt{\textcolor{highlight}{99} \std{3}} \\
    & \texttt{task5}        & \texttt{74 \std{21}}                      & \texttt{91 \std{9}}                       & \texttt{\textcolor{highlight}{97} \std{3}} \\
    & \texttt{overall}      & \texttt{89 \std{11}}                      & \texttt{\textcolor{highlight}{96} \std{3}}& \texttt{\textcolor{highlight}{99} \std{1}} \\ \midrule
    \multirow{6}{*}{\texttt{puzzle-4x6-1B}}
    & \texttt{task1}        & \texttt{\textcolor{highlight}{100} \std{0}}& \texttt{\textcolor{highlight}{99} \std{2}}& \texttt{\textcolor{highlight}{100} \std{0}} \\
    & \texttt{task2}        & \texttt{42 \std{19}}                      & \texttt{54 \std{29}}                      & \texttt{\textcolor{highlight}{91} \std{15}} \\
    & \texttt{task3}        & \texttt{74 \std{20}}                      & \texttt{92 \std{6}}                       & \texttt{\textcolor{highlight}{100} \std{0}} \\
    & \texttt{task4}        & \texttt{42 \std{18}}                      & \texttt{81 \std{16}}                      & \texttt{\textcolor{highlight}{99} \std{2}} \\
    & \texttt{task5}        & \texttt{0 \std{0}}                        & \texttt{3 \std{4}}                        & \texttt{\textcolor{highlight}{31} \std{16}} \\
    & \texttt{overall}      & \texttt{51 \std{8}}                       & \texttt{66 \std{9}}                       & \texttt{\textcolor{highlight}{84} \std{6}} \\ \midrule
    \multirow{6}{*}{\texttt{cube-triple-100M}}
    & \texttt{single\_pnp}  & \texttt{60 \std{10}}                      & \texttt{\textcolor{highlight}{79} \std{9}}& \texttt{\textcolor{highlight}{83} \std{9}} \\
    & \texttt{triple\_pnp}  & \texttt{29 \std{10}}                      & \texttt{44 \std{7}}                       & \texttt{\textcolor{highlight}{57} \std{11}} \\
    & \texttt{from\_stack}  & \texttt{27 \std{13}}                      & \texttt{38 \std{9}}                       & \texttt{\textcolor{highlight}{51} \std{6}} \\
    & \texttt{cycle}        & \texttt{3 \std{4}}                        & \texttt{7 \std{2}}                        & \texttt{\textcolor{highlight}{16} \std{7}} \\
    & \texttt{stack}        & \texttt{37 \std{12}}                      & \texttt{18 \std{12}}                      & \texttt{\textcolor{highlight}{53} \std{13}} \\
    & \texttt{overall}      & \texttt{31 \std{5}}                       & \texttt{37 \std{5}}                       & \texttt{\textcolor{highlight}{52} \std{5}} \\
}
\end{table}

\clearpage
\subsection{Standard OGBench Tasks}
\label{sec:full_standard}
\begin{table}[htbp]
    \centering
    \caption{Full results on state-based OGBench tasks.}
    \label{tab:standard_performance}
    \resizebox{\textwidth}{!}{%
\begin{tabular}{@{}llcccccccccccccc@{}}
            \toprule
            \texttt{Environment} & \texttt{Task}    & \texttt{BC}                               & \texttt{FBC}                               & \texttt{IVL}                                & \texttt{IQL}                               & \texttt{TD-n}                               & \texttt{CRL}                                & \texttt{MC}                               & \texttt{QRL}                               & \texttt{TMD} & \texttt{TDP}                              & \texttt{COE}                              & \texttt{TRL}                                & \textcolor{highlight}{\method}              \\ \midrule
            \multirow{6}{*}{\texttt{pointmaze-large}}
                                 & \texttt{task1}   & \texttt{53 \std{24}}                      & \texttt{\textcolor{highlight}{98} \std{3}} & \texttt{\textcolor{highlight}{94} \std{6}}  & \texttt{92 \std{9}}                        & \texttt{74 \std{13}}                        & \texttt{23 \std{18}}                        & \texttt{13 \std{12}}                       & \texttt{31 \std{28}}                       & \texttt{18 \std{25}} & \texttt{40 \std{19}}                      & \texttt{67 \std{26}}                      & \texttt{52 \std{19}}                        & \texttt{82 \std{9}}                         \\
                                 & \texttt{task2}   & \texttt{3 \std{4}}                        & \texttt{17 \std{6}}                        & \texttt{0 \std{0}}                          & \texttt{0 \std{0}}                         & \texttt{0 \std{0}}                          & \texttt{\textcolor{highlight}{42} \std{28}} & \texttt{0 \std{0}}                         & \texttt{5 \std{6}}                         & \texttt{4 \std{9}} & \texttt{0 \std{0}}                        & \texttt{0 \std{0}}                        & \texttt{1 \std{2}}                          & \texttt{0 \std{0}}                          \\
\bottomrule
        \end{tabular}%
    }
\end{table}

\begin{table}[htbp]
    \centering
    \caption{Full results on pixel-based OGBench tasks.}
    \label{tab:visual_performance}
    \resizebox{\textwidth}{!}{%
        \begin{tabular}{@{}llcccccccccccccc@{}}
            \toprule
            \texttt{Environment} & \texttt{Task}    & \texttt{BC} & \texttt{FBC} & \texttt{IVL} & \texttt{IQL} & \texttt{TD-n} & \texttt{CRL} & \texttt{MC} & \texttt{QRL} & \texttt{TMD} & \texttt{TDP} & \texttt{COE} & \texttt{TRL} & \textcolor{highlight}{\method} \\ \midrule
            \multirow{6}{*}{\texttt{visual-cube-single}}
                                 & \texttt{task1}   & \texttt{12 \std{4}} & \texttt{15 \std{8}} & \texttt{70 \std{3}} & \texttt{42 \std{12}} & \texttt{73 \std{13}} & \texttt{47 \std{20}} & \texttt{52 \std{7}} & \texttt{68 \std{9}} & \texttt{9 \std{3}} & \texttt{1 \std{1}} & \texttt{0 \std{0}} & \texttt{\textcolor{highlight}{80} \std{9}} & \texttt{\textcolor{highlight}{77} \std{10}} \\
                                 & \texttt{task2}   & \texttt{4 \std{4}} & \texttt{5 \std{3}} & \texttt{\textcolor{highlight}{65} \std{9}} & \texttt{44 \std{11}} & \texttt{64 \std{16}} & \texttt{40 \std{20}} & \texttt{55 \std{11}} & \texttt{35 \std{34}} & \texttt{5 \std{1}} & \texttt{1 \std{1}} & \texttt{0 \std{0}} & \texttt{\textcolor{highlight}{67} \std{12}} & \texttt{\textcolor{highlight}{68} \std{11}} \\
                                 & \texttt{task3}   & \texttt{6 \std{5}} & \texttt{4 \std{4}} & \texttt{49 \std{2}} & \texttt{24 \std{9}} & \texttt{60 \std{11}} & \texttt{33 \std{17}} & \texttt{57 \std{5}} & \texttt{41 \std{27}} & \texttt{10 \std{3}} & \texttt{2 \std{1}} & \texttt{0 \std{0}} & \texttt{\textcolor{highlight}{78} \std{9}} & \texttt{62 \std{11}} \\
                                 & \texttt{task4}   & \texttt{0 \std{0}} & \texttt{0 \std{1}} & \texttt{60 \std{13}} & \texttt{21 \std{6}} & \texttt{52 \std{15}} & \texttt{18 \std{16}} & \texttt{39 \std{6}} & \texttt{32 \std{10}} & \texttt{0 \std{1}} & \texttt{0 \std{0}} & \texttt{0 \std{0}} & \texttt{\textcolor{highlight}{67} \std{11}} & \texttt{48 \std{10}} \\
\bottomrule
        \end{tabular}%
    }
\end{table}

\clearpage
\section{Hyperparameters}
\label{sec:hyperparameters}

\begin{table}[ht]
\centering
\caption{Common hyperparameters.}
\label{tab:hyperparameters_common}
\resizebox{\textwidth}{!}{
\begin{tabular}{ll}
\toprule
\texttt{Hyperparameter} & \texttt{Value} \\
\midrule
Optimizer & Adam \citep{kingma2015adam} \\
Learning rate & $0.0003$ \\
Batch size & $1024$ (state-based), $256$ (pixel-based) \\
Number of slots & $128$ (state-based), $64$ (pixel-based) \\
Nonlinearity & GELU \citep{hendrycks2016gaussian} \\
Target network update rate & $0.005$ \\
Critic ensemble size & 2 \\
Critic target & \texttt{min}$(Q_1,Q_2)$ for \texttt{cube-*} and CALVIN, otherwise \texttt{mean}$(Q_1,Q_2)$ \\
Number of flow steps & 10 \\
Policy extraction & Reparameterized gradients or rejection sampling (see \cref{tab:hyperparameters}) \\
\bottomrule
\end{tabular}
}
\end{table}

\begin{table}[ht]
\centering
\caption{Hyperparameters for long-horizon OGBench tasks.}
\label{tab:hyperparameters_ogbench}
\resizebox{\textwidth}{!}{
\begin{tabular}{ll}
\toprule
\texttt{Hyperparameter} & \texttt{Value} \\
\midrule
Gradient steps & $10^6$ \\
MLP size & $[1024, 1024, 1024, 1024]$ \\
Discount factor $\gamma$ & $0.999$ \\
$n$-step length ($n$) (\texttt{TD-n}, \method) & $100$ (\texttt{humanoidmaze}, \texttt{puzzle}), $25$ (\texttt{cube})\\
Policy $(p_{\text{curr}}^{\mathcal{D}}, p_{\text{geom}}^{\mathcal{D}}, p_{\text{traj}}^{\mathcal{D}}, p_{\text{rand}}^{\mathcal{D}})$ ratio & $(0, 0, 1, 0)$ (\texttt{humanoidmaze}), $(0, 0.5, 0, 0.5)$ (\texttt{puzzle}) \\
Value $(p_{\text{curr}}^{\mathcal{D}}, p_{\text{geom}}^{\mathcal{D}}, p_{\text{traj}}^{\mathcal{D}}, p_{\text{rand}}^{\mathcal{D}})$ ratio (\trl, \texttt{TD}, \texttt{MC}) & $(0, 0, 1, 0)$ \\
Value $(p_{\text{curr}}^{\mathcal{D}}, p_{\text{geom}}^{\mathcal{D}}, p_{\text{traj}}^{\mathcal{D}}, p_{\text{rand}}^{\mathcal{D}})$ ratio (\texttt{CRL}) & $(0, 1, 0, 0)$ \\
Behavior Value $(p_{\text{curr}}^{\mathcal{D}}, p_{\text{geom}}^{\mathcal{D}}, p_{\text{traj}}^{\mathcal{D}}, p_{\text{rand}}^{\mathcal{D}})$ ratio (\method) & $(0, 0, 1, 0)$ \\
Optimal Value $(p_{\text{curr}}^{\mathcal{D}}, p_{\text{geom}}^{\mathcal{D}}, p_{\text{traj}}^{\mathcal{D}}, p_{\text{rand}}^{\mathcal{D}})$ ratio (\method) & $(0.2, 0.5, 0, 0.3)$ \\
Value $(p_{\text{curr}}^{\mathcal{D}}, p_{\text{geom}}^{\mathcal{D}}, p_{\text{traj}}^{\mathcal{D}}, p_{\text{rand}}^{\mathcal{D}})$ ratio (others) & $(0.2, 0.5, 0, 0.3)$ \\
\bottomrule
\end{tabular}
}
\end{table}

\begin{table}[ht]
\centering
\caption{Hyperparameters for standard OGBench tasks.}
\label{tab:hyperparameters_standard_ogbench}
\resizebox{\textwidth}{!}{
\begin{tabular}{ll}
\toprule
\texttt{Hyperparameter} & \texttt{Value} \\
\midrule
Gradient steps & $10^6$ (state-based), $5 \times 10^5$ (pixel-based) \\
MLP size & $[512, 512, 512]$ \\
Discount factor $\gamma$ & $0.99$ (default), $0.995$ (\texttt{humanoidmaze}) \\
$n$-step length ($n$) (\texttt{TD-n}, \method) & $25$ \\
Image augmentation probability (pixel-based) & $0.5$ (manipulation), $0.0$ (locomotion) \\
Random-crop padding (pixel-based) & $3$ \\
Policy $(p_{\text{curr}}^{\mathcal{D}}, p_{\text{geom}}^{\mathcal{D}}, p_{\text{traj}}^{\mathcal{D}}, p_{\text{rand}}^{\mathcal{D}})$ ratio & $(0, 0, 1, 0)$ \\
Value $(p_{\text{curr}}^{\mathcal{D}}, p_{\text{geom}}^{\mathcal{D}}, p_{\text{traj}}^{\mathcal{D}}, p_{\text{rand}}^{\mathcal{D}})$ ratio (\trl, \texttt{CRL}) & $(0, 1, 0, 0)$ \\
Value $(p_{\text{curr}}^{\mathcal{D}}, p_{\text{geom}}^{\mathcal{D}}, p_{\text{traj}}^{\mathcal{D}}, p_{\text{rand}}^{\mathcal{D}})$ ratio (\texttt{TD}, \texttt{MC}) & $(0, 0, 1, 0)$ \\
Behavior Value $(p_{\text{curr}}^{\mathcal{D}}, p_{\text{geom}}^{\mathcal{D}}, p_{\text{traj}}^{\mathcal{D}}, p_{\text{rand}}^{\mathcal{D}})$ ratio (\method) & $(0, 0, 1, 0)$ \\
Optimal Value $(p_{\text{curr}}^{\mathcal{D}}, p_{\text{geom}}^{\mathcal{D}}, p_{\text{traj}}^{\mathcal{D}}, p_{\text{rand}}^{\mathcal{D}})$ ratio (\method) & $(0.2, 0.5, 0, 0.3)$ \\
Value $(p_{\text{curr}}^{\mathcal{D}}, p_{\text{geom}}^{\mathcal{D}}, p_{\text{traj}}^{\mathcal{D}}, p_{\text{rand}}^{\mathcal{D}})$ ratio (others) & $(0.2, 0.5, 0, 0.3)$ \\
\bottomrule
\end{tabular}
}
\end{table}

\begin{table}[ht]
\centering
\caption{Hyperparameters for CALVIN.}
\label{tab:hyperparameters_calvin}
\resizebox{\textwidth}{!}{
\begin{tabular}{ll}
\toprule
\texttt{Hyperparameter} & \texttt{Value} \\
\midrule
Gradient steps & $5 \times 10^5$ (state-based), $2 \times 10^5$ (pixel-based) \\
Learning rate & $3 \times 10^{-4}$ \\
Actor MLP size & $[256, 256]$ \\
Value MLP size & $[512, 512, 512]$ \\
Discount factor $\gamma$ & $0.99$ \\
$n$-step length ($n$) (\texttt{TD-n}, \method) & $25$ \\
Image augmentation probability (pixel-based) & $0.5$ \\
Random-crop padding (pixel-based) & $6$ \\
Policy $(p_{\text{curr}}^{\mathcal{D}}, p_{\text{geom}}^{\mathcal{D}}, p_{\text{traj}}^{\mathcal{D}}, p_{\text{rand}}^{\mathcal{D}})$ ratio & $(0, 0.5, 0, 0.5)$ \\
Value $(p_{\text{curr}}^{\mathcal{D}}, p_{\text{geom}}^{\mathcal{D}}, p_{\text{traj}}^{\mathcal{D}}, p_{\text{rand}}^{\mathcal{D}})$ ratio (\trl, \texttt{TD}, \texttt{MC}, \texttt{CRL}) & $(0, 1, 0, 0)$ \\
Behavior Value $(p_{\text{curr}}^{\mathcal{D}}, p_{\text{geom}}^{\mathcal{D}}, p_{\text{traj}}^{\mathcal{D}}, p_{\text{rand}}^{\mathcal{D}})$ ratio (\method) & $(0, 1, 0, 0)$ \\
Optimal Value $(p_{\text{curr}}^{\mathcal{D}}, p_{\text{geom}}^{\mathcal{D}}, p_{\text{traj}}^{\mathcal{D}}, p_{\text{rand}}^{\mathcal{D}})$ ratio (\method) & $(0.2, 0.5, 0, 0.3)$ \\
Value $(p_{\text{curr}}^{\mathcal{D}}, p_{\text{geom}}^{\mathcal{D}}, p_{\text{traj}}^{\mathcal{D}}, p_{\text{rand}}^{\mathcal{D}})$ ratio (others) & $(0.2, 0.5, 0, 0.3)$ \\
Gradient clipping norm & $1.0$ \\
\midrule
Subgoal (waypoint) steps (\texttt{HIQL}) & $25$ \\
AWR temperature (low-level / high-level) (\texttt{HIQL}) & $1$ / $1$ \\
High-level random-goal probability (\texttt{HIQL}) & $0.3$ \\
Latent subgoal representation (\texttt{HIQL}) & \texttt{False} (state-based), \texttt{True} (pixel-based) \\
High-level action dimension (\texttt{HIQL}) & $21$ (state-based), $527$ (pixel-based) \\
\bottomrule
\end{tabular}
}
\end{table}

\begin{table}[h]
\centering
\caption{Task-specific hyperparameters for all methods. ($\alpha$: BC coefficient for reparameterized gradients, $N$: sample count for rejection sampling, $M$: subgoal count, $P$: random
goal distance, $\beta$: goal regularization weight, $\kappa$: expectile, $\lambda$: distance re-weighting factor.) For \method, \(\kappa\) denotes \(\kappa_{\mathrm{prop}}\); we fix \(\kappa_{\mathrm{dc}}=0.5\). For hyperparameters not listed here, we adopt the default values from \citet{eysenbach2022contrastive, park2023hiql, park2025ogbench, park2026transitive}.}
\label{tab:hyperparameters}
\resizebox{\textwidth}{!}{
\begin{tabular}{lccccccccc}
\toprule
\multirow{2}{*}{\texttt{Environment}} & \texttt{QRL} & \texttt{TDP} & \texttt{COE} & \texttt{TD} & \texttt{MC} & \texttt{TMD} & \{\texttt{TD-n},\trl,\method\} & \{\trl,\method\}  & \trl \\
 & $\alpha$ & $(\alpha,N),M,P$ & $(\alpha),\beta$ & $(\alpha,N)$ & $(\alpha,N)$ & $\alpha$ & $(\alpha,N)$ & $\kappa$ & $\lambda$ \\
\midrule
\texttt{humanoidmaze-giant-1B} & $0.001$ & $(3,\text{-}),8,1000$ & $1,0.5$ & $(0.3,\text{-})$ & $(0.3,\text{-})$ & - & \trl:\;$(0.1,\text{-})$ / \method:\;$(\text{-},32)$ & $0.7$ & $0$ \\
\texttt{puzzle-4x5-1B} & $3$ & $(\text{-},32),8,1000$ & $3,10$ & $(\text{-},32)$ & $(\text{-},32)$ & - & $(\text{-},32)$ & $0.7$ & $0$ \\
\texttt{puzzle-4x6-1B} & $3$ & $(\text{-},32),8,1000$ & $3,10$ & $(\text{-},32)$ & $(\text{-},32)$ & - & $(\text{-},32)$ & $0.7$ & $0$ \\
\texttt{cube-quadruple-100M} & $3$ & $(\text{-},4),8,1000$ & $3,10$ & $(\text{-},4)$ & $(\text{-},4)$ & - & $(\text{-},4)$ & $0.7$ & $0$ \\
\texttt{cube-octuple-1B} & $3$ & $(\text{-},4),8,1000$ & $3,10$ & $(\text{-},4)$ & $(\text{-},4)$ & - & $(\text{-},4)$ & $0.7$ & $0$ \\
\midrule
\texttt{pointmaze-large} & $0.0003$ & $(5,\text{-}),8,500$ & $1,1$ & $(10,\text{-})$ & $(2,\text{-})$ & $0.03$ & $(2,\text{-})$ & $0.7$ & $0.7$ \\
\texttt{antmaze-large} & $0.003$ & $(10,\text{-}),8,100$ & $3,10$ & $(0.7,\text{-})$ & $(0.7,\text{-})$ & $0.1$ & $(0.7,\text{-})$ & $0.7$ & $0$ \\
\texttt{antsoccer-arena} & $0.003$ & $(10,\text{-}),8,200$ & $0.3,1$ & $(0.3,\text{-})$ & $(0.3,\text{-})$ & $0.3$ & $(0.3,\text{-})$ & $0.7$ & $0.5$ \\
\texttt{humanoidmaze-medium} & $0.001$ & $(5,\text{-}),8,500$ & $0.1,3$ & $(0.1,\text{-})$ & $(0.1,\text{-})$ & $0.1$ & $(0.1,\text{-})$ & $0.7$ & $0$ \\
\texttt{humanoidmaze-large} & $0.001$ & $(5,\text{-}),8,500$ & $0.1,1$ & $(0.1,\text{-})$ & $(0.1,\text{-})$ & $0.1$ & $(0.1,\text{-})$ & $0.7$ & $0.1$ \\
\texttt{cube-single} & $0.3$ & $(5,\text{-}),8,500$ & $0.3,10$ & $(1,\text{-})$ & $(1,\text{-})$ & $3.0$ & $(1,\text{-})$ & $0.7$ & $0.7$ \\
\texttt{cube-double} & $0.3$ & $(5,\text{-}),8,500$ & $0.3,10$ & $(10,\text{-})$ & $(2,\text{-})$ & $3.0$ & $(2,\text{-})$ & $0.7$ & $1$ \\
\texttt{scene} & $0.3$ & $(1,\text{-}),16,200$ & $1,10$ & $(1,\text{-})$ & $(1,\text{-})$ & $3.0$ & $(1,\text{-})$ & $0.7$ & $1$ \\
\texttt{puzzle-3x3} & $0.3$ & $(5,\text{-}),8,500$ & $3,10$ & $(2,\text{-})$ & $(2,\text{-})$ & $3.0$ & $(2,\text{-})$ & $0.7$ & $0.5$ \\
\texttt{puzzle-4x4} & $0.3$ & $(5,\text{-}),8,500$ & $1,10$ & $(2,\text{-})$ & $(1,\text{-})$ & $3.0$ & $(1,\text{-})$ & $0.7$ & $2$ \\
\midrule
\texttt{visual-cube-single} & \text{0.3} & $(5,\text{-}),16,500$ & $0.3,10$ & \text{-} & $(2,\text{-})$ & $3.0$ & $(\text{-},8)$ & $0.7$ & \text{0} \\
\texttt{visual-cube-double} & \text{0.3} & $(5,\text{-}),16,500$ & $0.3,10$ & \text{-} & $(0.7,\text{-})$ & $3.0$ & $(\text{-},8)$ & $0.7$ & \text{0} \\
\texttt{visual-cube-triple} & \text{0.3} & $(5,\text{-}),16,500$ & $0.3,10$ & \text{-} & $(0.7,\text{-})$ & $3.0$ & $(\text{-},8)$ & $0.7$ & \text{0} \\
\texttt{visual-scene} & \text{0.3} & $(1,\text{-}),16,200$ & $1,10$ & \text{-} & $(1,\text{-})$ & $3.0$ & $(\text{-},8)$ & $0.7$ & \text{0} \\
\texttt{visual-puzzle-3x3} & \text{0.3} & $(5,\text{-}),16,500$ & $3,10$ & \text{-} & $(2,\text{-})$ & $3.0$ & $(\text{-},8)$ & $0.7$ & \text{0} \\
\texttt{visual-puzzle-4x4} & \text{0.3} & $(5,\text{-}),16,500$ & $1,10$ & \text{-} & $(1,\text{-})$ & $3.0$ & $(\text{-},8)$ & $0.7$ & \text{0} \\
\texttt{visual-puzzle-4x5} & \text{0.3} & $(5,\text{-}),16,500$ & $1,10$ & \text{-} & $(0.5,\text{-})$ & $3.0$ & $(\text{-},8)$ & $0.7$ & \text{0} \\
\texttt{visual-antmaze-medium} & \text{0.003} & $(10,\text{-}),8,100$ & $3,10$ & \text{-} & $(0.3,\text{-})$ & $0.1$ & $(\text{-},8)$ & $0.7$ & \text{0} \\
\texttt{visual-antmaze-large} & \text{0.003} & $(10,\text{-}),8,100$ & $3,10$ & \text{-} & $(0.3,\text{-})$ & $0.1$ & $(\text{-},8)$ & $0.7$ & \text{0} \\
\texttt{visual-antmaze-giant} & \text{0.003} & $(10,\text{-}),8,100$ & $3,10$ & \text{-} & $(0.3,\text{-})$ & $0.1$ & $(\text{-},8)$ & $0.7$ & \text{0} \\
\midrule
\texttt{CALVIN} & $0.003$ & \text{-} & \text{-} & \text{-} & $(\text{-},32)$ & $1.0$ & $(\text{-},32)$ & $0.7$ & \text{0} \\
\bottomrule
\end{tabular}
}
\end{table}

\end{document}